\documentclass{article} %
\usepackage{arxiv,times}
\usepackage{url}
\usepackage{microtype}
\usepackage{graphicx}
\usepackage{subfigure}
\usepackage{booktabs}
\usepackage[colorlinks=true,allcolors=blue]{hyperref}

\usepackage{amsmath}
\usepackage{amssymb}
\usepackage{mathtools}
\usepackage{amsthm}
\usepackage{listings}%
\usepackage{amsfonts}       %
\usepackage[utf8]{inputenc} %
\usepackage[T1]{fontenc}    %
\usepackage{url}            %
\usepackage{nicefrac}      
\usepackage{microtype}      %
\usepackage{xcolor}         %
\usepackage{natbib}
\usepackage{braket}
\usepackage{thmtools}
\usepackage{comment}
\usepackage{caption}

\usepackage[shortlabels]{enumitem}
\setlist[enumerate,1]{leftmargin=0.5cm}
\setlist[itemize,1]{leftmargin=0.5cm}

\usepackage[capitalize,noabbrev]{cleveref}
\newcommand{\TransformerBlock}{f}
\newcommand{\RNNBlock}{f}
\newcommand{\Transformer}{T}
\newcommand{\Poly}{\text{Poly}}
\newcommand{\poly}{\text{poly}}
\newcommand{\GLU}{\mathrm{LU}}
\newcommand{\FFN}{g}

\newcommand{\R}{\mathbb{R}}
\newcommand{\argmax}{\arg \max}
\newcommand{\argmin}{\arg \min}

\newcommand{\Attention}{\mathcal{A}}
\newcommand{\Softmax}{\mathrm{softmax}}
\newcommand{\Key}{W^{(K)}}
\newcommand{\Query}{W^{(Q)}}
\newcommand{\Value}{W^{(V)}}
\newcommand{\HeadKey}{W^{(K,h)}}
\newcommand{\HeadQuery}{W^{(Q, h)}}
\newcommand{\HeadValue}{W^{(V, h)}}
\newcommand{\CausalMask}{C}
\newcommand{\Embedding}{\textrm{Emb}}
\newcommand{\WordEmbedding}{W^{(E)}}
\newcommand{\PositionEmbedding}{W^{(P)}}

\newcommand{\Vocab}{V}
\newcommand{\VocabAnswer}{V_{\mathrm{A}}}
\newcommand{\VocabSize}{|V|}
\newcommand{\Prb}{\mathbb{P}}
\newcommand{\width}{w}
\newcommand{\dimension}{d}
\newcommand{\hiddendimension}{\Lambda}
\newcommand{\state}{s}
\newcommand{\TuringState}{\mathrm{State}}
\newcommand{\Direction}{\mathrm{Direction}}
\newcommand{\StateTransition}{\mathbf{t}}
\newcommand{\StateOutput}{\mathbf{o}}
\newcommand{\RNN}{R}
\newcommand{\IsTree}{\mathrm{IsTree}}
\newcommand{\YES}{\mathrm{YES}}
\newcommand{\NO}{\mathrm{NO}}

\newcommand{\ReLU}{\mathrm{ReLU}}
\newcommand{\one}[1]{\mathbf{1}{\left[#1\right]}}
\newcommand{\ROUND}{\mathrm{ROUND}}
\newcommand{\COPY}{\mathrm{COPY}}
\newcommand{\COUNT}{\mathrm{COUNT}}

\newcommand{\RETRIEVE}{\mathrm{RETRIEVE}}
\newcommand{\verticenumber}{n}
\newcommand{\edgenumber}{m}
\newcommand{\graphset}{\mathcal{G}}
\newcommand{\graph}{G}
\newcommand{\Edge}{\mathrm{E}}
\newcommand{\Tokenize}{\mathrm{Tokenize}}
\newcommand{\precision}{p}
\newcommand{\maxposition}{L}
\newcommand{\position}{l}
\newcommand{\onehot}{w}
\newcommand{\numspecial}{n_S}
\newcommand{\hiddenstate}{X}
\newcommand{\VEmbed}[1]{#1 \onehot_1} 
\newcommand{\TokenSequence}{\mathcal{S}}
\newcommand{\InputTokenSequence}{\mathcal{S}_{\mathrm{in}}}

\newcommand{\ParameterSize}{\mathrm{P}}

\newcommand{\MemorySize}{\mathrm{M}}
\newcommand{\CircuitSize}{\mathrm{C}}
\newcommand{\model}{M}
\newcommand{\inputlength}{l_0}
\newcommand{\outputlength}{T}
\newcommand{\StartSearch}{\texttt{<}\mathrm{StartSearch}\texttt{>}}
\newcommand{\EndSearch}{\texttt{<}\mathrm{EndSearch}\texttt{>}}
\newcommand{\NextToken}{s^{\mathrm{next}}}
\newcommand{\StartPos}{l_s}
\newcommand{\EndPos}{l_e}
\newcommand{\iterpos}{k}

\newcommand{\StartSentence}{\texttt{<s>}}
\newcommand{\Integer}{\mathbb{Z}_{\precision}}
\newcommand{\Match}{\mathrm{Match}}
\newcommand{\MatchNext}{\mathrm{MatchNext}}
\newcommand{\nextnum}{\mathrm{next}}
\newcommand{\appnext}{\overline{\mathrm{next}}}
\newcommand{\algo}{\mathcal{A}}
\newcommand{\Z}{\mathbb{Z}}
\newcommand{\layer}{\ell}
\newcommand{\headnumber}{H}
\newcommand{\Time}{\mathrm{TIME}}
\newcommand{\tape}{\mathrm{TAPE}}
\newcommand{\Tape}{\tape}
\newcommand{\pointer}{\mathrm{POINTER}}
\newcommand{\Pointer}{\pointer}
\newcommand{\SearchResult}{\mathrm{SearchResult}}
\newcommand{\Failed}{\mathrm{FAILED}}

\newcommand{\Detokenize}{\mathrm{Detokenize}}

\newcommand{\regexparam}[1]{{\color{blue}{#1}}}
\newcommand{\result}[1]{{\color{red}{{#1}}}}
\newcommand{\Hybrid}{\mathcal{H}}
\newcommand{\MatchClose}{\mathrm{MatchClose}}
\newcommand{\MatchNearest}{\mathrm{MatchNearest}}
\newcommand{\ICRA}{\text{In-context RAG}}
\newcommand{\ICR}{\text{In-context Retrieval}}
\newcommand{\specialtoken}{\kappa}

\theoremstyle{plain}
\newtheorem{theorem}{Theorem}[section]
\newtheorem{proposition}[theorem]{Proposition}
\newtheorem{lemma}[theorem]{Lemma}

\theoremstyle{definition}
\newtheorem{definition}[theorem]{Definition}

\theoremstyle{remark}

\usepackage[textsize=tiny]{todonotes}

\title{RNNs are not Transformers (Yet): \\ The Key Bottleneck on $\ICR$}
\date{}

\author{
    Kaiyue Wen$^1$\thanks{Equal contribution} \quad Xingyu Dang$^1$\footnotemark[1] \quad
    Kaifeng Lyu$^2$\thanks{Corresponding author} \\
    $^1$Institute for Interdisciplinary Information Sciences, Tsinghua University\\
    $^2$Department of Computer Science \& Princeton Language and Intelligence, Princeton University \\
    \texttt{\{wenky20,dangxy20\}@mails.tsinghua.edu.cn} \\
    \texttt{klyu@cs.princeton.edu}
}

\usepackage{tikz}
\usetikzlibrary{arrows.meta, positioning, calc, shapes.geometric, decorations.pathreplacing, fit, arrows.meta}

\begin{document}

\maketitle

\begin{abstract}
This paper investigates the gap in representation powers of Recurrent Neural Networks (RNNs) and Transformers in the context of solving algorithmic problems. We focus on understanding whether RNNs, known for their memory efficiency in handling long sequences, can match the performance of Transformers, particularly when enhanced with Chain-of-Thought (CoT) prompting. Our theoretical analysis reveals that CoT improves RNNs but is insufficient to close the gap with Transformers. A key bottleneck lies in the inability of RNNs to perfectly retrieve information from the context, even with CoT: 
for several tasks that explicitly or implicitly require this capability, such as associative recall and determining if a graph is a tree, we prove that RNNs are not expressive enough to solve the tasks while Transformers can solve them with ease.
Conversely, we prove that adopting techniques to enhance the in-context retrieval capability of RNNs, including Retrieval-Augmented Generation (RAG) and adding a single Transformer layer, can elevate RNNs to be capable of solving all polynomial-time solvable problems with CoT, hence closing the representation gap with Transformers.\footnote{Code is available at \url{https://github.com/dangxingyu/rnn-icrag}}
\end{abstract}

\section{Introduction}
\label{sec:intro}
Transformers~\citep{vaswani2017attention} have become the dominant choice of the backbone for Large Language Models (LLMs). The core component of Transformers is self-attention modules, which allow the model to route information densely across the entire sequence. However, 
this design leads to high inference costs for long sequences, including a memory cost linear in the sequence length to maintain intermediate attention keys and values for each token, and a time cost quadratic in the sequence length to compute the attention score for each pair of tokens.

Recently, Recurrent Neural Networks (RNNs) have become an increasingly popular choice in sequence modeling tasks due to their ability to maintain a memory size constant in sequence length during inference, thus being more memory efficient than Transformers.
\citet{katharopoulos2020transformers} showed that Linear Transformers (Transformers with linear attention) can be expressed as RNNs. 
\citet{gu2022efficiently} took a different path to design RNNs by structuring latent states as State Space Models (SSMs) from control theory. These ideas have led to a series of development of modern RNNs, including RWKV~\citep{peng2023rwkv}, RetNet~\citep{sun2023retentive}, and Mamba~\citep{gu2023mamba}. Most notably, Mamba can achieve competitive performance with Transformers on several sequence modeling tasks with linear time and constant memory in sequence length.

{\em Can RNNs replace Transformers yet?}
The rise of these modern RNNs has led to an interest in understanding their limitations.
A recent work by~\citet{arora2023zoology} showed that a broad family of RNNs, input-independent gating SSMs, are empirically inferior to Transformers in a task that has a long history in artificial intelligence, \textit{Associative Recall} (AR) \citep{willshaw1969non,hopfield1982neural,hinton2014parallel}:
Given a series of key-value pairs as a string, the model is required to recall the value given a key.
On the theory side, \citet{sanford2023representational} and \citet{jelassi2024repeat} demonstrated that constant-memory RNNs do not have sufficient representation power to solve the tasks of averaging a given subset of input vectors ($q$-sparse averaging) and repeating the input sequence (copying), respectively, while there exist shallow Transformers that can solve these tasks.

However, the above results do not exclude the possibility that enhancing RNNs with additional prompting techniques or minor architectural changes could close the gap with Transformers.
In fact, Transformers themselves are not perfect either and sometimes need additional techniques at inference time to perform well. 
For instance, Transformers may struggle with mathematical and algorithmic reasoning problems if they are forced to produce the correct answer immediately after processing the input sequence.
But with {\em Chain-of-Thought} (CoT) prompting applied, a prompting technique that guides the model to generate a series of intermediate tokens before arriving at the final answer, their performance can be significantly improved.
\cite{feng2023towards,li2024chain} explained the success of CoT from the perspective of representation power: Transformers alone do not have sufficient representation power to solve problems beyond a certain circuit complexity class ($\mathrm{TC}^0$), but with CoT, they can even simulate any polynomial-time Turing machine. 

The effectiveness of CoT on Transformers naturally leads to the following question:
\begin{center}
    {\em Can similar enhancements, such as adopting CoT, improve RNNs to be on par with Transformers?}
\end{center}

\begin{figure}[t]
    \vspace{-0.15in}
    \centering
        \begin{tikzpicture}
            \definecolor{softBlue}{RGB}{135, 206, 235} %
            \definecolor{softGreen}{RGB}{144, 238, 144} %
            \definecolor{softRed}{RGB}{240, 128, 128}   %
  
            \draw[fill=softGreen] (0, 0) rectangle (6.1, 3.8);
            \node at (4.2, 3.2) [align=center] {Transformer + CoT  $\supset \mathrm{P}$};

            \node[] at (3, -0.5) [align=center] {Transformer};

            \draw[fill=softBlue] (0.6, 0.36) rectangle (4.8, 2.84);
            \node at (2.7, 1.6) [align=center] {Vanilla Transformer};

            \draw[fill=softRed] (7, 0) rectangle (13.1, 3.8);
            \node at (10.9, 3.2) [align=center] {RNN + $\ICRA$ $\supset \mathrm{P}$};
        
            \draw[fill=softGreen] (7.6, 0.36) rectangle (11.8, 2.84);
            \node at (10.5, 2.4) [align=center] {RNN + CoT $\not \supset \mathrm{P}$};
        
            \draw[fill=softBlue] (7.9, 0.54) rectangle (10.6, 2.13);
            \node at (9.25, 1.34) [align=center] {Vanilla RNN};

            \node[] at (10, -0.5) [align=center] {RNN (This work)};
        \end{tikzpicture}

        \centering
        \caption{\textbf{Hierarchy of Representation Power}. While RNN with chain-of-thought (CoT) with $O(\log n)$ bit memory provably has strictly stronger representation power than RNN without CoT under mild complexity assumptions (\Cref{thm:rnncot}), it is still exponentially weaker than Transformer with CoT in representing solutions to algorithmic problems (\Cref{thm:rnn_trans_istree}). We proceed to show that the incapability of RNNs in $\ICR$  is the root cause of the gap and propose two forms of $\ICR$ Augmented Generation ($\ICRA$) to close the gap by illustrating their power to simulate any polynomial-time Turing machines (\Cref{thm:rnn_turing,thm:hybrid_turing}).}
        %\vspace{-0.1in}
\end{figure}
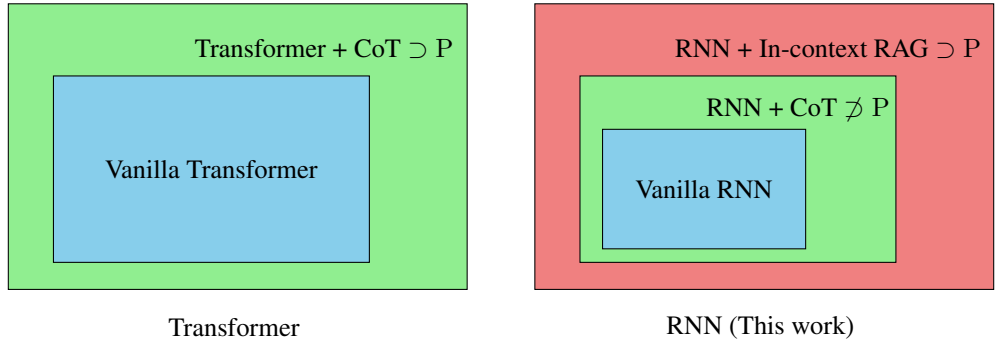

\textbf{Our Contributions.} This paper examines potential ways to close the gap in the representation powers of RNNs and Transformers (with softmax attention) on algorithmic problems.
Through a series of lower and upper bound results, we show that CoT improves the representation power of RNNs, but to close the gap with Transformers, CoT alone is not enough
to overcome a key bottleneck of RNNs: their inability to retrieve information from the context, which we call {\em in-context retrieval} for short.

We further illustrate that addressing this in-context retrieval bottleneck is sufficient to close this gap: RNNs can solve all polynomial-time solvable problems if adopting techniques to enhance the in-context retrieval capability, including involving {\em Retrieval-Augmented Generation} (RAG) and using {\em hybrid architectures}, such as appending a single Transformer layer.

Our main contributions are listed as follows:
\begin{enumerate}[leftmargin=*]
    \item \textbf{CoT improves RNNs but cannot close the representation gap with Transformers.} (\Cref{sec:curse})
        \begin{itemize}[leftmargin=*]
            \item On the positive side, we prove that CoT makes RNNs strictly more expressive under mild assumptions from circuit complexity ($\mathrm{PSPACE} \not\subset \mathrm{P/poly}$).
            \item On the negative side, we show that adopting CoT is not enough to close the representation gap between RNNs and Transformers: the memory efficiency of RNNs fundamentally limits their ability to perform in-context retrieval, even with CoT.
            This point is made concrete by proving that RNNs with CoT cannot solve a set of fundamental algorithmic problems that directly ask for in-context retrieval, including associative recall.
            \item We further exemplify that in-context retrieval can be implicitly required in tasks that appear unrelated, by proving the inability of RNNs to solve the classic problem of determining whether a graph is a tree ($\IsTree$).
            \item On the other hand, we prove that Transformers have the representation power to solve many of the above tasks with ease, including $\IsTree$. Moreover, Transformers with CoT can even simulate RNNs with CoT efficiently, with only a small multiplicative factor in the number of parameters.
        \end{itemize}
        Our negative results hold for a wide range of RNN architectures, including the aforementioned Mamba, RWKV, and even Linear Transformers. Technically, this is because RNNs are so memory efficient that they can trigger streaming lower bounds~\citep{sanford2023representational}, especially for problems that require in-context retrieval.
    \item \textbf{Enhancing the in-context retrieval capability of RNNs can close the gap.} (\Cref{sec:fix})
        \begin{itemize}[leftmargin=*]
            \item We prove that allowing RNNs to invoke function calls to perform a primitive of in-context retrieval based on regular expression is sufficient to boost their representation power to solve all polynomial-time solvable problems with CoT, hence closing the representation gap.
            \item %
            As one layer of the Transformer is sufficient to perform many in-context retrieval operations, mixing some Transformer layers in RNNs should also narrow the representation gap. We prove that a minimal possible change in the RNN architecture can just work: adding one Transformer layer at the end of the RNN architecture is sufficient to close the representation gap.
        \end{itemize}
        Our positive results showing that enhancing in-context retrieval can improve RNNs' representation power hold for vanilla Linear RNNs, and can be easily extended to more complex architectures.
        The intuition behind these results is that RNN can focus on the local reasoning steps and use the in-context retrieval module to adaptively fetch the relevant information from the context.
\end{enumerate}

We validate our theoretical findings through synthetic and natural language experiments on IsTree and HotPot-QA, confirming that while CoT alone cannot close the performance gap between RNNs and Transformers, the proposed two solutions effectively narrow this gap. (\Cref{sec:validation})

\textbf{Implications.} We believe these results could provide valuable insights into architecture designs of LLMs: RNNs alone can suffer from many fundamental limitations in representation power, even with CoT; on the other hand, it is promising to explore strategies to enhance the in-context retrieval capability of RNNs with little overhead, such as using a hybrid architecture that mixes in one or more Transformer layers~\citep{ren2024sambasimplehybridstate,waleffe2024empiricalstudymambabasedlanguage,lieber2024jambahybridtransformermambalanguage}.

\begin{figure}[t] 
\centering
\includegraphics[width=\textwidth]{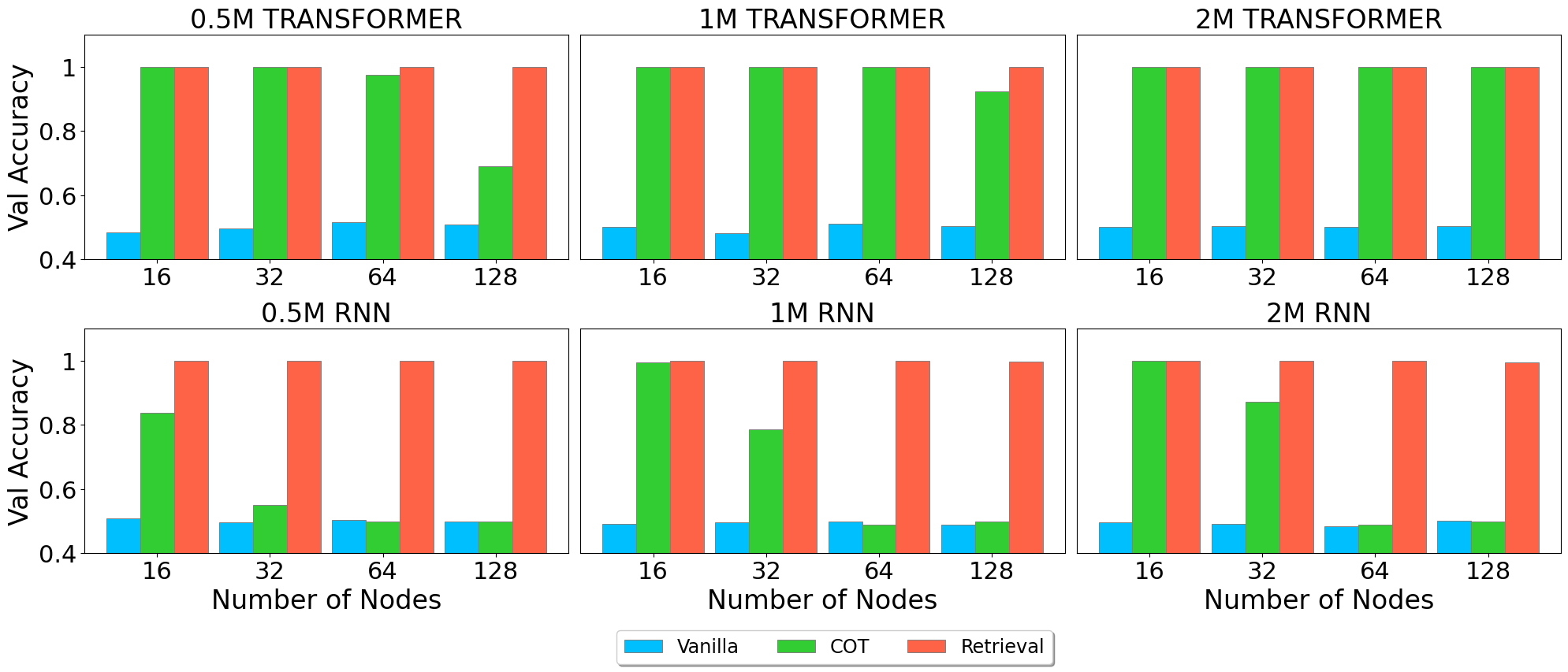} 
\caption{We train RNNs (Mamba) and Transformers (LLaMA 2~\cite{touvron2023llama}) with a frozen word embedding and decoding head of three different model sizes (0.5M, 1M, 2M) on IsTree with three different sizes of graph (16, 32, 64) under three different setups. \textbf{{\textcolor[RGB]{0, 191, 255}{Vanilla}}} means the model directly predicts the label. \textbf{{\textcolor[RGB]{50, 205, 50}{COT}}} means the model will generate a chain-of-thought process based on DFS (see~\Cref{alg:istree}) before prediction.   {\textbf{\textcolor[RGB]{255, 99, 71}{Retrieval}}} means the model will generate the chain of search queries and reasoning before prediction (see~\Cref{alg:istreeretrieval}). We observe that (1) Both Transformer and RNNs can't solve the IsTree question without a chain of thought; (2) RNNs' performance with chain-of-thought decays quickly when the number of nodes increase, which is consistent with our theory; (3) All models reach almost perfect accuracy when enhanced with retrieval.}
%\vspace{-1.5em}
\label{fig:performance_comparison}
\end{figure}

\section{Related Works}
\label{sec:related}

\textbf{State Space Machines and Linear Transformers.} 
There has been a recent surge of interest in state space machines and (kernalized) linear transformers~\citep{gu2022efficiently,katharopoulos2020transformers,peng2023rwkv,sun2023retentive,gu2023mamba,fu2023hungry,poli2023hyena,luo2021stable,peng2021random,wang2020linformer}, which are a class of models that combine the parallelizability of the Transformer with the memory efficiency of the RNN. These models can process both a sequential and a recurrent form, and can use the former for fast parallelizable training and the latter for memory-efficient inference. 
However, these models are still empirically inferior to the Transformer in terms of performance. Our work investigates the reasons behind this gap and proposes to close it by enhancing the in-context retrieval capability.

\textbf{Chain of Thought (CoT).} Chain of thought \citep{wei2023chainofthought,nye2021work,kojima2023large,wang2024chainofthought} is an augmentation to the Transformer, that allows it to solve more complex reasoning tasks by generating a reasoning process before outputting the answer. It has been shown that Transformers with CoT provably have more expressive power than the original Transformer without CoT~\citep{feng2023towards,li2024chain}. However, the expressive power of RNNs with CoT has not yet been systematically studied. Theorem F.1 in~\citet{feng2023towards} shows that RNN cannot output a particular format of CoT for evaluating arithmetic expressions and solving linear equations while Transformers with the same amount of parameters can. Concurrent work~\citep{yang2024efficient} discovers that linear Transformers, a special class of RNNs, are not able to solve some dynamic programming problems with CoT, unless the number of parameters grows with the length of the input. One high-level message our work conveys is similar to theirs: RNNs have limited representation power to perform reasoning with CoT.
However, we show that such limitation is not specific to the output format or architecture and apply tools from streaming complexity to prove lower bounds on a broader range of tasks and memory-efficient architectures.

\textbf{Streaming Algorithms.} Our lower bound leverages the technique in streaming algorithms. Streaming algorithms are algorithms that take constant (typically just 1) pass over the input and use sublinear space, hence including RNNs with fixed state space as a special case. Works in streaming algorithms date back to the 1980s~\citep{munro1980selection} and have been formalized and popularized in the 1990s~\citep{alon1996space} due to the need to process large data streams. {{
The study of streaming algorithm is closely connected with the concept of 
\textbf{communication complexity}. The communication complexity of an algorithm is defined by the amount of communication cost required when the algorithm is distributed, and all streaming algorithms can be viewed as distributed algorithms with sublinear communication complexity, whose communication content is the internal state of the algorithm. The lower bound in our work is a direct application of this observation to the study of RNNs and we mainly consider the lower bounds for (1) indexing the input~\citep{munro1980selection} and (2) determining whether the input is a tree~\citep{henzinger1998computing}.}}

\textbf{Retrieval Augmented Generation.} Our work proposes to use retrieval augmentation to close the representation gap between RNNs and Transformers. This is consistent with the recent trend of retrieval augmented generation~\citep{guu2020realm,borgeaud2022improving,rubin2023longrange}. Empirically, retrieval augmented generation has been shown to improve the performance of recurrent models in various tasks~\citep{kuratov2024search,akyürek2024incontext} and our work provides a theoretical foundation for this phenomenon. Our work also shows that an attention layer can be used to simulate the retrieval process, which is consistent with the finding that attention can improve the performance of RNNs~\citep{vaswani2017attention,arora2023zoology,park2024mamba,peng2023rwkv,hao-etal-2019-modeling}. It has also been shown empirically that attention can be used to simulate complex retrieval process~\citep{jiang2022retrieval}.

\textbf{Comparison Between Transformers and RNNs (Without CoT).} %
A line of works focused on the comparison between RNNs and Transformers in terms of recognizing or generating formal languages~\citep{bhattamishra2020ability,Hahn_2020,merrill2022saturated}. These works show that the lack of recurrent structure in Transformers makes them fail to recognize some formal languages that RNNs can recognize. However,~\citet{liu2023transformers,yao2023selfattention,hao2022formal} show that such limitation can be mitigated when we consider bounded length of input or bounded grammar depth. Our work differs from these works in that we consider the expressive power of RNNs and Transformers with CoT and show that in this case, the gap between RNNs and Transformers is one-sided~(\Cref{thm:trans_beat_rnn}).

Prior work~\citep{arora2023zoology} has shown that input-independent gating SSMs are inferior to Transformers in the task called \textit{associative recall} \citep{willshaw1969non,hopfield1982neural,hinton2014parallel}. The task requires the model to recall a previously seen pattern given a partial input. They show that input-dependent gating SSMs have better performance in associative recall and also propose a hybrid architecture that combines input-independent state space machines with attention to achieve better performance. Our work differs from this work in the following ways: (1) Our work studies associative recall from a theoretical perspective and proves formal lower bounds on the memory size of RNNs necessary for solving associative recall and other retrieval tasks; (2) We also study hybrid architectures but we provide a proof that appending a single Transformer layer to RNNs can make them expressive enough; (3) Our theory applies to not only input-independent gating SSMs but also all RNNs with $o(n)$-bit memory.

Prior work~\citep{jelassi2024repeat} proves a representation gap between RNNs and Transformers in repeating a long sequence, which can be seen as a retrieval task. They show that RNNs have difficulty performing the task due to their limited memory.
Our work further proves that RNNs are limited in solving many other retrieval tasks, even with CoT. Technically, a key ingredient in their proof is a counting argument on the output sequence to show a limited memory size is not enough to produce too many different output sequences, but our proof can handle retrieval tasks that only require outputting a single token.

Notably,~\citet{sanford2023representational} apply communication complexity to prove circuit size or memory size lower bounds for RNNs and Transformers on the task of sparse averaging. 
\citet{sanford2024transformers} extend this technique to another task called hop$_k$, a generalization of the associative recall task. Our technique is similar to theirs since our proof is also based on communication complexity. But we consider a broader range of tasks including seemingly irrelevant reasoning tasks such as IsTree, and further explore various ways to close the representation gap.

\textbf{Representation Theory of RNNs.} 
Another line of works~\citep{li2021on,li2022on,alberti2023sumformer} studies the universal approximation power of RNNs. They show that the upper bound of the approximation power of linear RNNs will be constrained by the dimension of the hidden states. Their works on the high level are consistent with our findings but are not directly comparable because we are considering finite precision compute models with the assistance of CoT or $\ICRA$. 
\section{Preliminaries}

We introduce the definitions that are necessary for understanding our results
and defer other definitions to~\Cref{sec:def}.

\textbf{Vocabulary and Embeddings.}
A vocabulary $\Vocab$ is a finite set of tokens.
A word embedding for a token $v \in \Vocab$ is a vector $\WordEmbedding_v \in \R^{\dimension}$ that represents the token, and a position embedding for the $k$-th token in a sequence is a vector $\PositionEmbedding_k \in \R^{\dimension}$ that represents the position.
Given a sequence of tokens $\TokenSequence$, an embedding function $\Embedding(\TokenSequence)$ maps each token to a vector in $\R^{\dimension}$ by mixing word and position embeddings, resulting in a sequence of vectors.
To ease our comparison between RNNs and Transformers, in this paper, we assume fixed word and position embeddings and do not learn them during training.
See~\Cref{sec:def-models} for the formal definitions. This is common practice and has been used in many previous works studying the theory of Transformers~\citep{li2023transformers,tian2023scan}.

Many of the problems we study involve natural numbers up to $n$, where the input sequence length is linear in $n$. 
For simplicity, we assume the vocabulary contains $[n] = \{1, \dots, n\}$ and the word embedding for $i$ is defined as $\WordEmbedding_i = i \onehot_1$, where $\onehot_1$ is the first coordinate vector.
But in practice, the vocabulary size does not increase with $n$ and numbers may be tokenized into a few tokens according to their decimal representations.
We note that our results can be easily extended to this more practical case since our lower bounds do not rely on the specific form of the vocabulary and embeddings and for the upper bounds, our embedding can be easily represented by a few RNN or Transformer layers.

\textbf{Numerical Precision.}
We will consider computation models with fixed numerical precision in this paper. We will use $\precision$ to denote the precision of the number of bits to represent real numbers and use $\R_\precision$ to denote the set of all real numbers that can be represented by $\precision$-bit floating point numbers. We defer the details to~\Cref{sec:precision}. We will assume $\precision = O(\log n)$ in this paper and state the constant explicitly when necessary. This is a common assumption when studying the finite precision neural networks~\citep{feng2023towards,merrill2023parallelism}. 

\textbf{Language Modeling.}
We use $\Vocab^*$ and $\Vocab^+$ to denote the set of all finite sequences and all non-empty finite sequences of tokens in $\Vocab$, respectively.
We study language models that can predict the next token given a prefix of tokens.
For this, we define a language model (LM) as a function $M: \Vocab^+ \to \Prb_{\Vocab}$ that maps a non-empty sequence to a probability distribution over the next token, where $\Prb_{\Vocab}$ is the probability simplex over $\Vocab$.
We specifically study the case where the language model is realized by deep neural networks: first map the input sequence $\TokenSequence$ into a sequence of embeddings $\Embedding(\TokenSequence)$, and then apply a neural network, such as a Transformer or RNN, to process the embeddings and output the probability distribution. We would call a series of parameterized models with increasing input size a \emph{family} of models.

\textbf{Transformer.}
We will first define the Transformer architecture used in the theoretical analysis in this paper.

\begin{definition}[Transformer Block] \label{def:transformer_block}
    Let $\hiddenstate \in \R^{\dimension \times \position}$ be the input matrix, where $\position$ is the sequence length. The output of a Transformer block $\TransformerBlock$ is defined as:
    \begin{align}
        \TransformerBlock(\hiddenstate) &= 
        \hiddenstate + \Attention(\hiddenstate)+ \FFN(\hiddenstate + \Attention(\hiddenstate)) \nonumber, \\
        \Attention(\hiddenstate) &= \sum_{h = 1}^{\headnumber} \HeadValue \hiddenstate \Softmax\left(\frac{\left(\HeadKey \hiddenstate \right)^\top \HeadQuery \hiddenstate }{\sqrt{d}} + \CausalMask \right),
    \end{align}
    where $\FFN$ is a column-wise ReGLU feed-forward network \footnote{ReGLU means $\sigma(x) = \ReLU(W_1x + b_1) \otimes (W_2x + b_2)$, this is a surrogate for the commonly used SwiGLU activation and allows the model to perform multiplication of two coordinates.} with width $\width$ and output dimension $\dimension$, $\Attention$ is the scaled dot-product attention, $\Softmax$ is the column-wise softmax function, $\HeadKey$, $\HeadQuery$, $\HeadValue$ are the learnable parameters and $\headnumber$ is the number of heads, and $\CausalMask = \begin{bmatrix} 0 & 0 & \ldots &0  \\ -\infty & 0 & \ldots & 0 \\ \vdots&\vdots&\vdots& \vdots \\ -\infty & -\infty & \ldots & 0\end{bmatrix} \in \R^{\position \times \position}$ is a mask to prevent the attention from attending to future tokens.
\end{definition}

In the context of language modeling, given a sequence of tokens $\TokenSequence$,
a Transformer $T(\TokenSequence)$ is defined as:
\begin{definition}[Transformer]
    \label{def:transformer}
        Let $\TokenSequence \in \VocabSize^{\position}$ be the tokenized input sequence, the output of a Transformer is defined as: 
        \begin{align}
            \Transformer(\TokenSequence) = \Softmax \left(\WordEmbedding \left(\TransformerBlock_L\left( \ldots \TransformerBlock_1\left(\Embedding\left(\TokenSequence\right)\right)\right)\right)\right)_{:, \position}.
        \end{align}
        where $\Softmax$ is the column-wise softmax function, $f_i$ is the $i$-th Transformer block. We will call the $i$-th Transformer block the $i$-th layer of the Transformer and denote its feed-forward layer and attention layer as $\FFN_i$ and $\Attention_i$ respectively.
    \end{definition}

\textbf{Recurrent Neural Networks}
Recently there has been a lot of interest in the linear-time Transformer, which replaces the full-attention calculation with linear-time alternatives. These variants are mostly special forms of recurrent neural networks (RNNs) that are parallelizable. 
Here we define a general form of RNNs containing all the common variants to the best of our knowledge,
including Mamba~\citep{gu2023mamba}, RWKV~\citep{peng2023rwkv}, RetNet~\citep{sun2023retentive}, StreamingLLM~\citep{xiao2023efficient}, RMT~\citep{bulatov2022recurrent}, TOVA~\citep{oren2024transformers}, xLSTM~\citep{beck2024xlstm} etc. 
An RNN maintains a state $s_t \in \R_{\precision}^{\hiddendimension}$ storing $\hiddendimension$ $p$-bit floating point numbers.

\begin{definition}[RNN]
\label{def:rnn}
An RNN architecture is characterized by two functions: state transition function $\StateTransition: \varTheta \to \left( \R_{\precision}^{\hiddendimension} \times \R_{\precision}^{\dimension} \to \R_{\precision}^\hiddendimension \right)$ and output function $\StateOutput: \varTheta  \to \left(\R_{\precision}^\hiddendimension \to \R_{\precision}^{\dimension}\right)$, where $\hiddendimension$ is the dimension of the state and $\varTheta$ is the parameter space. Let $\TokenSequence \in \VocabSize^{\position}$ be the input sequence, the output of a recurrent neural network with parameter $\theta \in \varTheta $ is defined as:
\begin{align*}
    \RNN_\theta(\TokenSequence) &= \Softmax\left(\WordEmbedding \StateOutput_\theta(\state_{\position})\right), \\
    \forall k \in [{\position}], \state_{k} &= \StateTransition_\theta(\state_{k - 1}, \Embedding(\TokenSequence)_{:,k}),
\end{align*}
where $\state_0 \in \R_{\precision}^{\hiddendimension}$ is a vector determined by $\theta$ and $\WordEmbedding$ is the word embedding matrix. We will omit the subscript $\theta$ when it is clear from the context. 
\end{definition}

We can characterize the complexity of an RNN architecture with the following three measures,
\begin{enumerate}
    \item Parameter size $\ParameterSize$: the number of bit of parameters determining $\StateTransition$ and $\StateOutput$.
    \item State memory size $\MemorySize$: the number of bits to record the state of the RNN, in this case, is $\hiddendimension \times \precision$.
    \item Circuit size $\CircuitSize$: the number of bit-wise arithmetic operations needed to calculate $\StateTransition$ and $\StateOutput$.
\end{enumerate}

A particularly interesting class of RNNs is constant-size RNNs, where the dimension of state and number of parameters are fixed and do not depend on $n$, i.e.,
$\hiddendimension = O(1)$, $p = O(\log n)$, $\ParameterSize = O(\log n)$, $\MemorySize = O(\log n)$, $C = O(\log n)$.

We do not assume a specific structure of the RNN when we want to prove impossibility results for RNNs, i.e., what RNNs cannot represent. But when we want to showcase what RNNs can represent, we focus on Linear RNNs, one of the simplest form of RNNs, so that our results are more likely to be generalizable to more complex RNNs.

\begin{definition}[RNN block]
\label{def:minimal_rnn_block}
A Linear RNN block is defined as follows:
\begin{align*}
    \RNNBlock(\hiddenstate) = \hiddenstate + \GLU(\hiddenstate) + \FFN(\hiddenstate + \GLU(\hiddenstate)),
\end{align*}
where $\FFN$ is a column-wise ReGLU feed-forward network with width $\width$ and output dimension $\dimension$ and $\GLU$ is a linear unit, defined as 
\begin{align*}
    h_0 &= 0, h_{:,t} = A h_{:, t-1} +  B \hiddenstate_{:,t},        \GLU(\hiddenstate_{:, 1:\position}) = h_{:, 1:\position}.
\end{align*}
\end{definition}

\begin{definition}[Linear RNN]
\label{def:minimal_rnn}
A Linear RNN is a recurrent neural network
\begin{align}
    \RNN(\TokenSequence) = \Softmax \left(\WordEmbedding \left(\RNNBlock_L\left( \ldots \RNNBlock_1\left(\Embedding\left(\TokenSequence\right)\right)\right)\right)\right)_{:, \position}.
\end{align}
where $\Softmax$ is the column-wise softmax function, $f_i$ is the $i$-th Linear RNN block. We will call the $i$-th Linear RNN  block the $i$-th layer of the Linear RNN and denote its feed-forward layer and linear unit layer as $\FFN_i$ and $\GLU_i$ respectively.
\end{definition}

\textbf{Language Models for Algorithmic Problems.}
An algorithmic problem is a problem that may be solved by an algorithm.
In this paper, we focus on algorithmic problems $f: \Vocab^+ \to \VocabAnswer$ that 
asks for computing $f(\InputTokenSequence)$ given a sequence of tokens $\InputTokenSequence$ as the input, where $\VocabAnswer$ is the set of possible answers.
We say that an LM $M$ can (directly) solve an algorithmic task $f$ if, given the sequence $\InputTokenSequence$, the probability distribution $M(\InputTokenSequence)$ for the next token is peaked at the correct output token $f(\InputTokenSequence)$, i.e., $\argmax_{j \in V} M(\InputTokenSequence)[j] = f(\InputTokenSequence)$.

\textbf{Chain-of-Thought Reasoning.}
{\em Chain-of-Thought} (CoT) reasoning allows the LM to produce intermediate steps before the final output.
Following~\citet{feng2023towards,li2024chain}, our paper studies the effectiveness of CoT reasoning in improving the expressiveness of LMs, and we allow the intermediate steps to be an arbitrary sequence of tokens. 
We say that an LM $M$ can solve an algorithmic problem $f$ with CoT if the following process terminates with a sequence ended with $f(\InputTokenSequence)$.
First, let $\TokenSequence_0 = \InputTokenSequence$.
For all $i \ge 0$, decode the next token $\NextToken_i = \argmax_{j \in \Vocab} \model(\TokenSequence_{i})[j]$ from $\model$, and append it to the current sequence $\TokenSequence_{i + 1} = \TokenSequence_{i} \oplus \NextToken_i$.
If $\NextToken_i \in \VocabAnswer$, then the process terminates with $\TokenSequence_{i+1}$ with $i$ steps of CoT; otherwise the process continues.

It is evident that 
if an LM can solve an algorithm problem with $0$ steps of CoT,
then an LM $\model$ can (directly) solve the problem. In this case, we also say that the LM can solve the problem without CoT.

\section{Can CoT improve the Representation Power of RNNs?} 
\label{sec:curse}

In this section, we aim to understand the representation power of RNNs with CoT. We first show the positive result that RNNs with CoT can solve tasks that are impossible for RNNs without CoT fixing the state size. We then proceed to understand whether CoT can make RNNs as expressive as Transformers. We show that, even with CoT, RNNs still struggle to solve problems that explicitly require in-context retrieval and this representation gap propagates to seemingly retrieval-irrelevant reasoning tasks such as IsTree. Finally, we show that this gap is indeed one-sided: there only exist tasks where Transformers require exponentially less parameters than RNNs, but not the other way around.

\subsection{CoT Strictly Improves RNNs}
\label{sec:cot-improves-representation}
On the positive side, we show that CoT broadens the representation power of RNNs under mild complexity assumption.

\begin{restatable}[]{theorem}{thmrnncot}
\label{thm:rnncot}
Assuming $\mathrm{PSPACE} \not\subset \mathrm{P/poly}$, there exists an algorithmic problem such that (1) there exist constant-size Linear RNNs that can solve the problem with polynomial length CoT; and (2) any constant-size regular RNNs cannot solve the problem without CoT.
\end{restatable}

See~\Cref{sec:rnncot} for the proof.
The key insight is that the representation power of RNNs without CoT is limited to shallow circuits of size $\poly(\log n)$, but RNNs with CoT can simulate $O(\log n)$-space Turing machines perfectly with $\poly(n)$ steps.
This result is consistent with previous works on the benefit of CoT for Transformers~\citep{feng2023towards,li2024chain}, which also prove based on mild complexity assumptions that Transformers with CoT have the representation power to simulate polynomial-size circuits to solve all the problems in $\mathrm{P}$ but Transformers without CoT cannot.
Here we rigorously prove that a similar benefit of CoT also applies to RNNs, but a key difference is that CoT cannot boost the representation power of RNNs to simulate every polynomial-size circuit family.

\subsection{CoT Cannot Close the Representation Gap with Transformers} 
\label{sec:representation-gap}

Now we proceed to understand whether CoT can make RNNs as expressive as Transformers. The answer turns out to be negative:
RNNs, even with CoT, struggle to solve very simple algorithmic problems that require the capability of retrieving information from the current context, which we call {\em $\ICR$} for short.
This limitation is directly related to the memory efficiency of RNNs. For a model with at most $o(n)$ bits in memory, we can involve techniques from streaming complexity to prove impossibility results for problems requiring $\ICR$.

\subsubsection{Simple Problems on In-Context Retrieval}
\label{sec:simple-icr}

Here we list several simple algorithmic problems that directly test the in-context retrieval capability of the model,
which turn out to be a good test-bed for understanding the limitations of RNNs compared to Transformers.

\begin{definition}[Index] \label{def:index}
    Index is a problem that given a sequence of tokens with length $n$  and a query token $i \in [n]$, requires the model to output the type of the $i$-th token. 
\end{definition}

\begin{definition}[Associative Recall] \label{def:ar}
Associative Recall (AR) is a problem that given a sequence of tokens with length $n$ consisting of tokens in $[n]$ and a query token $q \in [n]$, requires the model to output the next token of $q$ in the sequence.
\end{definition}

\begin{definition}[$c$-gram Retrieval] \label{def:ngram}
An $c$-gram is a contiguous subsequence of $c$ tokens in a sequence of tokens. $c$-gram retrieval is a problem that given a sequence of tokens with length $n$ and a query $(c-1)$-gram that is the prefix of a $c$-gram in the sequence, requires the model to output the last token of that $c$-gram.
\end{definition}

\begin{definition}[Counting] \label{def:counting}
Counting is a problem that given a sequence of tokens with length $n$, a query token $q \in [n]$, and a query number $t \in \mathbb{N}$, requires the model to output $0$ or $1$ to indicate whether the number of occurrences of $q$ in the sequence is greater than $t$.
\end{definition}

Here, Index and AR are perhaps the most basic problems in retrieval, where Index asks for retrieving a token from the input sequence viewed as a linear array of tokens, and AR asks for retrieving a token from the input sequence viewed as an associative array.
These two problems have been studied extensively by different communities. Index is a classic problem in streaming and communication complexity~\citep{munro1980selection}, known to be impossible to solve with $o(n)$ bits of memory for streaming algorithms. 
AR has been regarded as a fundamental problem that an artificial intelligence system should be able to solve~\citep{willshaw1969non,hopfield1982neural,hinton2014parallel,graves2014neural,ba2016using}. In the context of LLMs, AR has been observed to correlate with in-context learning performance~\citep{elhage2021mathematical} and has also been used extensively as synthetic surrogate tasks for pretraining performance~\citep{fu2023hungry,poli2023hyena,lutati2023focus}.
Besides Index and AR, $c$-gram retrieval is a natural extension of AR to the case where the query key can contain multiple tokens: instead of retrieving a token given a single-token key, $c$-gram retrieval asks for retrieving a token when the given key is a $(c-1)$-gram. This task has been studied empirically, but not theoretically in~\cite{jelassi2024repeat}.
Counting is a problem that asks for the number of occurrences of a token, thereby testing the model's capability to retrieve some statistics of relevant information from the input sequence.

The following theorems show that constant-size RNNs cannot solve any of the four tasks when the context length is long, while constant-size Transformers can solve them perfectly.
\begin{restatable}[]{theorem}{thmindex}
\label{thm:index}
For task $T \in \{\text{Index, AR, c-gram retrieval, Counting}\}$, there exist constant-size Transformers that can solve $T$.
On the other hand, any RNN with $o(n)$-bit memory cannot solve $T$ of size $n$ with any length of CoT for large enough $n$.
\end{restatable}

We note that \Cref{thm:index} does \textbf{not} imply that RNNs are incapable of in-context retrieval at all. Instead, it states that the maximal context length that RNNs can effectively retrieve from is linear in its state size. Although for a short context window, RNNs can be trained to perform in-context retrieval (see e.g. \cite{arora2023zoology,gu2023mamba}), this limitation in retrieval capabilities has been empirically observed: for example in \cite{waleffe2024empiricalstudymambabasedlanguage}, both pretrained Mamba and Mamba-2 7B models are shown to have significantly worse Phonebook-retrieval capabilities on 1K context length than Transformers with the same size and trained on the same data.

\textbf{Proof Sketch.}
The key idea of the lower bound of the proof is to put RNNs into the framework of communication complexity and use information-theoretic arguments to prove a lower bound. RNNs have the following property if party $\mathrm{A}$ simulates the RNN on the first part of the input and sends the state to party $\mathrm{B}$, then party $\mathrm{B}$ can simulate the RNN on the second part of the input with the state received from party $\mathrm{A}$ to recover the output of the RNN perfectly. Hence, in the above two theorems, if the RNN can solve the problem with $o(n)$ input size, then the information about the input can be compressed to $o(n)$ bit to produce the output, which contradicts the information-theoretic lower bound. 

For the upper bound, we show that the Transformer can solve the problem by utilizing an attention mechanism called $\Match$ that takes the query token and attends to previous keys that match the query token on certain predefined coordinates. This mechanism allows the Transformer to read its context window like a key-value dictionary and hence can solve the problems perfectly. For the counting problem, we additionally use a $\COUNT$ attention mechanism that counts the number of occurrences of the queried token by attending evenly to each appearance of the queried token. \qed

\subsubsection{Understanding the Representation Power of RNNs Beyond Simple $\ICR$ Problems}
\label{sec:istree}
A natural question would be if an algorithmic problem does not directly test the in-context retrieval capability, can we hope that RNNs would have the representation power to solve it?
Do RNNs and Transformers have the same representation power in this case? We show that the limited memory size in RNNs can still be a bottleneck in solving algorithmic problems.
Even if the retrieval capability is not explicitly tested in an algorithmic problem, it may still be required implicitly for reasoning about the answer.

We demonstrate this gap on a minimal example of algorithmic problems, called $\IsTree$: given an undirected graph $G$ of $n$ nodes, determine whether $G$ is a tree, i.e., whether every pair of nodes is connected by exactly one simple path. A classical solution to IsTree is running Depth First Search (DFS), which takes $O(n)$ time.

In the context of language modeling, we can write the graph $G$ as a sequence of tokens, and then the task of IsTree is to determine whether $G$ is a tree by predicting a $\YES/\NO$ token with or without CoT. We use the following tokenization for the graph $G$:
\begin{align}
    \Tokenize(\graph) = \{\StartSentence, u_1, \sim, v_1, u_2, \sim v_2, \ldots, u_\edgenumber, \sim, v_\edgenumber\},
\end{align}
where $\StartSentence$ and $\sim$ are two special tokens representing the start of the sentence and an edge, and $u_i, v_i$ are numbers denoting the nodes of the graph.

Our result states that RNN with $o(n)$ bit memory cannot solve IsTree, even with an arbitrary choice of chain of thought. On the other hand, there exists a Transformer with constant size and $O(\log n)$ precision that can generate a chain-of-thought of length $n$ following DFS and perfectly solve the same question.
\begin{restatable}[]{theorem}{thmistree}
    \label{thm:rnn_trans_istree}%
    There exist constant-size Transformers that can solve IsTree with CoT of length $O(n)$. On the other hand, any RNN with $o(n)$-bit memory cannot solve IsTree with any length of CoT.
\end{restatable}

\textbf{Proof Sketch.} The main idea of the proof is that the task of IsTree requires the model to reason about the global structure of the graph, which is beyond the capability of RNNs with limited memory. We prove the lower bound by constructing a graph from a binary sequence and showing that RNNs with $o(n)$ memory cannot solve the problem by a reduction to an information-theoretic lower bound. For the upper bound, we show that the Transformer can simulate the DFS algorithm by outputting the Euler tour of the connected components of vertex $1$ and then check the length of the Euler tour with its capability of $\ICR$. %

The key idea of the lower bound of the proof is to again utilize the information-theoretic lower bound. This idea lies in the core of streaming complexity literature and investigation on the IsTree problem dates back to~\cite{henzinger1998computing}. We hereby restate the proof for completeness. %
Given any binary sequence $x$ of length $n - 2$ and an index $k \in [n - 3]$, we can construct a graph as follows: the graph has $n$ nodes, and vertex $a$ is connected to vertex $x_{a} + n - 1$ for any $a \in [n - 2]$. Moreover, vertex $k$ is connected to vertex $k + 1$. The graph is a tree if and only if $x_{k} \neq x_{k + 1}$. 
\begin{figure}[h]
\centering
\begin{tikzpicture}
    \node[circle,draw] (1) at (0,0) {$v_1$};
    \node[circle,draw] (2) at (2,0) {$v_2$};
    \node[circle,draw] (3) at (4,0) {$v_3$};
    \node[circle,draw] (4) at (6,0) {$v_4$};
    \node[circle,draw] (5) at (2,-2) {$v_5$};
    \node[circle,draw] (6) at (4,-2) {$v_6$};

    \draw (1) -- (5);
    \draw (2) -- (6);
    \draw (3) -- (5);
    \draw (4) -- (6);
    \draw (2) -- (3);

    \path (2) -- coordinate[midway] (mid) (3);
    \node[above=1mm of mid] (abovemid) {};

    \path (4) -- coordinate[midway] (rightmid) (6);

    \node[above=2cm of mid] (label) {Binary Message: 0101, Index: $2$};

    \draw[line width=0.2mm, -{Stealth[length=2mm, open]}] (label) -- (abovemid);

    \node[below=2.5cm of mid] (graph) {$\graph$};

    \node[right=-1cm of rightmid] (tokenized) {
        \begin{minipage}{.5\textwidth}
            \begin{align*}
                &\Tokenize(\graph) \\
                =&\StartSentence,2,\sim,6,1,\sim,5, \\
                &3,\sim,5,4,\sim,6,2,\sim,3.
            \end{align*}
        \end{minipage}
    };

    \node[circle,draw, right=9cm of 1] (A) {$v_1$};
    \node[circle,draw] (B) at ([shift={(2,0)}]A) {$v_2$};
    \node[circle,draw] (C) at ([shift={(2,0)}]B) {$v_3$};
    \node[circle,draw] (D) at ([shift={(2,0)}]C) {$v_4$};
    \node[circle,draw] (E) at ([shift={(0,-2)}]B) {$v_5$};
    \node[circle,draw] (F) at ([shift={(0,-2)}]C) {$v_6$};

    \draw[->, red] (A) -- (E) node[pos=0.2, above] {1} node[pos=0.8, below, blue] {10};
    \draw[->, red] (E) -- (C) node[pos=0.1, above] {2} node[pos=0.9, below, right=0.7cm, blue] {9};
    \draw[->, red] (C) -- (B) node[pos=0.2, above] {3} node[pos=0.8, above=0.5cm, blue] {8};
    \draw[->, red] (B) -- (F) node[pos=0.2, above] {4} node[pos=0.8, below, blue] {7};
    \draw[->, red] (F) -- (D) node[pos=0.1, above] {5} node[pos=0.9, below, blue] {6};
    \draw[->, blue] (D) to [out=225, in=315] (F);
    \draw[->, blue] (F) to [out=135, in=225] (B);
    \draw[->, blue] (B) to [out=45, in=135] (C);
    \draw[->, blue] (C) to [out=315, in=45] (E);
    \draw[->, blue] (E) to [out=135, in=225] (A);

    \path (B) -- coordinate[midway] (pathmid) (C);
    \node[below=2.5cm of pathmid] (tourgraph) {DFS on $\graph$};

    \node[above=4mm of pathmid] (abovepathmid) {};

    \node[above=2cm of pathmid] (tourlabel) {Euler Tour: $1, 5, 3, 2, 6, 4, 6, 2, 3, 5, 1$, Is Tree: Yes};

    \draw[line width=0.2mm, -{Stealth[length=2mm, open]}] (abovepathmid) -- (tourlabel);

\end{tikzpicture}
\caption{An example of the graph constructed from the binary sequence $x = 0101$ and the index $k = 2$ and the corresponding DFS tour.}
\label{fig:istree}
\end{figure}
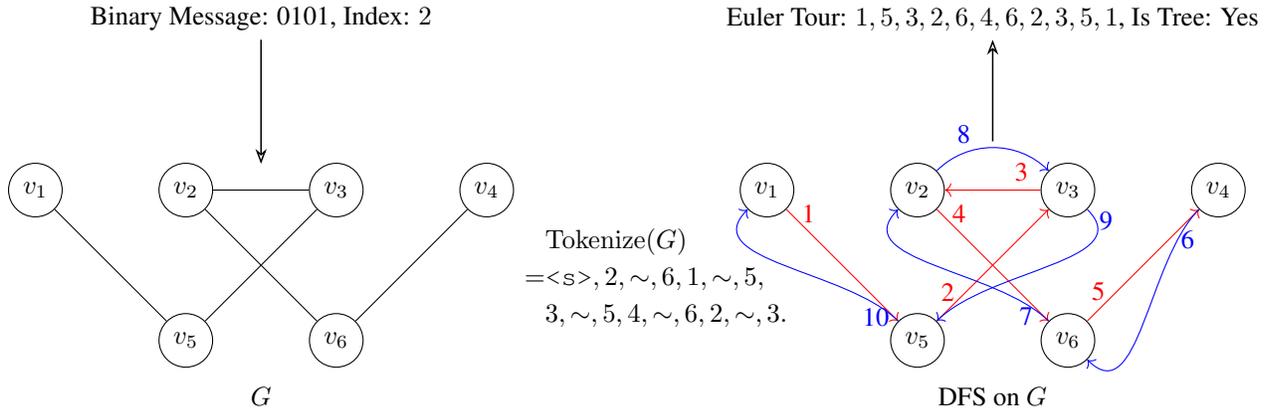
Now assuming there is an RNN with $o(n)$ memory that can solve IsTree, consider two parties $\mathrm{A}$ and $\mathrm{B}$ each holding the sequence $x$ and the index $k$, they can construct two parts of the graph using their information and then $\mathrm{A}$ can simulate RNN on the first part of the graph and send the state to $\mathrm{B}$, and $\mathrm{B}$ can simulate RNN (potentially with CoT) on the second part of the graph to recover the output of the IsTree problem, which is equivalent to whether $x_{k} \neq x_{k + 1}$. However, note that $k$ is never sent to $\mathrm{A}$, and hence actually $\mathrm{B}$ can get whether $x_{k} \neq x_{k + 1}$ for any $k \in [n - 3]$, which contradicts the information-theoretic lower bound.

Now for the upper bound, we will let the Transformer simulate the DFS algorithm by outputting the Euler tour of the connected components of vertex $1$ and then check the length of the Euler tour (see~\Cref{alg:istree}). To simulate the tour, we will implement two functions through the Transformer Block:
\begin{enumerate}[leftmargin=*]
    \item Given a prefix of the tour, find the parent of the last vertex in the tour. This can be implemented by copying each token's predecessor's type to that token and then using the $\Match$ mechanism to match the first occurrence of the current token in the sequence.
    \item Given the tokenized input graph and an edge $(u, v)$, find the next edge after $(u, v)$ containing $v$ in the input graph. We will use another attention mechanism called $\COUNT$ to count, for each edge $e = (a, b)$ in tokenized input graph, the number of occurrences of $a$ and $b$ up to that edge and store $1/(n_{e, a} + 1)$ and $1/(n_{e,b} + 1)$ in the token corresponding to the edge, where $n_{e, a}$ and $n_{e,b}$ are the corresponding counts. Then given the edge $(u, v)$, we can use the $\Match$ mechanism to find $1/(n_{(u,v), v} + 1)$. Then we will use a feed-forward layer with gated relu activation, constant depth, and constant width to approximate $1/(n_{(u,v), v} + 2)$ and then use the $\Match$ mechanism again to find the next edge containing $v$.
\end{enumerate}

Through the above two functions, the Transformer can simulate the DFS algorithm and hence solve the IsTree problem perfectly. \qed

\subsubsection{Transformers are Strictly More Expressive Than RNNs}

The above theorems show the existence of tasks where Transformers require exponentially less memory than RNNs. However, they have not rule out the possibility that there exists a corresponding task where the Transformer will be more redundant and require exponentially more parameters than RNNs. However, the following theorem confirms that such a task doesn't exist for constant-size RNN. 

The theorem is in the same vein as the recent work on the CoT for Transformer~\citep{li2024chain}, which shows the constant size and constant precision Transformer with a polynomial-size position embedding can simulate any polynomial size circuit. The major difference of our theorem is that (1) we consider a Transformer with fixed word and position embedding, hence allowing the parameter number to be logarithmic in the input size, and (2) we consider simulating RNNs, which is a special kind of circuit family and hence we can use more succinct representation utilizing the structural property attached to the recursive process.

\begin{restatable}[]{theorem}{thmtransrnn} 
\label{thm:trans_beat_rnn}
Given input length $n$,
let $R$ is an RNN with word embedding $\WordEmbedding \in \R_{\precision}^{ (n + \numspecial) \times \dimension}$, where $\precision = \Theta(\log n)$ is the precision, the constant $\numspecial$ is the number of special symbols in the vocabulary, the constant $d$ is the embedding dimension.
If each recurrent iteration can be computed by a circuit of size $\CircuitSize(n) \le 2^{\precision / 2}$, and if the RNN produces the final answer after running at most $n^A$ steps of CoT for some constant $A > 0$, then there exist Transformers with $O(\log n)$-bit precision, $O(\CircuitSize(n))$ parameters and word embedding $\left[\WordEmbedding~~\mathbf{0}^{(n + \numspecial) \times \dimension}\right]$
that can produce the same final answer after running $(\CircuitSize(n) + 1)n^A$ steps of CoT.
\end{restatable}

\textbf{Proof Sketch.}  The key idea is to encode the RNN circuit into the Transformer's weight and simulate the circuit gate by gate, utilizing the $\Match$ mechanism to fetch the input of each gate. In this manner, although the naive circuit to simulate the RNN for $n^A$ steps would require $O(n^{A})$ parameter, the Transformer only needs to store one instance of the RNN circuit in its weight and hence we only need $O(P(n) \log \max\{P(n), n\})$ parameter, which is at most polynomial in the parameter size of the RNN. \qed
\section{Enhancing the $\ICR$ Capability Closes the Representation Gap} \label{sec:fix}

In~\Cref{sec:representation-gap}, we show that RNNs are deficient at $\ICR$, hence leading to a significant representation gap with Transformers.
In this section, we aim to understand: if we enhance the $\ICR$ capability of RNNs, do RNNs remain to have any representation gap with Transformers?
We answer this question by examining two representative approaches to enhance the $\ICR$ capability, one explicit and one implicit, and show that both ways can close the representation gap between RNNs and Transformers in solving algorithmic problems.

\begin{figure}[htbp]
    \centering
    \begin{minipage}{0.59\textwidth}
        \centering
        \begin{tikzpicture}[
    scale=0.2, 
    node distance=0.5cm and 0.25cm,
    mynode/.style={draw,thick,inner sep=0.5cm, fill=blue!20},
    mylargenode/.style={draw,thick,inner sep=0.5cm, fill=green!30},
    myarrow/.style={-Stealth, thick, color=black},
    every node/.style={scale=0.25, font=\Huge}
]

\def\colorslist{{"red", "green", "blue", "orange", "purple", "cyan", "magenta", "yellow"}}

  \node[mynode, fill = cyan] (rnn1) {RNN};
  \node[right=0.25cm of rnn1] (dots1) {\ldots};
  \node[mynode, right=of dots1, fill = cyan] (rnn2) {RNN};
  \node[mynode, right=of rnn2, fill = cyan] (rnn3) {RNN};
  \node[mynode, right=of rnn3, fill = cyan] (rnn4) {RNN};
  \node[right=0.25cm of rnn4] (dots2) {\ldots};
  \node[mynode, right=of dots2, fill = cyan] (rnn5) {RNN};
  \node[mynode, right=of rnn5, fill = cyan] (rnn6) {RNN};
  \node[mynode, right=of rnn6, fill = cyan] (rnn7) {RNN};
  \node[right=0.25cm of rnn7] (dots3) {\ldots};
  \node[mynode, right=of dots3, fill = cyan] (rnn8) {RNN};

  \node[mynode, above=of rnn1] (context1) {};
  \node[mynode, right=of context1, above=of rnn2] (context2) {};
  \node[mynode, right=of context2, above=of rnn3] (context3) {};
  \node[right=of context3, above=0.6cm of rnn4] (context4) {StartSearch};
  \node[mynode, right=of context4, above=of rnn5] (context5) {};
  \node[right=of context5, above=0.6cm of rnn6] (context6) {EndSearch};
  \node[mynode, right=of context3, above=of rnn7] (context7) {};
  \node[mynode, right=of context7, above=of rnn8] (context8) {};

  \node[mynode, below=of rnn2] (output1) {};
  \node[below=of rnn3] (output2) {StartSearch};
  \node[mynode, right=of output2, below=of rnn4] (output3) {};
  \node[below=of rnn5] (output4) {EndSearch};
  \node[mynode, right=of output4, below=of rnn8] (output7) {};

  \draw[myarrow] (context1) -- (rnn1);
  \draw[myarrow] (context2) -- (rnn2);
  \draw[myarrow] (context3) -- (rnn3);
  \draw[myarrow] (context4) -- (rnn4);
  \draw[myarrow] (context5) -- (rnn5);
  \draw[myarrow] (context6) -- (rnn6);
  \draw[myarrow] (context7) -- (rnn7);
  \draw[myarrow] (context8) -- (rnn8);

  \draw[myarrow] (rnn2) -- (output1);
  \draw[myarrow] (rnn3) -- (output2);
  \draw[myarrow] (rnn4) -- (output3);
  \draw[myarrow] (rnn5) -- (output4);
  \draw[myarrow] (rnn8) -- (output7);

  \draw[myarrow] (rnn1) -- (dots1);
  \draw[myarrow] (dots1) -- (rnn2);
  \draw[myarrow] (rnn2) -- (rnn3);
  \draw[myarrow] (rnn3) -- (rnn4);
  \draw[myarrow] (rnn4) -- (dots2);
  \draw[myarrow] (dots2) -- (rnn5);
  \draw[myarrow] (rnn5) -- (rnn6);
  \draw[myarrow] (rnn6) -- (rnn7);
  \draw[myarrow] (rnn7) -- (dots3);
  \draw[myarrow] (dots3) -- (rnn8);

  \draw[dotted, myarrow] (output1) -- (context3);
  \draw[dotted, myarrow] (output2) -- (context4);
  \draw[dotted, myarrow] (output4) -- (context6);

  \draw[decorate, decoration={brace, amplitude=5pt}, thick] 
    ([yshift=2ex]context1.north west) -- ([yshift=2ex]context3.north east) 
    node[midway, above=5pt] (context) {Context};

  \draw[decorate, decoration={brace, amplitude=5pt}, thick, right=of context] 
    ([yshift=2ex]context4.north west) -- ([yshift=2ex]context6.north east) 
    node[midway, above=5pt] (query) {Search Query};

  \draw[decorate, decoration={brace, amplitude=5pt}, thick] 
    ([yshift=2ex]context7.north west) -- ([yshift=2ex]context8.north east) 
    node[midway, above=5pt] (result){Search Result};
    
  \node[mylargenode, above=of query] (retriever) {Retriever};

  \draw[myarrow] (query) -- (retriever);

  \draw[myarrow] (context) -- (context |- retriever.center) -- (retriever.west);
  \draw[myarrow] (retriever.east) -- (result |- retriever.center) -- (result);

\end{tikzpicture} %
        \caption*{\quad \quad \quad The retrieval augmented RNN}
        \label{fig:rnn}
    \end{minipage}%
\begin{minipage}{0.39\textwidth}
    \centering
    \begin{tikzpicture}[
  scale=0.2, 
  node distance=0.73cm and 0.3cm,
  mynode/.style={draw,thick,inner sep=0.5cm, fill=blue!20},
  mylargenode/.style={draw,thick,inner sep=0.9cm, fill=green!30},
  myarrow/.style={-Stealth, thick, color=black},
  every node/.style={scale=0.25, font=\Huge}
]
\def\colorslist{{"red", "green", "blue", "orange", "purple", "cyan", "magenta", "yellow"}}

  \node[mynode, fill = cyan] (rnn1) {RNN};
  \node[right=0.5cm of rnn1] (dots1) {\ldots};
  \node[mynode, right=of dots1, fill = cyan] (rnn2) {RNN};
  \node[mynode, right=of rnn2, fill = cyan] (rnn3) {RNN};
  \node[right=0.5cm of rnn3] (dots2) {\ldots};
  \node[mynode, right=of dots2, fill = cyan] (rnn4) {RNN};
  \node[mynode, right=of rnn4, fill = cyan] (rnn5) {RNN};

  \node[mynode, above=of rnn1] (context1) {};
  \node[mynode, right=of context1, above=of rnn2] (context2) {};
  \node[mynode, right=of context2, above=of rnn3] (context3) {};
  \node[mynode, right=of context3, above=of rnn4] (context4) {};
  \node[mynode, right=of context4, above=of rnn5] (context5) {};

  \node[mynode, below=of rnn1] (fakeoutput0) {};
  \node[mynode, below=of rnn2] (fakeoutput1) {};
  \node[mynode, below=of rnn3] (fakeoutput2) {};
  \node[mynode, below=of rnn4] (fakeoutput3) {};
  \node[mynode, below=of rnn5] (fakeoutput4) {};

  \node[mynode, below=of fakeoutput1] (output1) {};
  \node[mynode, below=of fakeoutput2] (output2) {};
  \node[mynode, below=of fakeoutput3] (output3) {};
  \node[mynode, below=of fakeoutput4] (output4) {};

  \draw[myarrow] (context1) -- (rnn1);
  \draw[myarrow] (context2) -- (rnn2);
  \draw[myarrow] (context3) -- (rnn3);
  \draw[myarrow] (context4) -- (rnn4);
  \draw[myarrow] (context5) -- (rnn5);

  \draw[myarrow] (rnn1) -- (fakeoutput0);
  \draw[myarrow] (rnn2) -- (fakeoutput1);
  \draw[myarrow] (rnn3) -- (fakeoutput2);
  \draw[myarrow] (rnn4) -- (fakeoutput3);
  \draw[myarrow] (rnn5) -- (fakeoutput4);

  \draw[myarrow] (fakeoutput1) -- (output1);
  \draw[myarrow] (fakeoutput2) -- (output2);
  \draw[myarrow] (fakeoutput3) -- (output3);
  \draw[myarrow] (fakeoutput4) -- (output4);

  \draw[myarrow] (rnn1) -- (dots1);
  \draw[myarrow] (dots1) -- (rnn2);
  \draw[myarrow] (rnn2) -- (rnn3);
  \draw[myarrow] (rnn3) -- (dots2);
  \draw[myarrow] (dots2) -- (rnn4);
  \draw[myarrow] (rnn4) -- (rnn5);

  \draw[dotted, myarrow] (output1) -- (context3);
  \draw[dotted, myarrow] (output3) -- (context5);

  \path (rnn1) -- coordinate[midway] (mid) (rnn5);

  \path let \p1 = ($(rnn5.west) - (rnn1.east)$) in
  node[draw, thick, below=of mid, minimum width=\x1, inner sep=0.5cm, label=center:Attention + MLP, fill=orange!30, yshift=-0.8cm] (attention) {};

\end{tikzpicture} %
    \caption*{The hybrid architecture}
    \label{fig:hybrid}
\end{minipage}
\caption{\textbf{$\ICRA$.} 
The retrieval augmented RNN (left) and the hybrid architecture (right) close the representation gap between RNNs and Transformers.}
\end{figure}

\subsection{Explicit Retrieval Through Regular Expression} \label{sec:icrag}

First, we explore the power of RNNs with {\em Retrieval Augmented Generation} (RAG),
which gives an LM the capability to retrieve relevant information to assist generation. In our context, we are specifically interested in allowing LMs to call functions to retrieve information from their context, which we call {\em $\ICR$ Augmented Generation} ($\ICRA$).

We will first show that adding function calls to associative recall is not enough to close the representation gap between RNNs and Transformers. 

\begin{proposition}
\label{prop:ar_not_enough}
For any RNN family with $O(\log n)$ bit memory and $O(\log n)$ parameter with an oracle to receive results for the AR problem (\Cref{def:ar}) for any queries, for large enough $n$, the RNN can't solve the index problem (\Cref{def:index}) with length $n$ in any CoT steps.
\end{proposition}

\begin{proof}
Consider a special type of index problem where every token at the even position of the input sequence is a special token $\specialtoken$ and the rest of the tokens are uniformly random. Then the oracle for the AR problem can be simulated by the RNN by simply outputting the $\specialtoken$ when the query is not $\specialtoken$ and outputting the third token when the query is $\specialtoken$. However, following similar proof of~\Cref{thm:index}, we can show that the RNN can't solve this special form of index problem with length $n$ in any CoT steps.
\end{proof}

In this light, we need to consider a more general form of $\ICR$ capability. We specifically consider a special form of $\ICRA$ that enables an LM to perform regular expression matching because the regular expression is a flexible primitive that can be used to describe a wide range of retrieval tasks and can be implemented efficiently on modern hardware.

Given an LM $\model$ with vocabulary $\Vocab$ (containing two additional special tokens, $\StartSearch$ and $\EndSearch$) and the tokenized input sequence $\InputTokenSequence \in \VocabSize^{\inputlength}$, the LM $\model$ with $\ICRA$ generates following sequence of tokenized sequence:
\begin{align*}
    \TokenSequence_0 &= \InputTokenSequence, 
    \qquad \NextToken_i = \argmax_{j \in \Vocab} \model(\TokenSequence_{i})[j], \\
    \TokenSequence_{i + 1} &= \begin{cases} \TokenSequence_{i} \oplus \NextToken_i,        &\quad \text{if }\NextToken_i \neq \EndSearch \\
    \TokenSequence_{i} \oplus \NextToken_i \oplus \RETRIEVE\left(\TokenSequence_i\right) , &\quad \text{otherwise.}
    \end{cases}
\end{align*}
Here $\RETRIEVE$ looks for the last occurrence of $\StartSearch$ at position $\StartPos$ and $\EndSearch$ in $\TokenSequence$ at position $\EndPos$ and treat $\Detokenize(\TokenSequence_{\StartPos: \EndPos})$ as a regular expression, where $\Detokenize$ maps the tokenized sequence back to the string, inserting a space between every pair of adjacent tokens. The algorithm then runs a regular expression matching on $\Detokenize(\TokenSequence_{1:\StartPos - 1})$, finds the first matching substring, and returns the first capturing group according to the regular expression (i.e., content embraced by a pair bracket in the regular expression). 
While there are many grammar standards of regular expressions, we adhere to the standard specified in the \texttt{re} library of Python. That is, we evaluate the following Python code to get the result of the regular expression matching:
\begin{center}
    \begin{tabular}{c}
\begin{lstlisting}[language=Python]
re.search(pattern, string).group(1)
\end{lstlisting}
\end{tabular}
\end{center}
where $\Detokenize(\TokenSequence_{\StartPos: \EndPos})$ is the pattern and $\Detokenize(\TokenSequence_{1:\StartPos - 1})$ is the string.

The following theorem shows that $\ICRA$ with regular expressions is powerful enough for RNNs to solve the $\ICR$ problems in~\Cref{sec:simple-icr} with CoT.

\begin{restatable}[]{theorem}{thmindexre}
\label{thm:index_re}
For task~$T \in \{\text{Index, AR, c-gram retrieval, Counting, IsTree}\}$, there exist constant-size Linear RNNs that can solve $T$ with $\ICRA$. For $T$ other than IsTree, $O(1)$ steps of CoT is required and for IsTree, $O(n \log n)$ steps of CoT is required.
\end{restatable}

\textbf{Proof Sketch.} For the Index problem, let the RNN output the regular expression {\verb|^(?:\S\s*){|\regexparam{$a$}\verb|}(\S)|}, where $\regexparam{a} = k - 1$. For AR, let the RNN output {\verb|\b|\,\regexparam{$q$}\,\verb|\b(\S+)\b|}, where \regexparam{$q$} is the number in query. For $c$-gram retrieval, let the RNN output {\verb|\b|\,\regexparam{$q_1$}\,\ldots\,\regexparam{$q_{c-1}$}\,\verb|\b(\S+)\b|}, where \regexparam{$q_i$} is the $i$-th number in the query. For Counting, let the RNN output {\verb|(\b|\,\regexparam{$v$}\,\verb|\b){|\regexparam{$k+1$}\verb|}|}, where \regexparam{$v$} is the query token and \regexparam{$k$} is the query threshold.  \qed

Beyond these simple $\ICR$ problems, in~\Cref{thm:rnn_trans_istree}, we have shown that RNNs cannot solve IsTree due to its implicit requirement of $\ICR$ capability. We now show that $\ICRA$ can help linear RNN solve IsTree in $O(n)$ CoT steps. See \Cref{sec:proof_rag_istree} for the proof.
\begin{restatable}{theorem}{thmragistree}
\label{thm:rag_istree}
There exists a Linear RNN family with $O(\log n)$ bit memory and $O(\log n)$ parameter, that can solve IsTree of size $n$ with $\ICRA$ in $O(n)$ CoT steps.
\end{restatable}

To further show the power of the explicit retrieval, the following theorem further proves a general result, showing that $\ICRA$ empowers RNNs with $O(\log n)$ bit memory to simulate polynomial-time Turing machines.

\begin{restatable}[]{theorem}{rnnturing}
\label{thm:rnn_turing}
Given $A, B > 0$, for all polynomial-time Turing machines $T \in \Time(n^A)$ with $B$ states and vocabulary size $B$, there exist Linear RNNs with $B$ special symbols, $O(A\log n)$-bit precision and memory, and $O(AB^2)$ parameters that can output the result of $T$ by running $O(n^A)$ steps of CoT with $\ICRA$.
\end{restatable}
\textbf{Proof Sketch.} The retrieval augmented RNN can simulate the Turing machine by maintaining the state of the Turing machine and the position of the tape head in its memory and writing down the tape in the context in the form of $c\ v\ c$ to indicate the value of the tape at position $c$ is updated to $v$. Given the input tape $\tape[1,i]$, the retrieval augmented RNN will first write down the initial tape in the previous form of $i\ \tape[1,i]\ i$ using the regular expression used in the Index problem.
The RNN can then generate the search queries with forms of {\verb||\regexparam{$p$}\verb| (.) |\regexparam{$p$}\verb| .*?$|} with \regexparam{$p$}  being the pointer to retrieve the information from the context.

By introducing $\ICRA$, RNNs are able to use regular expressions to retrieve information from a distance and put relevant information together.
Then RNNs can use perform more complex operations locally and output tokens to be retrieved later. \qed

As a final note, our focus here is to understand the representation power of RNNs given an appropriate RAG, but not to propose a method that immediately leads to practical applications.
While the above results show that $\ICRA$ can close the representation gap between RNNs and Transformers in solving algorithmic problems, a limitation here is that $\ICRA$ is not an immediate practical solution, as there is no existing training data for this $\ICRA$.

\subsection{Implicit Retrieval by Appending Just One Transformer Layer}

Since~\cite{bahdanau2016neural}, attention mechanisms have been understood as a form of compensation for the fixed memory size of RNNs, allowing the model to attend to the entire context. We show in this section formally that this form of implicit retrieval can close the representation gap between RNNs and Transformers in solving algorithmic problems. We consider the following hybrid architecture, which combines the RNN and the Transformer by appending a single Transformer layer to the RNN output.

\begin{definition}[Hybrid RNN]
\label{def:hybrid}
A hybrid RNN is a model that consists of an RNN with transition and output function $\StateTransition, \StateOutput$ and one Transformer layer $\TransformerBlock$, the output of the RNN is used as the input of the Transformer layer and the output of the Transformer layer is used to produce the next token. Concretely, given the input sequence $\InputTokenSequence$, the output of the hybrid architecture is:
\begin{align*}
    \Hybrid(\TokenSequence) &= \Softmax\left(\WordEmbedding\TransformerBlock\left([\StateOutput(\state_{\iterpos})]_{\iterpos \in [\position]}\right)\right)_{:, \position}, \\
    \forall k \in [\position], \state_{k} &= \StateTransition(\state_{k - 1}, \Embedding(\TokenSequence)_{:,k}),
\end{align*}
\end{definition}

First, we show that hybrid RNNs can solve the $\ICR$ problems in~\Cref{sec:simple-icr} without CoT.

\begin{restatable}[]{theorem}{thmhybridindex}
\label{thm:hybrid_index}
For task $T \in \{\text{Index, AR, c-gram retrieval, Counting, IsTree}\}$, there exist constant-size hybrid Linear RNNs that can solve $T$. For $T$ other than IsTree, no CoT is required and for IsTree, $O(n \log n)$ steps of CoT is required.
\end{restatable}

\textbf{Proof Sketch.} The proof is similar to the proof of~\Cref{thm:index}, using the appended Transformer layer to simulate the $\Match$ function and $\COUNT$ function in the RNN.\qed

Similar to the situation in~\Cref{sec:icrag}, the implicit retrieval method can enpower the hybrid linear RNN to solve IsTree with CoT. See~\Cref{sec:thmhybridistree} for the proof.

\begin{restatable}[]{theorem}{thmhybridistree}
\label{thm:hybrid_istree}
There exists a hybrid Linear RNN with $O(\log n)$ bit memory and $O(\log n)$ parameter, that can solve IsTree of size $n$ with a chain of thought of length $O(n \log n)$. 
\end{restatable}

Further, we show that this hybrid architecture with only one attention block is powerful enough to even simulate any polynomial-time Turing machine with CoT.
\begin{restatable}[]{theorem}{hybridturing}
    \label{thm:hybrid_turing}
Given any constant $A, B$, for any polynomial-time Turing machine $T \in \Time(n^A)$ with $B$ states and vocabulary size $B$, there exists a hybrid Linear RNN (see~\Cref{def:hybrid}) with vocabulary of $B$ special symbol, $O(A\log n)$ bit precision and memory, and $O(AB^2\log n)$ bit parameters, that can simulate the result of $T$ on any input with length $n$ in $O(n^A)$ CoT steps. 
\end{restatable}
\textbf{Proof Sketch.} The proof is similar to the proof of~\Cref{thm:rnn_turing}. Instead of using regular expressions to retrieve the information from the context, the hybrid architecture can use the attention mechanism in the Transformer layer to implement the $\Match$ function to retrieve the information from the context. \qed

\section{Empirical Validation} \label{sec:validation}
We validate our theoretical findings through synthetic and natural language experiments: CoT alone cannot close the performance gap between RNNs and Transformers, but enhancing the in-context retrieval capability can narrow the gap.

\vspace{-0.3cm}
\subsection{Validation on Synthetic Task: IsTree}
\vspace{-0.2cm}

\begin{figure}[t] 
\centering
\includegraphics[width=\textwidth]{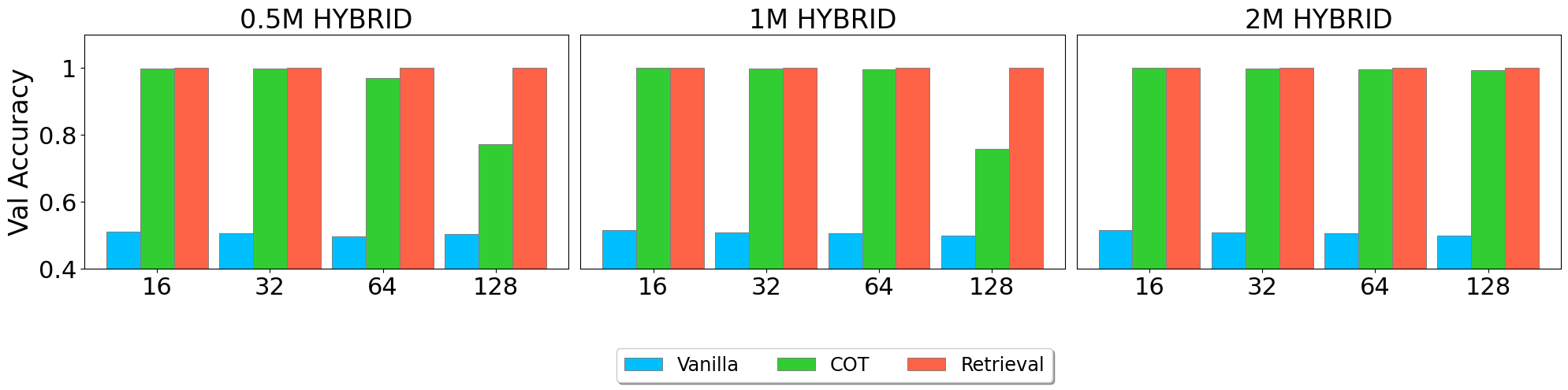} 
\caption{We train Mamba with one additional Transformer layer with a frozen word embedding and decoding head of three different model sizes (0.5M, 1M, 2M) on IsTree with three different sizes of graph (16, 32, 64) under three different setups. \textbf{{\textcolor[RGB]{0, 191, 255}{Vanilla}}} means the model directly predicts the label. \textbf{{\textcolor[RGB]{50, 205, 50}{COT}}} means the model will generate a chain-of-thought process based on DFS (see~\Cref{alg:istree}) before prediction.   {\textbf{\textcolor[RGB]{255, 99, 71}{Retrieval}}} means the model will generate the chain of search queries and reasoning before prediction (see~\Cref{alg:istreeretrieval}).}
\label{fig:performance_hybrid}
\end{figure}

First, we validate our theoretical findings on the IsTree task introduced in~\Cref{sec:istree}.
We generate a synthetic dataset for this task, where each data point is a tokenized graph concatenated with an answer of $\YES$ or $\NO$, potentially with a CoT inserted in between.
We then train RNNs and Transformers as language models on this dataset autoregressively.

\textbf{Experiment Details.} To generate the graph, we follow the procedure described in the proof of~\Cref{thm:rnn_trans_istree} (see~\Cref{fig:istree}). The CoT data is generated using~\Cref{alg:istree} and the retrieval data is generated using~\Cref{alg:istreeretrieval}. For the CoT model, we decode the reasoning path during inference time until we reach the first $\YES$ or $\NO$ up to a max token limit greater than the length of all ground truth CoT. For the data points that the model fails to give a prediction, we assume the model gets it correct with 0.5 probability. For the retrieval task, we omit the explicit format of the regular expression and only ask the model to generate the vertices and special tokens in the regular expression to shorten the length of the input sequence. The reported accuracy is calculated over a validation set of 5000 samples using the last iteration of the model. We defer the experiment details to~\Cref{sec:moreexp}.
The results are shown in~\Cref{fig:performance_comparison,fig:performance_hybrid}.

\textbf{Results with CoT.} 
Without CoT, both the Transformers and RNNs model cannot learn to solve the IsTree problem. CoT improves the performance of both Transformers and RNNs but the RNNs' performance degrades sharply as the graph size increases and the Transformers consistently outperforms RNNs in this case. This is consistent with our theory that CoT can improve the expressiveness of the RNN models but the expressiveness is still not enough to solve IsTree (see~\Cref{thm:rnncot,thm:rnn_trans_istree}).

\textbf{Results with In-Context RAG.} In-Context RAG allows all the models to reach near-perfect accuracy. This is consistent with our theory that retrieval augmentation via regular expression can improve the expressiveness of the RNN models to solve algorithmic tasks (see~\Cref{thm:rnn_turing,thm:rag_istree}).
The hybrid model (a Mamba model with one additional Transformer layer on the top)
shows performance on par with the Transformer, which is consistent with our theory (see~\Cref{thm:hybrid_istree,thm:hybrid_turing}).
\vspace{-0.1in}
\subsection{Validation on Hotpot-QA}
\vspace{-0.1in}
We further conduct experiments on open-source LLMs based on Transformer, Mambda, and a hybrid architecture to show that for tasks that require stronger in-context retrieval capability, RNNs suffer more performance degradation compared to Transformers.

\textbf{Experiment Details.} We use Phi-1.5 1.3B~\citep{li2023textbooksneediiphi15} as our Transformer model.
We further use two Mamba and Transformer-Mamba hybrid models distilled from Phi-1.5  with approximately the same parameters' sizes \citep{bick2024transformersssmsdistillingquadratic} as our RNN and hybrid models.
To test our theory, we use the Hotpot-QA \citep{yang2018hotpotqadatasetdiverseexplainable} dataset. In the validation set of Hotpot-QA, each question is accompanied by a set of different related paragraphs from Wikipedia. $2$ of these paragraphs contain useful information to answer this question. The model needs to retrieve the correct paragraphs and reason based on these paragraphs to get the answer to the question. 

\textbf{HotpotQA with Controlled Retrieval Difficulty.} We consider the following version of Hotpot-QA with enhanced retrieval difficulty: We choose $n$ paragraphs containing the correct paragraphs and order them randomly before the question. The model needs to answer the question based on the given paragraphs after CoT reasoning. This design allows us to use the hyperparameter $n$ to control the difficulty of in-context retrieval in this task, with larger $n$ corresponding to a higher difficulty.

\textbf{Evaluation.} We test our models under a 4-shot setting with Chain-of-Thought. The model we tested has varying performance even if $n = 2$. This is mostly due to their different capabilities to follow instructions and in-context examples. To mitigate this effect and highlight the impact of retrieval, we only test on a subset of $350$ samples of the validation set where all the models can answer correctly given the correct paragraphs. This ensures perfect accuracy when $n = 2$ in both settings.  
We vary $n$ in $\{2,4,6,8\}$ and the result is shown in~\Cref{fig:real}. 

\textbf{RNNs' performance drops sharply with increased retrieval difficulty.} While all the model's performance drops with increased retrieval difficulty (\Cref{fig:real}, left), the RNN model has the largest drops. The hybrid model with only $4$ attention layers performs significantly better than the RNN model. This validates our theory that RNN architectures' limited in-context retrieval capabilities will impact their performance in reasoning and hybrid architecture is a potential solution to this limitation.

\textbf{Deconfouding the effect of context length.} The context length also increases with the number of provided paragraphs and can be a potential confounder. We then experiment with a controlled group: after choosing the same set of $n$ paragraphs, we always order the correct paragraph at the end, which significantly simplifies the in-context retrieval process. In this case, all the models have stable performance when the number of paragraphs increases.

\begin{figure}[t]
\vspace{-0.3in}
    \centering
    \begin{subfigure}
        \centering
        \includegraphics[width=0.45\textwidth]{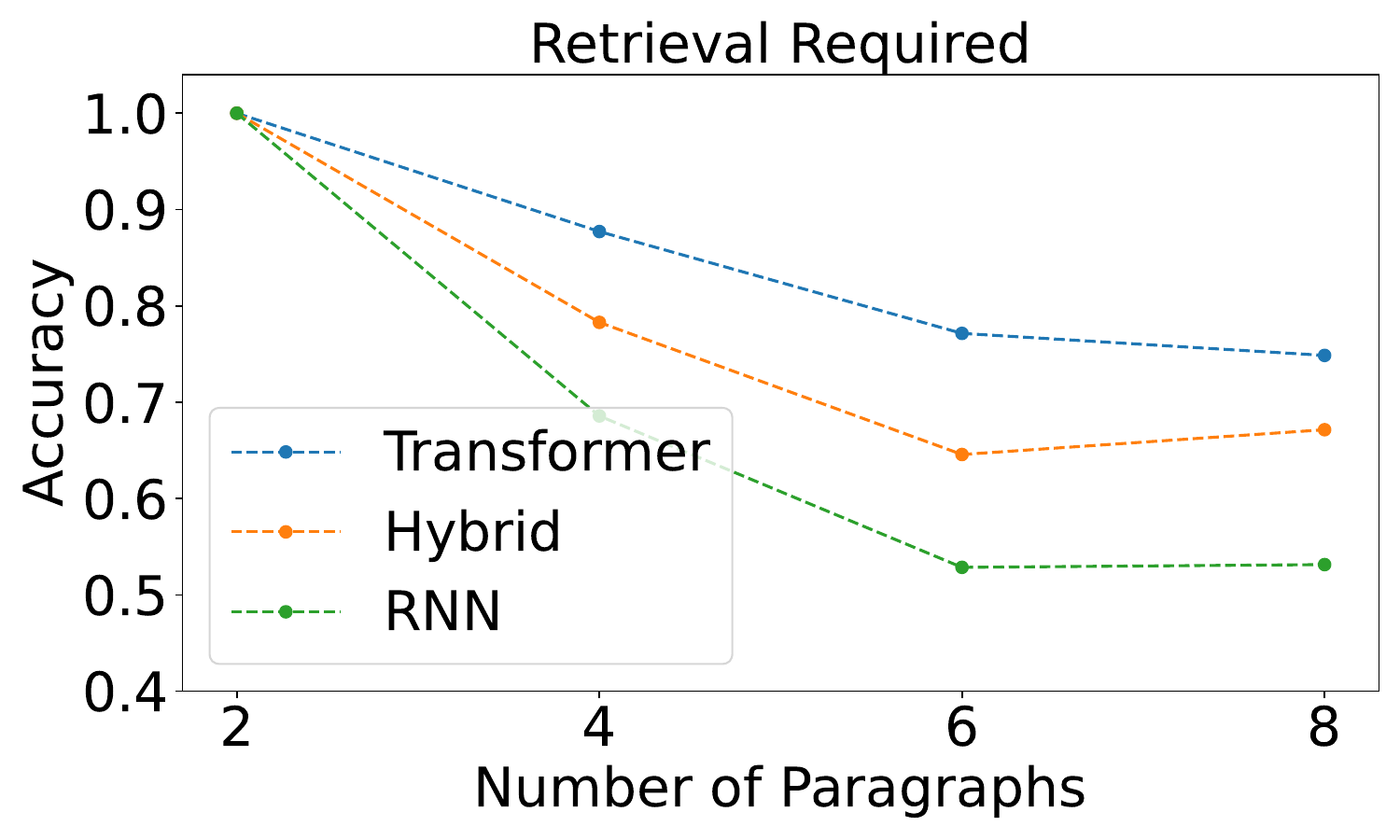}
        \label{fig:reqret}
    \end{subfigure}
    \hfill
    \begin{subfigure}
        \centering
        \includegraphics[width=0.45\textwidth]{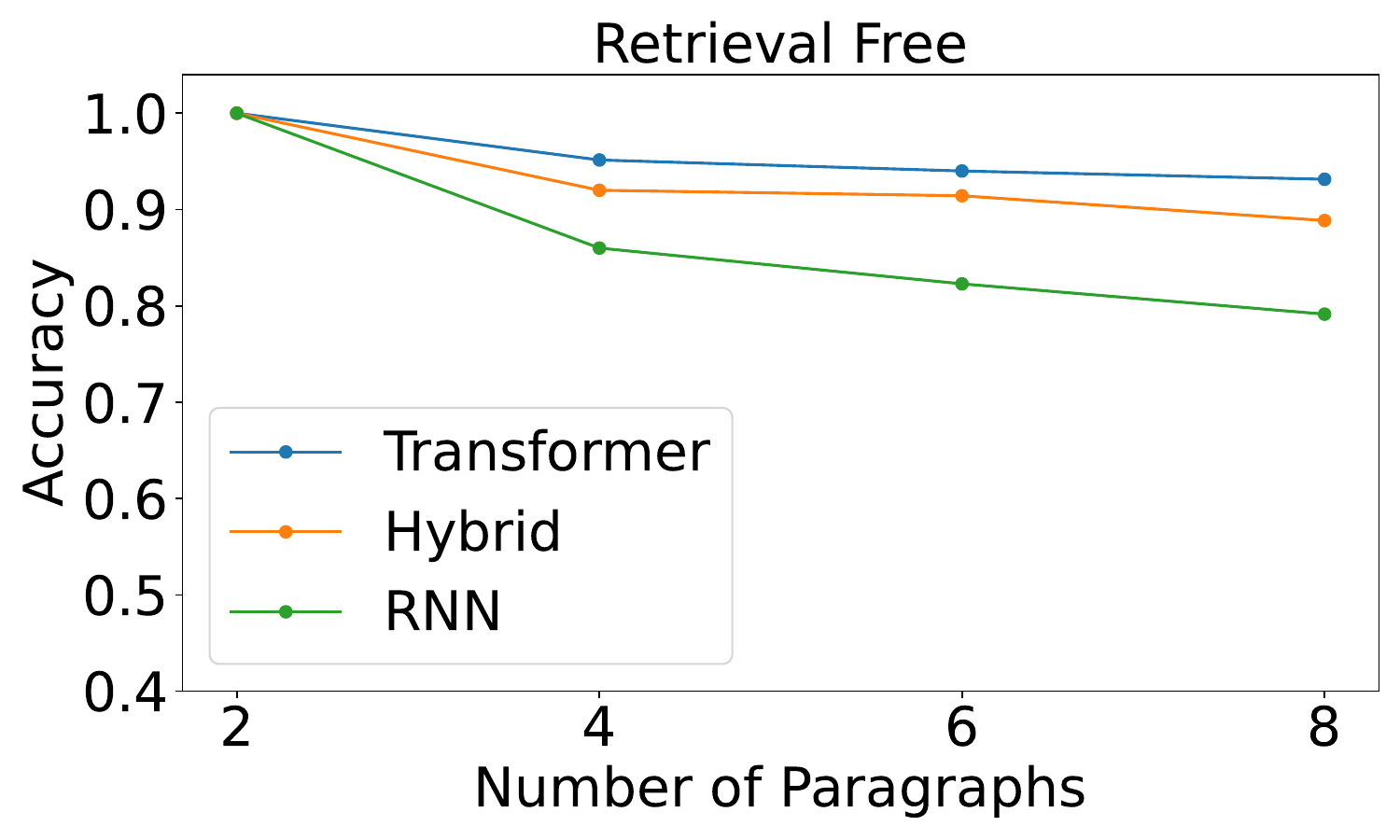}
        \label{fig:retfree}
    \end{subfigure}
    \caption{\textbf{Evaluation of Multi-Document Reasoning on Hotpot-QA.} $n$ paragraphs are given to the model before the question and $2$ among which contains required information to answer the question.
    (Left). When provided paragraphs are given in a random order, the task requires stronger in-context retrieval capabilities when $n$ increases. RNNs' performance drop more sharply, as predicted by theory.
    (Right). When the required paragraphs are always at the end of the provided paragraphs, the retrieval difficulty significantly decreases and all the models have stable performance with respect to $n$.}
    \label{fig:real}
    %\vspace{-0.25in}
\end{figure}

\section{Conclusion and Discussion}

This paper studies the representation gap between RNNs and Transformers on algorithmic problems.
It is proved that CoT can improve RNNs but is insufficient to close the gap with Transformers:
the inability of RNNs to perform in-context retrieval is a key bottleneck.
To address this, we show that adopting In-Context RAG or appending a single Transformer layer to RNNs can enhance their in-context retrieval capabilities and close the representation gap with Transformers.

One limitation of this work is that the solution of In-Context RAG through regular expression is for the purpose of understanding the bottleneck of the representation power of RNNs, and may not be a practical method beyond $\IsTree$ since there is no existing training data for this type of RAG.
Effectively enabling or eliciting LLMs' capability to perform in-context RAG or other types of RAG is an interesting direction for future work.
Second, appending a single Transformer layer to RNNs is a minimal example of making architectural changes to RNNs to improve their representation power while marginally increasing the memory cost. It is left unexplored what other architectural changes can pose a similar effect or enjoy a better trade-off between representation power and memory efficiency.
Finally, we only study the aspect of representation power, and do not analyze the training dynamics and generalization of RNNs and Transformers. We believe this is the most challenging but important direction for future work.

\newpage
\bibliography{all,our}

\begin{thebibliography}{65}
\providecommand{\natexlab}[1]{#1}
\providecommand{\url}[1]{\texttt{#1}}
\expandafter\ifx\csname urlstyle\endcsname\relax
  \providecommand{\doi}[1]{doi: #1}\else
  \providecommand{\doi}{doi: \begingroup \urlstyle{rm}\Url}\fi

\bibitem[Akyürek et~al.(2024)Akyürek, Wang, Kim, and Andreas]{akyürek2024incontext}
Ekin Akyürek, Bailin Wang, Yoon Kim, and Jacob Andreas.
\newblock In-context language learning: Architectures and algorithms, 2024.

\bibitem[Alberti et~al.(2023)Alberti, Dern, Thesing, and Kutyniok]{alberti2023sumformer}
Silas Alberti, Niclas Dern, Laura Thesing, and Gitta Kutyniok.
\newblock Sumformer: Universal approximation for efficient transformers, 2023.

\bibitem[Alon et~al.(1996)Alon, Matias, and Szegedy]{alon1996space}
Noga Alon, Yossi Matias, and Mario Szegedy.
\newblock The space complexity of approximating the frequency moments.
\newblock In \emph{Proceedings of the twenty-eighth annual ACM symposium on Theory of computing}, pp.\  20--29, 1996.

\bibitem[Arora et~al.(2023)Arora, Eyuboglu, Timalsina, Johnson, Poli, Zou, Rudra, and R{\'e}]{arora2023zoology}
Simran Arora, Sabri Eyuboglu, Aman Timalsina, Isys Johnson, Michael Poli, James Zou, Atri Rudra, and Christopher R{\'e}.
\newblock Zoology: Measuring and improving recall in efficient language models.
\newblock \emph{arXiv preprint arXiv:2312.04927}, 2023.

\bibitem[Ba et~al.(2016)Ba, Hinton, Mnih, Leibo, and Ionescu]{ba2016using}
Jimmy Ba, Geoffrey~E Hinton, Volodymyr Mnih, Joel~Z Leibo, and Catalin Ionescu.
\newblock Using fast weights to attend to the recent past.
\newblock \emph{Advances in neural information processing systems}, 29, 2016.

\bibitem[Bahdanau et~al.(2016)Bahdanau, Cho, and Bengio]{bahdanau2016neural}
Dzmitry Bahdanau, Kyunghyun Cho, and Yoshua Bengio.
\newblock Neural machine translation by jointly learning to align and translate, 2016.

\bibitem[Beck et~al.(2024)Beck, Pöppel, Spanring, Auer, Prudnikova, Kopp, Klambauer, Brandstetter, and Hochreiter]{beck2024xlstm}
Maximilian Beck, Korbinian Pöppel, Markus Spanring, Andreas Auer, Oleksandra Prudnikova, Michael Kopp, Günter Klambauer, Johannes Brandstetter, and Sepp Hochreiter.
\newblock xlstm: Extended long short-term memory, 2024.

\bibitem[Bhattamishra et~al.(2020)Bhattamishra, Ahuja, and Goyal]{bhattamishra2020ability}
Satwik Bhattamishra, Kabir Ahuja, and Navin Goyal.
\newblock On the ability and limitations of transformers to recognize formal languages, 2020.

\bibitem[Bick et~al.(2024)Bick, Li, Xing, Kolter, and Gu]{bick2024transformersssmsdistillingquadratic}
Aviv Bick, Kevin~Y. Li, Eric~P. Xing, J.~Zico Kolter, and Albert Gu.
\newblock Transformers to ssms: Distilling quadratic knowledge to subquadratic models, 2024.
\newblock URL \url{https://arxiv.org/abs/2408.10189}.

\bibitem[Borgeaud et~al.(2022)Borgeaud, Mensch, Hoffmann, Cai, Rutherford, Millican, van~den Driessche, Lespiau, Damoc, Clark, de~Las~Casas, Guy, Menick, Ring, Hennigan, Huang, Maggiore, Jones, Cassirer, Brock, Paganini, Irving, Vinyals, Osindero, Simonyan, Rae, Elsen, and Sifre]{borgeaud2022improving}
Sebastian Borgeaud, Arthur Mensch, Jordan Hoffmann, Trevor Cai, Eliza Rutherford, Katie Millican, George van~den Driessche, Jean-Baptiste Lespiau, Bogdan Damoc, Aidan Clark, Diego de~Las~Casas, Aurelia Guy, Jacob Menick, Roman Ring, Tom Hennigan, Saffron Huang, Loren Maggiore, Chris Jones, Albin Cassirer, Andy Brock, Michela Paganini, Geoffrey Irving, Oriol Vinyals, Simon Osindero, Karen Simonyan, Jack~W. Rae, Erich Elsen, and Laurent Sifre.
\newblock Improving language models by retrieving from trillions of tokens, 2022.

\bibitem[Bulatov et~al.(2022)Bulatov, Kuratov, and Burtsev]{bulatov2022recurrent}
Aydar Bulatov, Yuri Kuratov, and Mikhail~S. Burtsev.
\newblock Recurrent memory transformer, 2022.

\bibitem[Elhage et~al.(2021)Elhage, Nanda, Olsson, Henighan, Joseph, Mann, Askell, Bai, Chen, Conerly, DasSarma, Drain, Ganguli, Hatfield-Dodds, Hernandez, Jones, Kernion, Lovitt, Ndousse, Amodei, Brown, Clark, Kaplan, McCandlish, and Olah]{elhage2021mathematical}
Nelson Elhage, Neel Nanda, Catherine Olsson, Tom Henighan, Nicholas Joseph, Ben Mann, Amanda Askell, Yuntao Bai, Anna Chen, Tom Conerly, Nova DasSarma, Dawn Drain, Deep Ganguli, Zac Hatfield-Dodds, Danny Hernandez, Andy Jones, Jackson Kernion, Liane Lovitt, Kamal Ndousse, Dario Amodei, Tom Brown, Jack Clark, Jared Kaplan, Sam McCandlish, and Chris Olah.
\newblock A mathematical framework for transformer circuits.
\newblock \emph{Transformer Circuits Thread}, 2021.
\newblock https://transformer-circuits.pub/2021/framework/index.html.

\bibitem[Feng et~al.(2023)Feng, Zhang, Gu, Ye, He, and Wang]{feng2023towards}
Guhao Feng, Bohang Zhang, Yuntian Gu, Haotian Ye, Di~He, and Liwei Wang.
\newblock Towards revealing the mystery behind chain of thought: A theoretical perspective.
\newblock In \emph{Thirty-seventh Conference on Neural Information Processing Systems}, 2023.
\newblock URL \url{https://openreview.net/forum?id=qHrADgAdYu}.

\bibitem[Fu et~al.(2023)Fu, Dao, Saab, Thomas, Rudra, and Ré]{fu2023hungry}
Daniel~Y. Fu, Tri Dao, Khaled~K. Saab, Armin~W. Thomas, Atri Rudra, and Christopher Ré.
\newblock Hungry hungry hippos: Towards language modeling with state space models, 2023.

\bibitem[Graves et~al.(2014)Graves, Wayne, and Danihelka]{graves2014neural}
Alex Graves, Greg Wayne, and Ivo Danihelka.
\newblock Neural turing machines.
\newblock \emph{arXiv preprint arXiv:1410.5401}, 2014.

\bibitem[Gu \& Dao(2023)Gu and Dao]{gu2023mamba}
Albert Gu and Tri Dao.
\newblock Mamba: Linear-time sequence modeling with selective state spaces, 2023.

\bibitem[Gu et~al.(2022)Gu, Goel, and Re]{gu2022efficiently}
Albert Gu, Karan Goel, and Christopher Re.
\newblock Efficiently modeling long sequences with structured state spaces.
\newblock In \emph{International Conference on Learning Representations}, 2022.
\newblock URL \url{https://openreview.net/forum?id=uYLFoz1vlAC}.

\bibitem[Guu et~al.(2020)Guu, Lee, Tung, Pasupat, and Chang]{guu2020realm}
Kelvin Guu, Kenton Lee, Zora Tung, Panupong Pasupat, and Ming-Wei Chang.
\newblock Realm: Retrieval-augmented language model pre-training, 2020.

\bibitem[Hahn(2020)]{Hahn_2020}
Michael Hahn.
\newblock Theoretical limitations of self-attention in neural sequence models.
\newblock \emph{Transactions of the Association for Computational Linguistics}, 8:\penalty0 156–171, December 2020.
\newblock ISSN 2307-387X.
\newblock \doi{10.1162/tacl_a_00306}.
\newblock URL \url{http://dx.doi.org/10.1162/tacl_a_00306}.

\bibitem[Hao et~al.(2019)Hao, Wang, Yang, Wang, Zhang, and Tu]{hao-etal-2019-modeling}
Jie Hao, Xing Wang, Baosong Yang, Longyue Wang, Jinfeng Zhang, and Zhaopeng Tu.
\newblock Modeling recurrence for transformer.
\newblock In Jill Burstein, Christy Doran, and Thamar Solorio (eds.), \emph{Proceedings of the 2019 Conference of the North {A}merican Chapter of the Association for Computational Linguistics: Human Language Technologies, Volume 1 (Long and Short Papers)}, pp.\  1198--1207, Minneapolis, Minnesota, June 2019. Association for Computational Linguistics.
\newblock \doi{10.18653/v1/N19-1122}.
\newblock URL \url{https://aclanthology.org/N19-1122}.

\bibitem[Hao et~al.(2022)Hao, Angluin, and Frank]{hao2022formal}
Yiding Hao, Dana Angluin, and Robert Frank.
\newblock Formal language recognition by hard attention transformers: Perspectives from circuit complexity, 2022.

\bibitem[Henzinger et~al.(1998)Henzinger, Raghavan, and Rajagopalan]{henzinger1998computing}
Monika~Rauch Henzinger, Prabhakar Raghavan, and Sridhar Rajagopalan.
\newblock Computing on data streams.
\newblock \emph{External memory algorithms}, 50:\penalty0 107--118, 1998.

\bibitem[Hinton \& Anderson(2014)Hinton and Anderson]{hinton2014parallel}
Geoffrey~E Hinton and James~A Anderson.
\newblock \emph{Parallel models of associative memory: updated edition}.
\newblock Psychology press, 2014.

\bibitem[Hopfield(1982)]{hopfield1982neural}
John~J Hopfield.
\newblock Neural networks and physical systems with emergent collective computational abilities.
\newblock \emph{Proceedings of the national academy of sciences}, 79\penalty0 (8):\penalty0 2554--2558, 1982.

\bibitem[Jelassi et~al.(2024)Jelassi, Brandfonbrener, Kakade, and Malach]{jelassi2024repeat}
Samy Jelassi, David Brandfonbrener, Sham~M Kakade, and Eran Malach.
\newblock Repeat after me: Transformers are better than state space models at copying.
\newblock \emph{arXiv preprint arXiv:2402.01032}, 2024.

\bibitem[Jiang et~al.(2022)Jiang, Gao, Wang, Araki, Ding, Callan, and Neubig]{jiang2022retrieval}
Zhengbao Jiang, Luyu Gao, Zhiruo Wang, Jun Araki, Haibo Ding, Jamie Callan, and Graham Neubig.
\newblock Retrieval as attention: End-to-end learning of retrieval and reading within a single transformer.
\newblock In Yoav Goldberg, Zornitsa Kozareva, and Yue Zhang (eds.), \emph{Proceedings of the 2022 Conference on Empirical Methods in Natural Language Processing}, pp.\  2336--2349, Abu Dhabi, United Arab Emirates, December 2022. Association for Computational Linguistics.
\newblock \doi{10.18653/v1/2022.emnlp-main.149}.
\newblock URL \url{https://aclanthology.org/2022.emnlp-main.149}.

\bibitem[Katharopoulos et~al.(2020)Katharopoulos, Vyas, Pappas, and Fleuret]{katharopoulos2020transformers}
Angelos Katharopoulos, Apoorv Vyas, Nikolaos Pappas, and Fran{\c{c}}ois Fleuret.
\newblock Transformers are {RNN}s: Fast autoregressive transformers with linear attention.
\newblock In Hal~Daumé III and Aarti Singh (eds.), \emph{Proceedings of the 37th International Conference on Machine Learning}, volume 119 of \emph{Proceedings of Machine Learning Research}, pp.\  5156--5165. PMLR, 13--18 Jul 2020.
\newblock URL \url{https://proceedings.mlr.press/v119/katharopoulos20a.html}.

\bibitem[Kojima et~al.(2023)Kojima, Gu, Reid, Matsuo, and Iwasawa]{kojima2023large}
Takeshi Kojima, Shixiang~Shane Gu, Machel Reid, Yutaka Matsuo, and Yusuke Iwasawa.
\newblock Large language models are zero-shot reasoners, 2023.

\bibitem[Kuratov et~al.(2024)Kuratov, Bulatov, Anokhin, Sorokin, Sorokin, and Burtsev]{kuratov2024search}
Yuri Kuratov, Aydar Bulatov, Petr Anokhin, Dmitry Sorokin, Artyom Sorokin, and Mikhail Burtsev.
\newblock In search of needles in a 11m haystack: Recurrent memory finds what llms miss, 2024.

\bibitem[Li et~al.(2023{\natexlab{a}})Li, Bubeck, Eldan, Giorno, Gunasekar, and Lee]{li2023textbooksneediiphi15}
Yuanzhi Li, Sébastien Bubeck, Ronen Eldan, Allie~Del Giorno, Suriya Gunasekar, and Yin~Tat Lee.
\newblock Textbooks are all you need ii: phi-1.5 technical report, 2023{\natexlab{a}}.
\newblock URL \url{https://arxiv.org/abs/2309.05463}.

\bibitem[Li et~al.(2023{\natexlab{b}})Li, Li, and Risteski]{li2023transformers}
Yuchen Li, Yuanzhi Li, and Andrej Risteski.
\newblock How do transformers learn topic structure: Towards a mechanistic understanding, 2023{\natexlab{b}}.

\bibitem[Li et~al.(2024)Li, Liu, Zhou, and Ma]{li2024chain}
Zhiyuan Li, Hong Liu, Denny Zhou, and Tengyu Ma.
\newblock Chain of thought empowers transformers to solve inherently serial problems, 2024.

\bibitem[Li et~al.(2021)Li, Han, E, and Li]{li2021on}
Zhong Li, Jiequn Han, Weinan E, and Qianxiao Li.
\newblock On the curse of memory in recurrent neural networks: Approximation and optimization analysis.
\newblock In \emph{International Conference on Learning Representations}, 2021.
\newblock URL \url{https://openreview.net/forum?id=8Sqhl-nF50}.

\bibitem[Li et~al.(2022)Li, Jiang, and Li]{li2022on}
Zhong Li, Haotian Jiang, and Qianxiao Li.
\newblock On the approximation properties of recurrent encoder-decoder architectures.
\newblock In \emph{International Conference on Learning Representations}, 2022.
\newblock URL \url{https://openreview.net/forum?id=xDIvIqQ3DXD}.

\bibitem[Lieber et~al.(2024)Lieber, Lenz, Bata, Cohen, Osin, Dalmedigos, Safahi, Meirom, Belinkov, Shalev-Shwartz, Abend, Alon, Asida, Bergman, Glozman, Gokhman, Manevich, Ratner, Rozen, Shwartz, Zusman, and Shoham]{lieber2024jambahybridtransformermambalanguage}
Opher Lieber, Barak Lenz, Hofit Bata, Gal Cohen, Jhonathan Osin, Itay Dalmedigos, Erez Safahi, Shaked Meirom, Yonatan Belinkov, Shai Shalev-Shwartz, Omri Abend, Raz Alon, Tomer Asida, Amir Bergman, Roman Glozman, Michael Gokhman, Avashalom Manevich, Nir Ratner, Noam Rozen, Erez Shwartz, Mor Zusman, and Yoav Shoham.
\newblock Jamba: A hybrid transformer-mamba language model, 2024.
\newblock URL \url{https://arxiv.org/abs/2403.19887}.

\bibitem[Liu et~al.(2023)Liu, Ash, Goel, Krishnamurthy, and Zhang]{liu2023transformers}
Bingbin Liu, Jordan~T. Ash, Surbhi Goel, Akshay Krishnamurthy, and Cyril Zhang.
\newblock Transformers learn shortcuts to automata.
\newblock In \emph{The Eleventh International Conference on Learning Representations}, 2023.
\newblock URL \url{https://openreview.net/forum?id=De4FYqjFueZ}.

\bibitem[Luo et~al.(2021)Luo, Li, Cai, He, Peng, Zheng, Ke, Wang, and Liu]{luo2021stable}
Shengjie Luo, Shanda Li, Tianle Cai, Di~He, Dinglan Peng, Shuxin Zheng, Guolin Ke, Liwei Wang, and Tie-Yan Liu.
\newblock Stable, fast and accurate: Kernelized attention with relative positional encoding, 2021.

\bibitem[Lutati et~al.(2023)Lutati, Zimerman, and Wolf]{lutati2023focus}
Shahar Lutati, Itamar Zimerman, and Lior Wolf.
\newblock Focus your attention (with adaptive iir filters), 2023.

\bibitem[Merrill \& Sabharwal(2023)Merrill and Sabharwal]{merrill2023parallelism}
William Merrill and Ashish Sabharwal.
\newblock {The Parallelism Tradeoff: Limitations of Log-Precision Transformers}.
\newblock \emph{Transactions of the Association for Computational Linguistics}, 11:\penalty0 531--545, 06 2023.
\newblock ISSN 2307-387X.
\newblock \doi{10.1162/tacl_a_00562}.
\newblock URL \url{https://doi.org/10.1162/tacl\_a\_00562}.

\bibitem[Merrill et~al.(2022)Merrill, Sabharwal, and Smith]{merrill2022saturated}
William Merrill, Ashish Sabharwal, and Noah~A. Smith.
\newblock Saturated transformers are constant-depth threshold circuits, 2022.

\bibitem[Munro \& Paterson(1980)Munro and Paterson]{munro1980selection}
J~Ian Munro and Mike~S Paterson.
\newblock Selection and sorting with limited storage.
\newblock \emph{Theoretical computer science}, 12\penalty0 (3):\penalty0 315--323, 1980.

\bibitem[Nye et~al.(2021)Nye, Andreassen, Gur-Ari, Michalewski, Austin, Bieber, Dohan, Lewkowycz, Bosma, Luan, Sutton, and Odena]{nye2021work}
Maxwell Nye, Anders~Johan Andreassen, Guy Gur-Ari, Henryk Michalewski, Jacob Austin, David Bieber, David Dohan, Aitor Lewkowycz, Maarten Bosma, David Luan, Charles Sutton, and Augustus Odena.
\newblock Show your work: Scratchpads for intermediate computation with language models, 2021.

\bibitem[Oren et~al.(2024)Oren, Hassid, Adi, and Schwartz]{oren2024transformers}
Matanel Oren, Michael Hassid, Yossi Adi, and Roy Schwartz.
\newblock Transformers are multi-state rnns, 2024.

\bibitem[Park et~al.(2024)Park, Park, Xiong, Lee, Cho, Oymak, Lee, and Papailiopoulos]{park2024mamba}
Jongho Park, Jaeseung Park, Zheyang Xiong, Nayoung Lee, Jaewoong Cho, Samet Oymak, Kangwook Lee, and Dimitris Papailiopoulos.
\newblock Can mamba learn how to learn? a comparative study on in-context learning tasks, 2024.

\bibitem[Peng et~al.(2023)Peng, Alcaide, Anthony, Albalak, Arcadinho, Biderman, Cao, Cheng, Chung, Grella, GV, He, Hou, Lin, Kazienko, Kocon, Kong, Koptyra, Lau, Mantri, Mom, Saito, Song, Tang, Wang, Wind, Wozniak, Zhang, Zhang, Zhao, Zhou, Zhou, Zhu, and Zhu]{peng2023rwkv}
Bo~Peng, Eric Alcaide, Quentin Anthony, Alon Albalak, Samuel Arcadinho, Stella Biderman, Huanqi Cao, Xin Cheng, Michael Chung, Matteo Grella, Kranthi~Kiran GV, Xuzheng He, Haowen Hou, Jiaju Lin, Przemyslaw Kazienko, Jan Kocon, Jiaming Kong, Bartlomiej Koptyra, Hayden Lau, Krishna Sri~Ipsit Mantri, Ferdinand Mom, Atsushi Saito, Guangyu Song, Xiangru Tang, Bolun Wang, Johan~S. Wind, Stanislaw Wozniak, Ruichong Zhang, Zhenyuan Zhang, Qihang Zhao, Peng Zhou, Qinghua Zhou, Jian Zhu, and Rui-Jie Zhu.
\newblock Rwkv: Reinventing rnns for the transformer era, 2023.

\bibitem[Peng et~al.(2021)Peng, Pappas, Yogatama, Schwartz, Smith, and Kong]{peng2021random}
Hao Peng, Nikolaos Pappas, Dani Yogatama, Roy Schwartz, Noah~A. Smith, and Lingpeng Kong.
\newblock Random feature attention, 2021.

\bibitem[Poli et~al.(2023)Poli, Massaroli, Nguyen, Fu, Dao, Baccus, Bengio, Ermon, and Ré]{poli2023hyena}
Michael Poli, Stefano Massaroli, Eric Nguyen, Daniel~Y. Fu, Tri Dao, Stephen Baccus, Yoshua Bengio, Stefano Ermon, and Christopher Ré.
\newblock Hyena hierarchy: Towards larger convolutional language models, 2023.

\bibitem[Ren et~al.(2024)Ren, Liu, Lu, Shen, Liang, and Chen]{ren2024sambasimplehybridstate}
Liliang Ren, Yang Liu, Yadong Lu, Yelong Shen, Chen Liang, and Weizhu Chen.
\newblock Samba: Simple hybrid state space models for efficient unlimited context language modeling, 2024.
\newblock URL \url{https://arxiv.org/abs/2406.07522}.

\bibitem[Rubin \& Berant(2023)Rubin and Berant]{rubin2023longrange}
Ohad Rubin and Jonathan Berant.
\newblock Long-range language modeling with self-retrieval, 2023.

\bibitem[Sanford et~al.(2023)Sanford, Hsu, and Telgarsky]{sanford2023representational}
Clayton Sanford, Daniel Hsu, and Matus Telgarsky.
\newblock Representational strengths and limitations of transformers, 2023.

\bibitem[Sanford et~al.(2024)Sanford, Hsu, and Telgarsky]{sanford2024transformers}
Clayton Sanford, Daniel Hsu, and Matus Telgarsky.
\newblock Transformers, parallel computation, and logarithmic depth, 2024.

\bibitem[Su et~al.(2024)Su, Ahmed, Lu, Pan, Bo, and Liu]{su2024roformer}
Jianlin Su, Murtadha Ahmed, Yu~Lu, Shengfeng Pan, Wen Bo, and Yunfeng Liu.
\newblock Roformer: Enhanced transformer with rotary position embedding.
\newblock \emph{Neurocomputing}, 568:\penalty0 127063, 2024.

\bibitem[Sun et~al.(2023)Sun, Dong, Huang, Ma, Xia, Xue, Wang, and Wei]{sun2023retentive}
Yutao Sun, Li~Dong, Shaohan Huang, Shuming Ma, Yuqing Xia, Jilong Xue, Jianyong Wang, and Furu Wei.
\newblock Retentive network: A successor to transformer for large language models, 2023.

\bibitem[Tian et~al.(2023)Tian, Wang, Chen, and Du]{tian2023scan}
Yuandong Tian, Yiping Wang, Beidi Chen, and Simon Du.
\newblock Scan and snap: Understanding training dynamics and token composition in 1-layer transformer, 2023.

\bibitem[Touvron et~al.(2023)Touvron, Martin, Stone, Albert, Almahairi, Babaei, Bashlykov, Batra, Bhargava, Bhosale, Bikel, Blecher, Ferrer, Chen, Cucurull, Esiobu, Fernandes, Fu, Fu, Fuller, Gao, Goswami, Goyal, Hartshorn, Hosseini, Hou, Inan, Kardas, Kerkez, Khabsa, Kloumann, Korenev, Koura, Lachaux, Lavril, Lee, Liskovich, Lu, Mao, Martinet, Mihaylov, Mishra, Molybog, Nie, Poulton, Reizenstein, Rungta, Saladi, Schelten, Silva, Smith, Subramanian, Tan, Tang, Taylor, Williams, Kuan, Xu, Yan, Zarov, Zhang, Fan, Kambadur, Narang, Rodriguez, Stojnic, Edunov, and Scialom]{touvron2023llama}
Hugo Touvron, Louis Martin, Kevin Stone, Peter Albert, Amjad Almahairi, Yasmine Babaei, Nikolay Bashlykov, Soumya Batra, Prajjwal Bhargava, Shruti Bhosale, Dan Bikel, Lukas Blecher, Cristian~Canton Ferrer, Moya Chen, Guillem Cucurull, David Esiobu, Jude Fernandes, Jeremy Fu, Wenyin Fu, Brian Fuller, Cynthia Gao, Vedanuj Goswami, Naman Goyal, Anthony Hartshorn, Saghar Hosseini, Rui Hou, Hakan Inan, Marcin Kardas, Viktor Kerkez, Madian Khabsa, Isabel Kloumann, Artem Korenev, Punit~Singh Koura, Marie-Anne Lachaux, Thibaut Lavril, Jenya Lee, Diana Liskovich, Yinghai Lu, Yuning Mao, Xavier Martinet, Todor Mihaylov, Pushkar Mishra, Igor Molybog, Yixin Nie, Andrew Poulton, Jeremy Reizenstein, Rashi Rungta, Kalyan Saladi, Alan Schelten, Ruan Silva, Eric~Michael Smith, Ranjan Subramanian, Xiaoqing~Ellen Tan, Binh Tang, Ross Taylor, Adina Williams, Jian~Xiang Kuan, Puxin Xu, Zheng Yan, Iliyan Zarov, Yuchen Zhang, Angela Fan, Melanie Kambadur, Sharan Narang, Aurelien Rodriguez, Robert Stojnic, Sergey Edunov, and Thomas
  Scialom.
\newblock Llama 2: Open foundation and fine-tuned chat models, 2023.

\bibitem[Vaswani et~al.(2017)Vaswani, Shazeer, Parmar, Uszkoreit, Jones, Gomez, Kaiser, and Polosukhin]{vaswani2017attention}
Ashish Vaswani, Noam Shazeer, Niki Parmar, Jakob Uszkoreit, Llion Jones, Aidan~N Gomez, \L~ukasz Kaiser, and Illia Polosukhin.
\newblock Attention is all you need.
\newblock In I.~Guyon, U.~Von Luxburg, S.~Bengio, H.~Wallach, R.~Fergus, S.~Vishwanathan, and R.~Garnett (eds.), \emph{Advances in Neural Information Processing Systems}, volume~30. Curran Associates, Inc., 2017.
\newblock URL \url{https://proceedings.neurips.cc/paper_files/paper/2017/file/3f5ee243547dee91fbd053c1c4a845aa-Paper.pdf}.

\bibitem[Waleffe et~al.(2024)Waleffe, Byeon, Riach, Norick, Korthikanti, Dao, Gu, Hatamizadeh, Singh, Narayanan, Kulshreshtha, Singh, Casper, Kautz, Shoeybi, and Catanzaro]{waleffe2024empiricalstudymambabasedlanguage}
Roger Waleffe, Wonmin Byeon, Duncan Riach, Brandon Norick, Vijay Korthikanti, Tri Dao, Albert Gu, Ali Hatamizadeh, Sudhakar Singh, Deepak Narayanan, Garvit Kulshreshtha, Vartika Singh, Jared Casper, Jan Kautz, Mohammad Shoeybi, and Bryan Catanzaro.
\newblock An empirical study of mamba-based language models, 2024.
\newblock URL \url{https://arxiv.org/abs/2406.07887}.

\bibitem[Wang et~al.(2020)Wang, Li, Khabsa, Fang, and Ma]{wang2020linformer}
Sinong Wang, Belinda~Z. Li, Madian Khabsa, Han Fang, and Hao Ma.
\newblock Linformer: Self-attention with linear complexity, 2020.

\bibitem[Wang \& Zhou(2024)Wang and Zhou]{wang2024chainofthought}
Xuezhi Wang and Denny Zhou.
\newblock Chain-of-thought reasoning without prompting, 2024.

\bibitem[Wei et~al.(2023)Wei, Wang, Schuurmans, Bosma, Ichter, Xia, Chi, Le, and Zhou]{wei2023chainofthought}
Jason Wei, Xuezhi Wang, Dale Schuurmans, Maarten Bosma, Brian Ichter, Fei Xia, Ed~Chi, Quoc Le, and Denny Zhou.
\newblock Chain-of-thought prompting elicits reasoning in large language models, 2023.

\bibitem[Willshaw et~al.(1969)Willshaw, Buneman, and Longuet-Higgins]{willshaw1969non}
David~J Willshaw, O~Peter Buneman, and Hugh~Christopher Longuet-Higgins.
\newblock Non-holographic associative memory.
\newblock \emph{Nature}, 222\penalty0 (5197):\penalty0 960--962, 1969.

\bibitem[Xiao et~al.(2023)Xiao, Tian, Chen, Han, and Lewis]{xiao2023efficient}
Guangxuan Xiao, Yuandong Tian, Beidi Chen, Song Han, and Mike Lewis.
\newblock Efficient streaming language models with attention sinks, 2023.

\bibitem[Yang et~al.(2024)Yang, Ackermann, He, Feng, Zhang, Feng, Ye, He, and Wang]{yang2024efficient}
Kai Yang, Jan Ackermann, Zhenyu He, Guhao Feng, Bohang Zhang, Yunzhen Feng, Qiwei Ye, Di~He, and Liwei Wang.
\newblock Do efficient transformers really save computation?, 2024.

\bibitem[Yang et~al.(2018)Yang, Qi, Zhang, Bengio, Cohen, Salakhutdinov, and Manning]{yang2018hotpotqadatasetdiverseexplainable}
Zhilin Yang, Peng Qi, Saizheng Zhang, Yoshua Bengio, William~W. Cohen, Ruslan Salakhutdinov, and Christopher~D. Manning.
\newblock Hotpotqa: A dataset for diverse, explainable multi-hop question answering, 2018.
\newblock URL \url{https://arxiv.org/abs/1809.09600}.

\bibitem[Yao et~al.(2023)Yao, Peng, Papadimitriou, and Narasimhan]{yao2023selfattention}
Shunyu Yao, Binghui Peng, Christos Papadimitriou, and Karthik Narasimhan.
\newblock Self-attention networks can process bounded hierarchical languages, 2023.

\end{thebibliography}
\bibliographystyle{iclr2024_conference}
\newpage
\appendix
\tableofcontents
\newpage
\section{Additional Definitions} \label{sec:def}
We will now define some definitions used in the proofs.

\subsection{Reasoning Tasks on Graphs.}

When resaoning on graphs, without otherwise specified, we will use $\verticenumber$ as the number of vertices and $\edgenumber$ as the number of edges. Without loss of generality, we will assume the vertices are labeled by $[n]$. 

We will focus on decision problems on graphs, which are defined as follows:
\begin{definition}[Decision Problem on Graphs]
    A decision problem on graphs is a function $f: \graphset \to \{\YES, \NO\}$, where $\graphset$ is the set of all possible graphs.
\end{definition}

We will use the following decision problem as our main example:
\begin{definition}[IsTree]
    $\IsTree(\graph) = \YES$ if $\graph$ is a tree, and $\IsTree(\graph) = \NO$ otherwise.
\end{definition}

One can view IsTree as a minimal example of reasoning tasks. One of the classical solutions to IsTree is running Depth First Search and this algorithm takes $O(n)$ time.

\subsection{More on Numeric Precisions.}
\label{sec:precision}
We will use $\ROUND(x, \precision)$ to denote the rounding function that rounds $x$ to the nearest number in $\R_\precision$. We will assume $\precision$ is an odd number without loss of generality.
\begin{align}
    \R_\precision = \left\{(2b_p - 1)\left(\sum_{i = 1}^{p - 1} b_i 2^{(p - 1)/2 - i}\right) : \forall i \in [p], b_i \in \{0,1\} \right\}.
\end{align}

For calculation over $\R_\precision$, we will assume the calculation is exact and the result is rounded to $\R_\precision$ at the end, that is, for operator $\oplus$, we will have 
\begin{align*}
    &\ROUND(x, \precision) \oplus_{\precision} \ROUND(y, \precision) \\=& \ROUND\left(\ROUND(x, \precision) \oplus \ROUND(y, \precision), \precision) \right).
\end{align*}

We will additionally define $\Integer$ as the set of all integers that can be represented by $\precision$-bit floating point numbers. We will define $1/[m]$ as the set of unit fractional $\{\frac{1}{i}\}_{i \in [m]}$. Further, we will define $\ROUND(1/[m], \precision)$ as the rounding of $1/[m]$ to $\R_{\precision}$. We will additionally define for any real number $x \in \{0,1\}$, $\nextnum (x) = \frac{1}{m + 1}$ where $m = \argmin_{k \in \Z} |x - \frac{1}{k}|$.

\subsection{Models} \label{sec:def-models}

\paragraph{Tokenization.}
To tokenize a string, we will tokenize all the words separated by the space character into a sequence of tokens. 
To tokenize a graph $\graph$, we will order its edges $\Edge = \{(u_i, v_i) \mid u_i < v_i\}$ randomly and tokenize it into the following string:
\begin{align}
    \Tokenize(\graph) = \{\StartSentence, u_1, \sim, v_1, \ldots, u_\edgenumber, \sim, v_\edgenumber\}.
\end{align}
We hereby assume there are constant more special tokens that are not the same as any number token, which are listed below:
\begin{itemize}
    \item $\StartSentence$: the first special token, indicating the start of a sentence.
    \item $\sim$: the second special token, indicating an edge.
    \item $\YES$: the third special token, indicating the answer is yes.
    \item $\NO$: the fourth special token, indicating the answer is no.
    \item $\StartSearch$: the fifth special token, indicating the start of a search query.
    \item $\EndSearch$: the sixth special token, indicating the end of a search query.
\end{itemize}

We will denote the total number of special tokens as $\numspecial$ and the total vocabulary size as $\VocabSize = n + \numspecial$. We will further define the detokenization function $\Detokenize$, 
\begin{align*}
    \Detokenize(\TokenSequence) = ``\TokenSequence_1\ \TokenSequence_2\ \ldots\ \TokenSequence_{\position}".
\end{align*}
Here each $\TokenSequence_i$ is either a number or a special token, which we will treat as a word.

\paragraph{Embedding Functions.} 

We will use $\dimension$ to denote the dimension of the embedding and $\onehot_i$ to denote the $i$-th coordinate vector in $\R^\dimension$.

We will separate the embedding function into two parts: the word embedding and the position embedding.
For the word embedding, we will use $\VEmbed{i}  \in \R^{\dimension}$ to represent the embedding of the vertice $i$ in the tokenization of the graph. For the $k$-th special token, we will use $\onehot_{2 + k}$ to represent its embedding. For example, the embedding of $\sim$ is $\onehot_{2}$.  We will denote the word embedding matrix as $\WordEmbedding \in \R^{\VocabSize \times \dimension}.$
    
For the position embedding, we will use $l\onehot_{\dimension}$ to represent the position embedding of the $l$-th token in the tokenization of the graph, which is a hyperparameter. The final embedding of any token sequence is the sum of the word embedding and the position embedding. We will use $\Embedding$ to denote the embedding function.

This embedding function will be fixed and shared across all models we consider in this paper and will not be learned during training, hence we will not consider it as part of the model parameters.

\paragraph{Language Modeling.}
In this work, we will consider the difference between Transformer and Recurrent Neural Networks on reasoning tasks, which is a special case of language modeling. We will define language modeling as follows:
\begin{definition}[Language Model]
    A language model is a function $M: \cup_{\position = 1}^{\maxposition} \VocabSize^\position \to \Prb_{\VocabSize}$, where $\VocabSize$ is the vocabulary size, $\position$ is the sequence length, and $\Prb_{\VocabSize}$ is the probability simplex over $\VocabSize$. 
\end{definition}

\subsection{Language Models for Reasoning.} 

\paragraph{Chain of Thought.} We will now define how we use language models to solve reasoning tasks utilizing the following technique called chain of thought.

\begin{definition}[Chain of Thought]
\label{def:cot}
Given a language model, $\model$ with vocabulary $\Vocab$ and  the tokenized input sequence $\InputTokenSequence \in \VocabSize^{\inputlength}$, chain of thought (CoT) generates the following sequence of tokenized sequence:
\begin{align}
    \TokenSequence_0 &= \InputTokenSequence, \\
    \NextToken_i &= \argmax_{j \in \Vocab} \model(\TokenSequence_{i})[j],
    \\
    \TokenSequence_{i + 1} &= \TokenSequence_{i} \oplus \NextToken_i, \forall i \ge 0.
\end{align}
The process terminates at $\TokenSequence_i$ when $\argmax_{j \in \Vocab} \model(\TokenSequence_{i})[j]$ is $\YES$ or $\NO$. The language model can solve the reasoning task within $\outputlength$ steps of CoT if the process terminates at $\TokenSequence_i$ where $i \le \outputlength$ and the final output is correct. We will call the special case where the language model solves the reasoning task within $0$ steps of CoT as solving the reasoning task without CoT.
\end{definition}

\paragraph{Retrieval Augmentation.} We will show in this paper that retrieval augmentation is a necessary technique to solve reasoning tasks for recurrent neural networks. We will define retrieval augmentation as follows:

\begin{definition}[Retrieval Augmented Generation]
\label{def:rag}
Retrieval Augmented Generation means giving the language model the capability to retrieve relevant information to assist generation. We formally described the process here. Given a language model $\model$ with vocabulary $\Vocab$ (containing two additional special tokens called $\StartSearch$ and $\EndSearch$) and the tokenized input sequence $\InputTokenSequence \in \VocabSize^{\inputlength}$, retrieval augmented generation generates following sequence of tokenized sequence:
\begin{align*}
    \TokenSequence_0 &= \InputTokenSequence, \\
    \NextToken_i &= \argmax_{j \in \Vocab} \model(\TokenSequence_{i})[j], \\
    \TokenSequence_{i + 1} &= \begin{cases} &\TokenSequence_{i} \oplus \NextToken_i, \NextToken_i \neq \EndSearch \\
    &\TokenSequence_{i} \oplus \NextToken_i \oplus \RETRIEVE\left(\TokenSequence_i\right) , \text{otherwise.}
    \end{cases}
\end{align*}    

Here $\RETRIEVE$ looks for the last occurrence of $\StartSearch$ at position $\StartPos$ and $\EndSearch$ in $\TokenSequence$ at positon $\EndPos$ and treat $\Detokenize(\TokenSequence_{\StartPos: \EndPos})$ as a regular expression. The algorithm then uses the regular expression on $\Detokenize(\TokenSequence_{1:\StartPos - 1})$. If the regular expression ever matches, the $\RETRIEVE$ will return the match. If the regular expression never matches, $\RETRIEVE$ will return a special token $\Failed$. 

Similar to~\Cref{def:cot}, we can define the notion of solving the reasoning task within $\outputlength$ steps of retrieval augmented generation and solving the reasoning task without retrieval augmented generation.
\end{definition}

We will note that assuming $|V| = O(n)$ and every search query and the result is of length $O(1)$, the regular expression evaluation can typically be evaluated in $O(n)$ time.

\section{Omitted Experiment Details}
\label{sec:moreexp}
We train three different architectures: (1) LLaMA architecture~\citep{touvron2023llama} representing Transformers, (2) Mamba architecture~\citep{gu2023mamba} representing RNNs, and (3) Mamba with one additional layer of LLaMA block representing hybrid architectures. Following our theory, we freeze and weight-tie the prediction head and word embedding in all the models. For ease of training, we use a different embedding function mapping $i-$th token to $[\sin(\frac{i}{10000^{j/d}}), \cos(\frac{i}{10000^{j/d}})]_{j \in [d/2]}$ with $N$ being the number of different tokens and use standard RoPE~\citep{su2024roformer} as position embedding.
We train every model with at least 1M samples to guarantee convergence using Adam with a cosine learning rate. If the model doesn't converge, we retrain using 5M samples. After a grid search over learning rates, we train all the Transformer models with learning rates 1e-3 and the rest of the models with learning rates 3e-4.  We run all the experiments on a server with 8 A100s and the estimated time to reproduce the results is within 2 days.
\section{Omitted Proof}
\subsection{Building Blocks of FFNs Construction}

We will first show some basic operations that multiple layers of feedforward neural networks with ReGLU activation can perform that will be used in the following proofs.

\begin{lemma}[Multiplication]
\label{lem:multiplication}
Given two dimensions $i_1, i_2$, there exists a parameter configuration of a 1-layer feedforward neural network with ReGLU activation that for any input $x \in \R^{\dimension}$ and constant width, computes the product of $x_{i_1}$ and $x_{i_2}$.
\begin{align*}
    \FFN(x) = [x_{i_1} \times x_{i_2}, 0, \ldots, 0]^\top.
\end{align*}
\end{lemma}

\begin{proof}
We can construct the following feedforward neural network with ReGLU activation:
\begin{align*}
    &W_1 x = \begin{bmatrix}
        x_{i_1} & -x_{i_1} & 0 \ldots & 0
    \end{bmatrix}^\top, \\
    &W_2 x = \begin{bmatrix}
        x_{i_2} & x_{i_2} & 0 \ldots & 0
    \end{bmatrix}^\top, \\
    &W_3 h = \begin{bmatrix}
        h_1 + h_2  & 0 & 0 & 0 &\ldots & 0\end{bmatrix}^\top, \\
        &b_1 = b_2 = b_3 = 0.\\
        &\FFN(x) = W_3 ( \ReLU(W_1 x)  \otimes \ReLU(W_2 x))= [x_{i_1} \times x_{i_2}, 0, \ldots, 0]^\top.
\end{align*}
\end{proof}

\begin{lemma}[Linear Operations]
\label{lem:linear_operations}
Given a linear transformation $W \in \R^{\dimension \times \dimension}$, there exists a parameter configuration of a 1-layer feedforward neural network with ReGLU activation and width $\width = \dimension$ that for any input $x \in \R^{\dimension}$, computes $Wx$.
\end{lemma}

\begin{proof}
\begin{align*}
    &b_1 = 1^\width ,b_2 = 0 , b_3 = 0, \\ 
    &W_1 = 0,W_2 = W, W_3 = 
        I_{\dimension \times \dimension} 
\end{align*}
\end{proof}

\begin{lemma}[Indicator]
\label{lem:indicator}
Given a constant integer $B \le \dimension$ and a dimension $i$, there exists a parameter configuration of a 1-layer feedforward neural network with ReGLU activation and width $4$ that for any input $x \in \R^{\dimension}$, computes the indicator function of $x_{i} = B$ when $x_i$ is an integer.
\end{lemma}

\begin{proof}
\begin{align*}
    &b_2 = 1^\width, b_1 = [-B -0.5, -B + 0.5, B - 0.6, B + 0.4]^\top, b_3 = 10, \\
    &W_2 = 0, W_1x = [x_i, x_i, -x_i, -x_i]^\top, \\
    &\ReLU(W_1 x + b_1) = \begin{bmatrix}
        \ReLU(x_i - B - 0.5) & \ReLU(x_i - B + 0.5) & \ReLU(B - x_i - 0.6) & \ReLU(B - x_i + 0.4)
    \end{bmatrix} \\
    &W_3 = 10 \begin{bmatrix}
        1 & -1 & -1 & 1
    \end{bmatrix}^\top.
\end{align*}
Then,
\begin{align*}
    \FFN(x) &= W_3 \ReLU(W_1 x + b_1) + b_3\\
    &= 10\left(\ReLU(x_i - B - 0.5) - \ReLU(x_i - B + 0.5)\right) + 10\left(\ReLU(B - x_i - 0.6) - \ReLU(B - x_i + 0.4)\right) + 10 \\
    &= \begin{cases}
        10 \times 0 + 10 \times -1 + 10 = 0, & \text{if } x_i < B, \\
        10 \times -0.5 + 10 \times -0.4 + 10 = 1, & \text{if } x_i = B, \\
        10 \times -1 + 10 \times 0 + 10= 0, & \text{if } x_i > B.
    \end{cases}
\end{align*}
\end{proof}

\begin{lemma}[Lookup Table]
\label{lem:lookup_table}
For constant $B$ and $k$ such that $kB \le \dimension$, given a lookup table which key is tuple of $k$ integers bounded by $B$, and value is a scalar in $\R_{\precision}$, there exists a parameter configuration of a $1-$ layer feedforward neural network with ReGLU activation with width $O(B^k)$  that for any input $x \in \R^{\dimension}$, computes the value of the lookup table at the key $x_{i_1}, x_{i_2}, \ldots, x_{i_k}$.
\end{lemma}

\begin{proof}
We can calculate $x_{i_1} + B \times x_{i_2} + B^2 \times x_{i_3} + \ldots + B^{k-1} \times x_{i_k}$, and then scale $B^k$ indicator functions to get the value of the lookup table at the key $x_{i_1}, x_{i_2}, \ldots, x_{i_k}$.
\end{proof}

\begin{lemma}[Threshold]
\label{lem:threshold}
Given any threshold $u$ and constant $\epsilon > 0$, there exists a parameter configuration of a 1-layer feedforward neural network with ReGLU activation and width $2$ that for any input $x \in \R^{\dimension}$, computes the indicator function $x_i > u$ on $x_i \in [-\infty, u - \epsilon] \cup [u, \infty]$.
\end{lemma}

\begin{proof}
    \begin{align*}
        &b_2 = 1^\width, b_1 = [- u +  \epsilon, -u + 0.5\epsilon]^\top, b_3 = 0, \\
        &W_2 = 0, W_1x = [x_i, x_i]^\top, \\
        &\ReLU(W_1 x + b_1) = \begin{bmatrix}
            \ReLU(x_i - u + \epsilon) & \ReLU(x_i - u + 0.5\epsilon) 
        \end{bmatrix} \\
        &W_3 = \frac{2}{\epsilon} \begin{bmatrix}
            1 & -1 & -1 & 1
        \end{bmatrix}^\top.
    \end{align*}
    Then,
    \begin{align*}
        \FFN(x) &= W_3 \ReLU(W_1 x + b_1) + b_3\\
        &= 2/\epsilon \left(\ReLU(x_i - u - \epsilon ) - \ReLU(x_i - u + 0.5)\right)\\
        &= \begin{cases}
           0 & \text{if } x_i < u - \epsilon, \\
1, & \text{if } x_i > u, \\
        \end{cases}
    \end{align*}
\end{proof}

\begin{lemma}[Interval]
\label{lem:interval}
Given any constant $u, v$ and $\epsilon > 0$, there exists a parameter configuration of a 1-layer feedforward neural network with ReGLU activation and width $4$ that for any input $x \in \R^{\dimension}$, computes the indicator function $x_i \in [u, v]$ on $x_i \in [-\infty, u - \epsilon] \cup [u,v] \cup [v, \infty]$.
\end{lemma}

\begin{proof}
The interval function here can be written as the difference of two threshold functions. We can use the same construction as in \cref{lem:threshold} to approximate the indicator function $x_i > u$ and $x_i > v + \epsilon$ and then take the difference.
\end{proof}
\subsection{Building Blocks of Transformers Construction}
\label{app:lem_transformer}

We will show in this section some construction for basic functionality using Transformer Blocks. This construction will be used in the following sections to prove the main results. 

We will always use $\hiddenstate \in \R_{\precision}^{\dimension \times \position}$ as the input to the Transformer Block, where $\dimension$ is the dimension of the input, and $\position$ is the length of the sequence.
We will first outline all the building functionality and then show how to implement them.

\begin{definition}[Copying Function]
\label{def:copy}
For integer $s$, index set $I_1, I_2 \subset [d - 20]$ satisfying $|I_1| = |I_2|$, a copying function $\COPY[s, I_1, I_2]$ satisfies the following, $\forall \hiddenstate \in \R_{\precision}^{\dimension \times \position}$, then
\begin{align*}
    \COPY[s, I_1, I_2](\hiddenstate)_{I_2, \iterpos} &= x_{I_1, \max\{\iterpos - s, 0\}} \quad \forall \iterpos \le [m] \\
    \COPY[s, I_1, I_2](\hiddenstate)_{I_2^c, \iterpos} &= 0 \quad \forall r \in [m]
\end{align*}
\end{definition}

\begin{definition}[Counting Function]
\label{def:count}
For index set $I_1, I_2 \subset [d - 20], |I_1| = |I_2| \le 10$ and index $i$, a counting function $\COUNT[I_1, I_2, i]$ satisfies the following, if $\forall v \in I_1 \cup I_2, \iterpos \in [l], \hiddenstate_{v,\iterpos} \in \Integer$ and $\hiddenstate_{v, \iterpos} \neq 0$, then 
\begin{align*}
    \COUNT[I_1, I_2, i](\hiddenstate)_{i, \iterpos} &= \frac{1}{\sum_{h=1}^{\iterpos} \one{\hiddenstate_{I_1, h} = \hiddenstate_{I_2, \iterpos}} + 1} \quad \forall \iterpos \in [l]. \\
    \COUNT[I_1, I_2, i](\hiddenstate)_{i^c, \iterpos} &= 0 \quad \forall \iterpos \in [l].
\end{align*}
\end{definition}

\begin{definition}[Matching Function]
\label{def:match}
    For index set $I_1, I_2, I_3, I_4 \subset [d - 20], |I_1| = |I_2| \le 10, |I_3| = |I_4|$, a matching function $\Match[I_1, I_2, I_3, I_4]$ satisfies the following, if $\forall v \in I_1 \cup I_2, \iterpos \in [l], \hiddenstate_{v, \iterpos} \in \Integer$, then
    \begin{align*}
        \Match[I_1, I_2, I_3, I_4](x)_{I_3, \iterpos} &= \hiddenstate_{I_4, \iterpos^*} \quad \forall \iterpos \in [l] \\
        \text{where } \iterpos^* &= \begin{cases} \min \{h \mid X_{I_1,h} = X_{I_2,\iterpos}\}, \{h \mid X_{I_1,h} = X_{I_2,\iterpos}\}\neq \emptyset  \\ 1 \text{, otherwise}\end{cases} . 
    \end{align*}
\end{definition}

\begin{definition}[Matching Closest Function]
\label{def:matchclosest}
    For index set $I_1, I_2, I_3, I_4 \subset [d - 20], |I_1| = |I_2| \le 10, |I_3| = |I_4|$, a matching closest function $\Match[I_1, I_2, I_3, I_4]$ satisfies the following, if $\forall v \in I_1 \cup I_2, \iterpos \in [l], \hiddenstate_{v, \iterpos} \in \Integer$, then
    \begin{align*}
        \MatchClose[I_1, I_2, I_3, I_4](x)_{I_3, \iterpos} &= \hiddenstate_{I_4, \iterpos^*} \quad \forall \iterpos \in [l] \\
        \text{where } \iterpos^* &= \begin{cases} \max \{h \mid X_{I_1,h} = X_{I_2,\iterpos}\}, \{h \mid X_{I_1,h} = X_{I_2,\iterpos}\}\neq \emptyset  \\ 1 \text{, otherwise}\end{cases} . 
    \end{align*}
\end{definition}

\begin{definition}[Matching Nearest Function]
\label{def:matchnearest}
    For index set $I_1, I_2, I_3, I_4 \subset [d - 20], |I_1| = |I_2| \le 10, |I_3| = |I_4|$ and index $i$, a matching nearest function $\MatchNearest[I_1, I_2, I_3, I_4, i]$ satisfies the following, if $\forall v \in I_1 \cup I_2, \iterpos \in [l], \hiddenstate_{v, \iterpos} \in \Integer$, then
    \begin{align*}
        &\MatchNearest[I_1, I_2, I_3, I_4](x)_{I_3, \iterpos} = \hiddenstate_{I_4, \iterpos^*} \quad \forall \iterpos \in [l] \\
        &\text{where } \iterpos^* = \begin{cases} \argmin_{ h \in \{h \mid X_{I_1,h} = X_{I_2,\iterpos}\}} |h - X_{i, \iterpos} |, \{h \mid X_{I_1,h} = X_{I_2,\iterpos}\}\neq \emptyset  \\ 1 \text{, otherwise}\end{cases} . 
    \end{align*}
\end{definition}

\begin{definition}[Matching Next Function]
\label{def:matchnextfunc}
    Given any interger constant $A$, assuming $\precision > 10A\log n$, for index set $I_1, I_2, I_3, I_4 \subset [d - 20], |I_1| = |I_2| \le 10, |I_3| = |I_4|$, and a special counting index $a$, a matching next function $\MatchNext[ I_1, I_2, I_3, I_4, a]$ satisfies the following, if $\hiddenstate$ satisfies the following condition:
    \begin{enumerate}[leftmargin=*]
        \item $\forall v \in I_1 \cup I_2, \iterpos \in [l], \hiddenstate_{v, \iterpos} \in \Integer$,
        \item $\hiddenstate_{a, \iterpos} \in \ROUND(1/[n^A], \precision) \cup \{0\}$, 
        \item For any $\iterpos \in [\position]$, given any $k \ge \iterpos$,  the counting index multiset $S_{\iterpos} = \{\hiddenstate_{a, \iterpos'} \mid \hiddenstate_{I_1, \iterpos'} = \hiddenstate_{I_2, \iterpos}\}$ takes consecutive and disjoint values in $\ROUND(1/[n^A], \precision)$, that is, there exists $u_{\iterpos}, v_{\iterpos} \in \ROUND(1/[n^A], \precision)$ such that $S_{\iterpos} = [u_{\iterpos}, v_{\iterpos}] \cap \ROUND(1/[n^A], \precision)$.
    \end{enumerate}
    then, we have 
    \begin{align*}
        &\MatchNext[ I_1, I_2, I_3, I_4, a](\hiddenstate)_{I_3, \iterpos} = \hiddenstate_{I_4, \iterpos^*} \quad \forall \iterpos \in [l] \\
        &\text{where } \iterpos^* = \argmin_{h \in \{h \mid X_{I_1,h} = X_{I_2,\iterpos}\} \cup \{1\}} |X_{a, h} - \nextnum(X_{a, \iterpos})| . 
    \end{align*}
\end{definition}

Now we will show how to implement these functions using Transformer Blocks. The construction here is motivated by the construction in~\cite{feng2023towards} with some modifications.

\begin{lemma}[Copying Blocks]
\label{lem:copying}
    For integer $s$, index set $I_1, I_2 \subset [d - 10]$ satisfying $|I_1| = |I_2|$, a copying function $\COPY[s, I_1, I_2]$ can be implemented with $1$ feedforward block $\FFN$ and $1$ attention block $\Attention$ with $1$ attention head. Formally, when $\hiddenstate_{\dimension, \iterpos} = \iterpos$, it holds that
    \begin{align*}
        \Attention\left(\FFN\left(\hiddenstate\right) + \hiddenstate\right) = \COPY[s, I_1, I_2](\hiddenstate).
    \end{align*}
\end{lemma}

\begin{proof}[Proof of~\Cref{lem:copying}]
    We will use the feedforward block to calculate $\hiddenstate_{\iterpos, \dimension}^2$ and $1$ (\Cref{lem:multiplication}) and have 
    \begin{align*}
        \left(\FFN(\hiddenstate) + \hiddenstate\right)_{\dimension - 1, \iterpos} &= k^2 \\
        \left(\FFN(\hiddenstate) + \hiddenstate\right)_{\dimension - 2, \iterpos} &= 1. \\
        \forall i \not \in \{\dimension - 1, \dimension - 2\}, \left(\FFN(\hiddenstate) + \hiddenstate\right)_{i, \iterpos} &= \hiddenstate_{i, \iterpos}.
    \end{align*}

    We will use $\hiddenstate'$ to denote $\FFN(\hiddenstate) + \hiddenstate$.
    Then we will choose $\Key, \Query$ such that
    \begin{align*}
        \Key \hiddenstate_{:, \iterpos'}' = n\begin{bmatrix} 1 \\ \iterpos' \\ \iterpos'^2  \end{bmatrix} \\
        \Query \hiddenstate'_{:, \iterpos} = \begin{bmatrix} -(\iterpos^2 + s^2 - 2s\iterpos)\\ 2\iterpos - 2s \\ -1  \end{bmatrix} 
    \end{align*}

    Hence
    \begin{align*}
        &\left(\left(\Key \hiddenstate \right)^\top \left(\Query \hiddenstate\right)\right)_{\iterpos', \iterpos} \\ =& -n\left( \iterpos'^2  - \iterpos'(2\iterpos - 2s) + \iterpos^2  + s^2  -2s\iterpos\right) \\ =& -n(\iterpos - s - \iterpos')^2 
    \end{align*}

    Hence we have
    \begin{align*}
        \argmax_{\iterpos' < \iterpos } \left(\left(\Key \hiddenstate \right)^\top \left(\Query \hiddenstate\right)\right)_{\iterpos', \iterpos} &= \max \{\iterpos -s, 0\}.
    \end{align*}

    Also, for any $\iterpos' \le \iterpos, \iterpos' \ne \max \{\iterpos -s, 0\}$, we have
    \begin{align*}
        \left(\left(\Key \hiddenstate \right)^\top \left(\Query \hiddenstate\right)\right)_{\iterpos', \iterpos} - \left(\left(\Key \hiddenstate \right)^\top \left(\Query \hiddenstate\right)\right)_{\max \{\iterpos -s, 0\}, \iterpos} &< -n.
    \end{align*}

    Hence after the column-wise softmax and rounding to $\precision = O(\log n)$ bit, we have 
    \begin{align*}
        \left(\Softmax\left(\left(\Key \hiddenstate \right)^\top \left(\Query \hiddenstate\right)\right)\right)_{\iterpos', \iterpos} &= \one{\iterpos' = \max \{\iterpos -s, 0\}}.
    \end{align*}

    We will then choose $\Value$ such that
    \begin{align*}
        \Value \hiddenstate_{I_2, \iterpos'}' = \hiddenstate'_{I_1, \iterpos'}  = \hiddenstate_{I_1, \iterpos'} \quad \forall \iterpos' \in [l]. \\
        \Value \hiddenstate_{I_2^c, \iterpos'}' = 0 \quad \forall \iterpos' \in [l].
    \end{align*}

    This then concludes that
    \begin{align*}
        \Attention\left(\FFN\left(\hiddenstate\right) + \hiddenstate\right) = \COPY[s, I_1, I_2](\hiddenstate).
    \end{align*}
    The proof is complete.
\end{proof}

\begin{lemma}[Counting Blocks]
\label{lem:count}
     For index set $I \subset [d - 20]$ satisfying $|I_1| = |I_2| \le 10$, a counting function $\COUNT[i, I_1, I_2]$ can be approximated with $1$ feedforward block $\FFN$ and $1$ attention block $\Attention$ with $1$ attention head. Formally, when $\hiddenstate_{\dimension, \iterpos} = \iterpos$ and $\hiddenstate_{3, \iterpos} = \one{\iterpos = 1}, \hiddenstate_{I_1, 1} = 0$, it holds that
    \begin{align*}
        \Attention\left(\FFN\left(\hiddenstate\right) + \hiddenstate\right)_{i, k} &= \ROUND\left(\COUNT[s, I_1, I_2](\hiddenstate)_{i,k}, \precision \right). \\
        \Attention\left(\FFN\left(\hiddenstate\right) + \hiddenstate\right)_{i^c, k} &= 0.
    \end{align*}
\end{lemma}

\begin{proof}[Proof of~\Cref{lem:count}]
We will use the feedforward block to calculate $\hiddenstate_{v, \iterpos}^2, v \in I_1 \cup I_2$ (\Cref{lem:multiplication}) and have
\begin{align*}
    \left(\FFN(\hiddenstate) + \hiddenstate\right)_{\dimension - i, \iterpos} &= \hiddenstate_{I_1[i], \iterpos}^2, i \in [|I|]. \\
    \left(\FFN(\hiddenstate) + \hiddenstate\right)_{\dimension - |I| - i, \iterpos} &= \hiddenstate_{I_2[i], \iterpos}^2, i \in [|I|]. \\
    \left(\FFN(\hiddenstate) + \hiddenstate\right)_{\dimension - 2|I| - 1, \iterpos} &= 1. 
    \\
    \forall i \not \in \{\dimension - i \mid i \in [2|I| + 1]\}, \left(\FFN(\hiddenstate) + \hiddenstate\right)_{i, \iterpos} &= \hiddenstate_{i, \iterpos}.
\end{align*}

We will use $\hiddenstate'$ to denote $\FFN(\hiddenstate) + \hiddenstate$.
Then we will choose $\Key, \Query$ such that
\begin{align*}
    \Key \hiddenstate_{:, \iterpos'}' &= n\begin{bmatrix} 1 + \one{\iterpos' = 1}\\ \hiddenstate_{I_1[i], \iterpos'}  \\ \hiddenstate_{I_1[i], \iterpos'}^2  \end{bmatrix}_{i \in [I]} \\
    \Query \hiddenstate'_{:, \iterpos} &= \begin{bmatrix}\hiddenstate_{I_2[i], \iterpos}^2\\ -\hiddenstate_{I_2[i], \iterpos}  \\  1  \end{bmatrix}_{i \in [I]}
\end{align*}

Hence,

\begin{align*}
    &\left(\left(\Key \hiddenstate \right)^\top \left(\Query \hiddenstate\right)\right)_{\iterpos', \iterpos} \\ =& -n\sum_{i = 1}^{|I|} \left( \hiddenstate_{I_2[i], \iterpos'}'^2  - \hiddenstate_{I_1[i], \iterpos'}(2\hiddenstate_{I_2[i], \iterpos})  + \hiddenstate_{I_2[i], \iterpos}^2\right) + n \one{\iterpos' = 1} \sum_{i = 1}^{|I|} \hiddenstate_{I[i], \iterpos}^2. \\ =& -n \sum_{i = 1}^{|I|} (\hiddenstate_{I_1[i], \iterpos'} - \hiddenstate_{I_2[i], \iterpos})^2 + n \one{\iterpos' = 1} \sum_{i = 1}^{|I|} \hiddenstate_{I_2[i], \iterpos}^2.
\end{align*}

Hence we have
\begin{align*}
    \max_{\iterpos' < \iterpos } \left(\left(\Key \hiddenstate \right)^\top \left(\Query \hiddenstate\right)\right)_{\iterpos', \iterpos} &= 0.
\end{align*}

Equality holds when $\iterpos' = 1$ or $\hiddenstate_{I_1[i], \iterpos'} = \hiddenstate_{I_2[i], \iterpos}$ for all $i \in [|I_1|]$.

Also, for any $\iterpos' \le \iterpos, \iterpos' \ne 1$ or $\hiddenstate_{I_1[i], \iterpos'} \ne \hiddenstate_{I_2[i], \iterpos}$ for some $i \in [|I_1|]$, we have
\begin{align*}
    \left(\left(\Key \hiddenstate \right)^\top \left(\Query \hiddenstate\right)\right)_{\iterpos', \iterpos}  &< -n.
\end{align*}

Hence after the column-wise softmax and rounding to $\precision = O(\log n)$ bit, we have 
\begin{align*}
    \left(\Softmax\left(\left(\Key \hiddenstate \right)^\top \left(\Query \hiddenstate\right)\right)\right)_{\iterpos', \iterpos} &= \ROUND\left(\frac{1}{\sum_{h=1}^{\iterpos} \one{\hiddenstate_{I_1, h} = \hiddenstate_{I_2, \iterpos}} + 1}, \precision\right)
\end{align*}

Here the $O\left(\frac{1}{n^{A}}\right)$ term comes from the fact that the softmax is rounded to $\precision = O(\log n)$ bit.

We will then choose $\Value$ such that
\begin{align*}
    \Value \hiddenstate_{i, \iterpos'}' = \hiddenstate'_{3, \iterpos'}  = \one{\iterpos' = 1} \quad \forall \iterpos' \in [l]. \\
    \Value \hiddenstate_{I^c, \iterpos'}' = 0 \quad \forall \iterpos' \in [l].
\end{align*}

This then concludes that
\begin{align*}
    \Attention\left(\FFN\left(\hiddenstate\right) + \hiddenstate\right)_{i, k} &= \ROUND\left(\COUNT[s, I_1, I_2](\hiddenstate)_{i,k}, \precision \right). \\
    \Attention\left(\FFN\left(\hiddenstate\right) + \hiddenstate\right)_{i^c, k} &= 0.
\end{align*}

\end{proof}

\begin{lemma}[Matching Blocks]
\label{lem:match}
    Given any constant $c$, for index set $I_1, I_2, I_3, I_4 \subset [d - 20], |I_1| = |I_2| \le 10, |I_3| = |I_4|$, a matching function $\Match[I_1, I_2, I_3, I_4]$ can be implemented with $1$ feedforward block $\FFN$ and $1$ attention block $\Attention$ with $1$ attention head. Formally, when $\hiddenstate_{\dimension, \iterpos} = \iterpos, \hiddenstate_{3, \iterpos} = \one{\iterpos = 1}, \hiddenstate_{I_1, 1} = 0$ and $\iterpos \le n^c$, it holds that
    \begin{align*}
        \Attention\left(\FFN\left(\hiddenstate\right) + \hiddenstate\right) &= \Match[I_1, I_2, I_3, I_4](\hiddenstate)
    \end{align*}
\end{lemma}

\begin{proof}
    We will use the feedforward block to calculate $\iterpos^2, \hiddenstate_{v, \dimension}^2, v \in \cup I_1 \cup I_2$ as in the proof of~\Cref{lem:count,lem:copying}.

    We then choose $\Key, \Query$ such that
    \begin{align*}
        &\left(\left(\Key \hiddenstate \right)^\top \left(\Query \hiddenstate\right)\right)_{\iterpos', \iterpos} \\
        =& -n^{4c + 1} \sum_{i = 1}^{|I|} (\hiddenstate_{I_1[i], \iterpos'} - \hiddenstate_{I_2[i], \iterpos})^2  - n  \iterpos'^2\\
         &+  \one{\iterpos' = 1} \left( n^{4c + 1}\sum_{i = 1}^{|I|} \hiddenstate_{I_2[i], \iterpos}^2 + n - n^{2c + 2}\right).
    \end{align*}

    The detailed construction of $\Key, \Query$ is omitted here since it is similar to the proof of~\Cref{lem:count,lem:copying}.

    We will discuss several cases for the distribution of $\left(\left(\Key \hiddenstate \right)^\top \left(\Query \hiddenstate\right)\right)_{\iterpos', \iterpos}$. It always holds that $\left(\left(\Key \hiddenstate \right)^\top \left(\Query \hiddenstate\right)\right)_{1, \iterpos} = -n^{2c + 2}$.

    \begin{enumerate}
        \item If there doesn't exists $\iterpos'$, such that $\hiddenstate_{\iterpos',I_1} = \hiddenstate_{\iterpos, I_2}$, then for any $i > 1$, we have $\left(\left(\Key \hiddenstate \right)^\top \left(\Query \hiddenstate\right)\right)_{i, \iterpos} < -n^{4c + 1}$.
        \item If there exists $\iterpos'$, such that $\hiddenstate_{\iterpos',I_1} = \hiddenstate_{\iterpos, I_2}$, then for such $\iterpos'$, we have $\left(\left(\Key \hiddenstate \right)^\top \left(\Query \hiddenstate\right)\right)_{\iterpos', \iterpos} = -n\iterpos'^2 > -n^{2c + 1}$. The rest of the entries are all smaller than $-n^{4c + 1}$. Finally, the largest $\iterpos'$ satisfying that $\hiddenstate_{\iterpos',I_1} = \hiddenstate_{\iterpos, I_2}$ will corresponds to a $\left(\left(\Key \hiddenstate \right)^\top \left(\Query \hiddenstate\right)\right)_{\iterpos', \iterpos}$ that is at least $n$ larger than the second largest $\left(\left(\Key \hiddenstate \right)^\top \left(\Query \hiddenstate\right)\right)_{\iterpos', \iterpos}$, as in the proof of~\Cref{lem:copying}. 
    \end{enumerate}

    Concluding the above discussion, we have after the column-wise softmax and rounding to $\precision = O(\log n)$ bit,
    \begin{align*}
        \left(\Softmax\left(\left(\Key \hiddenstate \right)^\top \left(\Query \hiddenstate\right)\right)\right)_{\iterpos', \iterpos} &= \begin{cases} \one{ \iterpos' = \min \{h \mid X_{I_1,h} = X_{I_2,\iterpos}\}}, \{h \mid X_{I_1,h} = X_{I_2,\iterpos}\}\neq \emptyset  \\ \one{\iterpos' = 1} \text{, otherwise}\end{cases} 
    \end{align*}

    Further, we will choose $\Value$ such that
    \begin{align*}
        \Value \hiddenstate_{I_3, \iterpos'}' = \hiddenstate'_{I_4, \iterpos'}  = \hiddenstate_{I_4, \iterpos'} \quad \forall \iterpos' \in [l]. \\
        \Value \hiddenstate_{I_3^c, \iterpos'}' = 0 \quad \forall \iterpos' \in [l].
    \end{align*}

    This then concludes that
    \begin{align*}
        \Attention\left(\FFN\left(\hiddenstate\right) + \hiddenstate\right) &= \Match[I_1, I_2, I_3, I_4](\hiddenstate)
    \end{align*}
    This concludes the proof.
\end{proof}

\begin{lemma}[Matching Closest Blocks]
    \label{lem:matchclose}
        Given any constant $c$, for index set $I_1, I_2, I_3, I_4 \subset [d - 20], |I_1| = |I_2| \le 10, |I_3| = |I_4|$, a matching closest function $\MatchClose[I_1, I_2, I_3, I_4]$ can be implemented with $1$ feedforward block $\FFN$ and $1$ attention block $\Attention$ with $1$ attention head. Formally, when $\hiddenstate_{\dimension, \iterpos} = \iterpos, \hiddenstate_{3, \iterpos} = \one{\iterpos = 1}, \hiddenstate_{I_1, 1} = 0$ and $\iterpos \le n^c$, it holds that
        \begin{align*}
            \Attention\left(\FFN\left(\hiddenstate\right) + \hiddenstate\right) &= \MatchClose[I_1, I_2, I_3, I_4](\hiddenstate)
        \end{align*}
    \end{lemma}
\begin{proof}
    The proof is similar to the proof of~\Cref{lem:match}, and one can design the attention pattern as
    \begin{align*}
        &\left(\left(\Key \hiddenstate \right)^\top \left(\Query \hiddenstate\right)\right)_{\iterpos', \iterpos} \\
        =& -n^{4c + 1} \sum_{i = 1}^{|I|} (\hiddenstate_{I_1[i], \iterpos'} - \hiddenstate_{I_2[i], \iterpos})^2  - n  (\iterpos - \iterpos')^2\\
         &+  \one{\iterpos' = 1} \left( n^{4c + 1}\sum_{i = 1}^{|I|} \hiddenstate_{I_2[i], \iterpos}^2 + n (\iterpos - 1)^2 - n^{2c + 2}\right).
    \end{align*}
    The rest of the proof is omitted here.
\end{proof}

\begin{lemma}[Matching Nearest Blocks]
    \label{lem:matchnear}
        Given any constant $c$, for index set $I_1, I_2, I_3, I_4 \subset [d - 20], |I_1| = |I_2| \le 10, |I_3| = |I_4|$ and index $i$ , a matching nearest function $\MatchNearest[I_1, I_2, I_3, I_4, i]$ can be implemented with $1$ feedforward block $\FFN$ and $1$ attention block $\Attention$ with $1$ attention head. Formally, when $\hiddenstate_{\dimension, \iterpos} = \iterpos, \hiddenstate_{3, \iterpos} = \one{\iterpos = 1}, \hiddenstate_{I_1, 1} = 0$ and $\iterpos \le n^c$, it holds that
        \begin{align*}
            \Attention\left(\FFN\left(\hiddenstate\right) + \hiddenstate\right) &= \MatchNearest[I_1, I_2, I_3, I_4, i](\hiddenstate)
        \end{align*}
    \end{lemma}
\begin{proof}
    The proof is similar to the proof of~\Cref{lem:match}, and one can design the attention pattern as
    \begin{align*}
        &\left(\left(\Key \hiddenstate \right)^\top \left(\Query \hiddenstate\right)\right)_{\iterpos', \iterpos} \\
        =& -n^{4c + 1} \sum_{u = 1}^{|I|} (\hiddenstate_{I_1[u], \iterpos'} - \hiddenstate_{I_2[u], \iterpos})^2  - n  (\hiddenstate_{i, \iterpos} - \iterpos')^2\\
         &+  \one{\iterpos' = 1} \left( n^{4c + 1}\sum_{u = 1}^{|I|} \hiddenstate_{I_2[u], \iterpos}^2 + n (1 - \hiddenstate_{i, \iterpos})^2 - n^{2c + 2}\right).
    \end{align*}
    The rest of the proof is omitted here.
\end{proof}

\begin{lemma}[Matching Next Blocks]
\label{lem:matchnext}
    Given any constant $A, c$, for index set $I_1, I_2, I_3, I_4 \subset [d - 20], |I_1| = |I_2| \le 10, |I_3| = |I_4|$ and a special counting index $a$, a matching next function $\MatchNext[ I_1, I_2, I_3, I_4, a]$ can implement with $1$ feedforward block $\FFN$ and $1$ attention block $\Attention$ with $1$ attention head. Formally, when $\hiddenstate_{\dimension, \iterpos} = \iterpos, \hiddenstate_{3, \iterpos} = \one{\iterpos = 1}, \hiddenstate_{I_1, 1} = 0$ and $\iterpos \le n^c$, it holds that
    \begin{align*}
        \Attention\left(\FFN\left(\hiddenstate\right) + \hiddenstate\right) &= \MatchNext[ I_1, I_2, I_3, I_4, a](\hiddenstate)
    \end{align*}
\end{lemma}

\begin{proof}

    We will use the feedforward block to calculate the following $\appnext$ function, where 
    \begin{align*}
        \appnext(x) = \begin{cases}
            \frac{1}{2}, & x \ge \frac{2}{3}. \\
            \frac{1}{3}, & \frac{3}{5} > x > \frac{2}{5}. \\
            \frac{1}{4}, & \frac{7}{20} > x > \frac{3}{10} \\
            x - x^2 + x^3, & x \le \frac{11}{40}.
        \end{cases}
    \end{align*}
    The value can be arbitrary for $x \in [\frac{11}{40}, \frac{3}{10}] \cup [\frac{2}{5}, \frac{7}{20}] \cup[\frac{3}{5}, \frac{2}{3}] $.
    This function is achievable by a feedforward block through combination of~\Cref{lem:multiplication,lem:interval}.

    The purpose of this is to approximate the $\nextnum$ function for $x \in \ROUND(1/[n^A], \precision)$, and we have the following lemma.

    \begin{lemma}
    \label{lem:approx_next}
    For large enough $n$ and any $x \in \ROUND(1/[n^A], \precision)$, we have
    \begin{align*}
        |\appnext(x) - \nextnum(x) | &\le  \nextnum(x)^3 + O\left(\frac{1}{n^{10A}}\right). 
    \end{align*}

    \end{lemma}

    \begin{proof}
        We always have $\ROUND(1/[n^A], \precision) \cap \left( [\frac{11}{40}, \frac{3}{10}] \cup [\frac{2}{5}, \frac{7}{20}] \cup[\frac{3}{5}, \frac{2}{3}]\right) = \emptyset$.
        We will discuss several cases for $x \in \ROUND(1/[n^A], \precision)$.
        \begin{enumerate}
            \item If $x \ge \frac{3}{10}$, then $\appnext(x) = \nextnum(x)$.
            \item If $x \le \frac{7}{20}$, it holds that $|x - 1/m| \le 1/n^{10A}, m \ge 3$, then 
            \begin{align*}
                \appnext(x) &= x - x^2 = \frac{1}{m} - \frac{1}{m^2} + \frac{1}{m^3} + O\left(\frac{1}{n^{10A}}\right) \\
                &= \frac{1}{m + 1} - \frac{1}{m^3(m + 1)} + O\left(\frac{1}{n^{10A}}\right)
            \end{align*}
        \end{enumerate}
        This then concludes the proof.
    \end{proof}

    We then choose $\Key, \Query$ such that
    \begin{align*}
        &\left(\left(\Key \hiddenstate \right)^\top \left(\Query \hiddenstate\right)\right)_{\iterpos', \iterpos} \\
        =& -n^{4A + 3} \sum_{i = 1}^{|I|} (\hiddenstate_{I_1[i], \iterpos'} - \hiddenstate_{I_2[i], \iterpos})^2  -n^{4A + 1} (\appnext(\hiddenstate_{a, \iterpos}) - \hiddenstate_{a, \iterpos'})^2\\
         &+  \one{\iterpos' = 1} \left( n^{4A + 3}\sum_{i = 1}^{|I|} \hiddenstate_{I_2[i], \iterpos}^2 + n^{4A + 1} X_{a, \iterpos}^2 - n^{4A + 2}\right).
    \end{align*}

    Again, the detailed construction of $\Key, \Query$ is omitted here since it is similar to the proof of~\Cref{lem:count,lem:copying}.

    We will discuss several cases for the distribution of $\left(\left(\Key \hiddenstate \right)^\top \left(\Query \hiddenstate\right)\right)_{\iterpos', \iterpos}$. It always holds that $\left(\left(\Key \hiddenstate \right)^\top \left(\Query \hiddenstate\right)\right)_{1, \iterpos} = -n^{4A + 2}$.
    \begin{enumerate}[leftmargin=*]
        \item If there doesn't exists $\iterpos'$, such that $\hiddenstate_{\iterpos',I_1} = \hiddenstate_{\iterpos, I_2}$, then for any $i > 1$, we have $\left(\left(\Key \hiddenstate \right)^\top \left(\Query \hiddenstate\right)\right)_{i, \iterpos} < -n^{4A + 3}$.
        \item If there exists $\iterpos'$, such that $\hiddenstate_{\iterpos',I_1} = \hiddenstate_{\iterpos, I_2}$, then for such $\iterpos'$, we have $\left(\left(\Key \hiddenstate \right)^\top \left(\Query \hiddenstate\right)\right)_{\iterpos', \iterpos} = -n^{3A}(\nextnum(\hiddenstate_{a, \iterpos}) - \hiddenstate_{a, \iterpos'})^2 > -n^{4A + 1}$. The rest of the entries are all smaller than $-n^{4A+2}$. 
        
        It remains to discuss the distribution of $\left(\left(\Key \hiddenstate \right)^\top \left(\Query \hiddenstate\right)\right)_{\iterpos', \iterpos}$ for $\iterpos'$ satisfying $\hiddenstate_{\iterpos',I_1} = \hiddenstate_{\iterpos, I_2}$. When $\hiddenstate$ satisfies the condition in~\Cref{def:matchnextfunc}, we have that $S_{\iterpos} = \{X_{a, \iterpos'} \mid \hiddenstate_{\iterpos',I_1} = \hiddenstate_{\iterpos, I_2}\}$ takes consecutive and disjoint values in $\ROUND(1/[n^A], \precision)$. Hence, if $|S_{\iterpos}| > 2$, suppose $y, z \in S_{\iterpos}$ satisfies that
        \begin{align*}
            |y - \appnext(\hiddenstate_{a, \iterpos})| &= \min_{x \in S_{\iterpos}} |x - \appnext(\hiddenstate_{a, \iterpos})| \\
            |z - \appnext(\hiddenstate_{a, \iterpos})| &= \min_{x \in S_{\iterpos}, x \neq y} |x - \appnext(\hiddenstate_{a, \iterpos})|.
        \end{align*}

        We will discuss several cases for $y, z$.

        \begin{itemize}
            \item If $y - \appnext(\hiddenstate_{a, \iterpos})$ and $z - \appnext(\hiddenstate_{a, \iterpos})$ are both negative, then $y > z$, we have,
            \begin{align*}
                \left(y - \appnext(\hiddenstate_{a, \iterpos})\right)^2 - \left(z - \appnext(\hiddenstate_{a, \iterpos})\right)^2 &= (y - z)(y + z - 2\appnext(\hiddenstate_{a, \iterpos})) \\
                &\le -(y - z)^2 \le -\frac{1}{4n^{4A}}.
            \end{align*}.
            \item If $y - \appnext(\hiddenstate_{a, \iterpos})$ and $z - \appnext(\hiddenstate_{a, \iterpos})$ are both positive, then $y < z$, and same as above we have
            \begin{align*}
                \left(y - \appnext(\hiddenstate_{a, \iterpos})\right)^2 - \left(z - \appnext(\hiddenstate_{a, \iterpos})\right)^2 &= (y - z)(y + z - 2\appnext(\hiddenstate_{a, \iterpos})) \\
                &\le -(y - z)^2 \le -\frac{1}{4n^{4A}}.
            \end{align*}.
            \item If $y - \appnext(\hiddenstate_{a, \iterpos})$ and $z - \appnext(\hiddenstate_{a, \iterpos})$ have different signs, then according to~\Cref{lem:approx_next}, we have, $y = \ROUND(\nextnum(\hiddenstate_{a, \iterpos}), \precision)$ because $S_{\iterpos}$ takes consecutive and disjoint values in $\ROUND(1/[n^A], \precision)$. This then implies that 
            \begin{align*}
                &\left(y - \appnext(\hiddenstate_{a, \iterpos})\right)^2 - \left(z - \appnext(\hiddenstate_{a, \iterpos})\right)^2 \\ 
                \le& O(\frac{1}{n^{10A}}) + \frac{1}{\nextnum^6(\hiddenstate_{a, \iterpos})} - \left(\frac{1}{\nextnum(\hiddenstate_{a, \iterpos})(\nextnum(\hiddenstate_{a, \iterpos}) + 1)}\right)^2 \\
                \le& -\frac{1}{4n^{4A}}. 
            \end{align*}
        \end{itemize}

        Concluding, we always have for any $\iterpos'' \neq \iterpos^* = \argmax_{\iterpos', \iterpos} \left(\left(\Key \hiddenstate \right)^\top \left(\Query \hiddenstate\right)\right)_{\iterpos', \iterpos}$
        \begin{align*}
            \left(\left(\Key \hiddenstate \right)^\top \left(\Query \hiddenstate\right)\right)_{\iterpos', \iterpos} - \left(\left(\Key \hiddenstate \right)^\top \left(\Query \hiddenstate\right)\right)_{\iterpos^*, \iterpos} \le -\frac{n}{4}.
        \end{align*}
    \end{enumerate}

    Concluding the above discussion, we have after the column-wise softmax and rounding to $\precision = O(\log n)$ bit,
    \begin{align*}
        \left(\Softmax\left(\left(\Key \hiddenstate \right)^\top \left(\Query \hiddenstate\right)\right)\right)_{\iterpos', \iterpos} &= \one{\iterpos' = \argmin_{h \in \{h \mid X_{I_1,h} = X_{I_2,\iterpos}\} \cup \{1\}} |X_{a, h} - \nextnum(X_{a, \iterpos})|}.
    \end{align*}

    Further, we will choose $\Value$ such that
    \begin{align*}
        \Value \hiddenstate_{I_3, \iterpos'}' = \hiddenstate'_{I_4, \iterpos'}  = \hiddenstate_{I_4, \iterpos'} \quad \forall \iterpos' \in [l]. \\
        \Value \hiddenstate_{I_3^c, \iterpos'}' = 0 \quad \forall \iterpos' \in [l].
    \end{align*}
    This then concludes the proof.
\end{proof}

\subsection{Building Blocks of RNNs Construction}

We will now describe the building blocks of Linear RNNs construction. We will introduce some basic operations that will be used to build more complex RNNs family.

\begin{lemma}[Recent Input Memorizing]
\label{lem:recent_input_memorizing}
    Given any constant $k$ and $C$, there exists a parameter configuration of linear unit that maintains $C$ dimensions of last $k$ input vectors in the state.
\end{lemma}

\begin{proof}
    Suppose the input sequence is $x_{1:t} \in \R^{\dimension}$, and the dimensions that the state should memorize are $d_1, d_2, \ldots, d_C$. We can construct the following linear unit:
    \begin{align*}
        h_{t} = \begin{bmatrix}  x_{t-1,d_1} & \ldots & x_{t-1,d_C} & h_{t-1, 1} & \ldots & h_{t-1, C \times (k -1)} & h_{t-1, C \times k + 1} & \ldots & h_{t-1, \dimension}  \end{bmatrix}.
    \end{align*}
\end{proof}

\begin{lemma}[Summation]
\label{lem:summation}
    Given any constant $k$ and $C$, there exists a parameter configuration of linear unit that maintains the sum of one dimension of the last $k$ input vectors in the state.
\end{lemma}

\begin{proof}
    Suppose WLOG the input sequence is $x_{1:t} \in \R^{\dimension}$, and the dimension that the state should memorize is $1$. We can construct the following linear unit:
    \begin{align*}
        h_{t} = \begin{bmatrix}  x_{t-1,1} + h_{t-1, 1} & h_{t-1, 2} \ldots & h_{t-1, \dimension}  \end{bmatrix}.
    \end{align*}
\end{proof}

\begin{lemma}[Special Position Memorizing]
\label{lem:special_position_memorizing}
    Given any constant $k$ and $C$, there exists a parameter configuration of linear unit and a FFN with $Ck$ width that maintains the $C$ dimensions of the input vector at position $1$ to $k$ in the state.
\end{lemma}

\begin{proof}
This is a direct combination of~\Cref{lem:summation} and~\Cref{lem:lookup_table,lem:multiplication}. The FFN can assign all the input vectors with position greater than $k$ to $0$, and permute the corresponding dimensions of first $k$ input vectors to the first $Ck$ dimensions of the state. The linear unit can then maintain the state.
\end{proof}

\begin{lemma}[Recite Fixed Sequence]
\label{lem:recite_fixed_sequence}
    Given any constant integer $k$ and $C$, there exists a FFN with width $kC$ that can output fixed sequence of scalars that takes values in $[C]$ on a fixed set of positions $\position_1, \ldots \position_k$.
\end{lemma}

\begin{proof}
    This is a direct consequence of~\Cref{lem:lookup_table}.
\end{proof}

\subsection{Proof of~\Cref{thm:index}} \label{sec:thmindex}

We will first restate the theorem for clarity.

\thmindex*

\begin{proof}
    We will discuss by cases.

    When $T$ is Index, we will first show why RNN cannot solve the Index question without $\Omega(n)$ memory. The key observation is that the recurrent form of RNNs allowed the algorithm to be run in a streaming fashion with $o(n)$ bit memory. Here streaming means that the algorithm gets to look at each bit of the memory sequentially and can only update a constant size of memory. 
    \begin{lemma}
    \label{lem:index}
        Consider the following two-party game, where Alice receives string $x \in \{0,1\}^n$ and Bob receives an integer $k$, and Bob wants to know the value of $x_k$. If only Alice is allowed to send a signal to Bob, then $\Omega(n)$ bit communication complexity is required.
    \end{lemma}
    
    \begin{proof}[Proof of~\Cref{lem:index}]
        Suppose there exists a communication protocol where $B$ only receives $o(n)$ bit and can perfectly decide $x_k$. Because Alice doesn't know $k$, the protocol must send the same message to Bob for all $k$. Hence Bob can reconstruct the whole string $x$ with $n$ bit with $o(n)$ bit communication. This is a contradiction.
    \end{proof}

    Now if RNN can solve the Index problem with $o(n)$ bit memory, then it can also solve the Index problem with $o(n)$ bit communication complexity. This is because Alice can simply run the RNN on input $x$ and send the hidden state to Bob. Then Bob can run the RNN with the hidden state and $k$ to get the answer. This is a contradiction to~\Cref{lem:index}. Hence RNN cannot solve the Index problem with $o(n)$ bit memory.

    On the other hand, we will show that Transformers can solve the Index problem with $O(\log n)$ bit parameters. This is because using 2 layers of Transformer, we will implement a Match Block (\Cref{lem:match}) that can match the last query token with the position of the previous token and retrieve the type of the matched token to the query token. 

    When $T$ is AR, wthout loss of generality, we assume that $n$ is even. The proof is similar to the proof of the proof of theIndex problem. As there are $n$ different types of tokens, we can label them as $[n]$. Now for any boolean sequence $x \in \{0,1\}^{n/2}$, solving AR for the following input is equivalent to solving the Index problem for $x$:
\begin{align*}
    \InputTokenSequence = \StartSentence, 1, x_1 + n/2, 2, x_2 + n/2, \ldots, n/2, x_{n/2} +  n/2, k
\end{align*}

This then implies that RNN cannot solve AR with $o(n)$ bit memory. Transformers, on the other hand, can still solve AR with $O(\log n)$ bit parameters, we will use one layer of copying function to copy each token's previous token's type to it. Then we can use the Match Block to match the last query token with the position of the previous token and retrieve the type of the matched token to the query token.

When $T$ is $c$-gram retrieval,  without loss of generality, we assume that $n$ is a multiple of $c$. The proof is similar to the proof of~\Cref{thm:index}. As there are $n$ different types of tokens, we can label them as $[n]$. Now for any boolean sequence $x \in \{0,1\}^{n/2}$, solving AR for the following input is equivalent to solving the Index problem for $x$:
\begin{align*}
    \InputTokenSequence = \StartSentence, \underbrace{1, \ldots, 1}_{c - 1}, x_1 + n/c, \underbrace{2, \ldots, 2}_{c - 1}, x_2 + n/c, \ldots, \underbrace{n/c, \ldots, n/c}_{c - 1}, x_{n/c} +  n/c, \underbrace{k, \ldots, k}_{c - 1}
\end{align*}

This then implies that RNN cannot solve $c$-gram retrieval with $o(n)$ bit memory. Transformers, on the other hand, can still solve $c$-gram retrieval with $O(\log n)$ bit parameters, we will use one layer of copying function to copy each token's previous $c-1$ tokens' type to it. Then we can use the Match Block to match the last query token with the position of the previous token and retrieve the type of the matched token to the query token.

When $T$ is Counting, we will first show why RNN cannot solve the Counting question without $\Omega(n)$ memory. Consider the following setup, given any $x \in \{0,1\}^n$, the input string is $j_1 j_2 \ldots j_k$ where $\{j_i \ldots j_k\} = \{j \mid x_j = 1\}$, then solving the Counting question for this input string for queried threshold $1$ is equivalent to solving the Index problem for $x$. This then implies that RNN cannot solve the Counting question with $o(n)$ bit memory.

On the other hand, we will show that Transformers can solve the Counting question with $O(\log n)$ bit parameters. This is because using 2 layers of Transformer, we can first use a $\COPY$ block to copy the last query token to the token corresponds to the threshold, and then use a $\COUNT$ block (\Cref{lem:count}) that can count the number $m$ of the appearance of the last query token in the input sequence, and then write $1/ (m + 1)$ to one of the dimension. Finally, we can use the Feed Forward Network on the last layer to multiply threshold $ + 1$ with this value and compare the result to $1$ to get the answer.
\end{proof}
\subsection{Proof of~\Cref{thm:rnncot}}
\label{sec:rnncot}

We will first prove a lemma assuming $\mathrm{PSPACE} \not \subset \mathrm{P/poly}$.
\begin{lemma}
\label{lem:space_hierarchy}
If $\mathrm{PSPACE} \not \subset \mathrm{P/poly}$, then there exists a Turing machine $M$ with linear space complexity that cannot be simulated by a polynomial-size circuit family.
\end{lemma}
\begin{proof}
We will prove this by contradiction. Assuming for every Turing machine $M$ with linear space complexity, there exists a polynomial-size circuit family $\{C_n\}$ that can simulate $M$. We will construct a polynomial-size circuit family $\{C'_n\}$ that can decide $\mathrm{PSPACE}$, which contradicts the assumption that $\mathrm{PSPACE} \not \subset \mathrm{P/poly}$. 

Given any language $L \in \mathrm{PSPACE}$, we can decide $L$ by a Turing machine $M_L$ with space $O(n^k)$ for some constant $k$. We can consider another language $L' = \{x  1^{|x|^k}  \mid x \in L\}$. We can decide $L'$ by a Turing machine $M_{L'}$ with linear space complexity by checking the length of the input and then simulating $M_L$. By the assumption, there exists a polynomial-size circuit family $\{C_n\}$ that can simulate $M_{L'}$. We can then construct a polynomial-size circuit family $\{C'_n\}$ that can decide $L$ by first counting the length of the input and then simulating $C_n$ on the extended input. This contradicts the assumption that $\mathrm{PSPACE} \not \subset \mathrm{P/poly}$.
\end{proof}

Now we are ready to prove our theorem.
\thmrnncot*
\begin{proof}
By~\Cref{lem:space_hierarchy}, we know that if $\mathrm{PSPACE} \not \subset \mathrm{P/poly}$, then there exists a Turing machine $M$ with linear space complexity that cannot be simulated by a polynomial-size circuit family. We will use this result to prove~\Cref{thm:rnncot}.

We design the task as follows, for any $n$, let $m = \lfloor \log_2 n \rfloor $, for any boolean input $x$ of length $m$, we will choose input sequence as $\InputTokenSequence = \overline{0^{n-m}x}$ and the label as $y = \YES$ if $M(x) = 1$ and $y = \NO$ otherwise. 

Because we are considering regular RNN with $O(m)$ memory, we know that we can compute the result of RNN without CoT through a circuit family with size $\Poly(m)$. However, we know that $M$ cannot be simulated by a polynomial-size circuit family. Hence, no RNN family with $O(m)$ memory can solve the task for all $n$.

On the other hand, we can simulate $M$ by the RNN by maintaining the state, the pointer, and the tape of the $M$ inside the state of the RNN. The RNN can then maintain the transition function of the Turing machine in its output function as a lookup table~\Cref{lem:lookup_table} and write down the updated state, the direction of the pointer movement, and the updated cell value at the pointer in its context. Paired with the ability to memorize the recent input~\Cref{lem:recent_input_memorizing}, the RNN can then simulate the running of the Turing machine.

Because the space complexity of $M$ is linear in $m$, the time complexity of $M$ is $\exp(O(m))$ which is polynomial in $n$. Hence, we can solve the task by an RNN with CoT and $O(m)$ memory and polynomial-size circuit family.
\end{proof}
\subsection{Proof of~\Cref{thm:rnn_trans_istree}}

We will now proceed to prove our main theorem, which states that Transformers with chain-of-thought can solve IsTree perfectly, while RNNs cannot. We will first restate the theorem here.

\thmistree*

\begin{proof}[Proof of~\Cref{thm:rnn_trans_istree}]
\label{app:istree}

We will prove this theorem by proving the following lemmas.

\begin{lemma}
\label{lem:rnn_istree}
For any $n$ and RNN $R$ with $o(n)$ memory, $R$ cannot perfectly solve IsTree of size $n$.
\end{lemma}

\begin{lemma}
\label{lem:log_trans_istree}
There exists a Transformer $T$ with constant depth and width, and $O(\log n)$ precision, that can solve IsTree of size $n$ perfectly with Chain of Thought.
\end{lemma}

This proof is a direct combination of~\cref{lem:rnn_istree,lem:log_trans_istree}.
\end{proof}

\subsubsection{Proof of~\Cref{lem:rnn_istree}}
\label{sec:lem_istree}

We first reduce another problem in communication complexity to IsTree.
    
\begin{lemma}
\label{lem:index_same}
    Consider the following two-party game, where Alice receives string $x \in \{0,1\}^n$ and Bob receives an integer $k$, and Bob wants to know whether $x_k = x_{k - 1}$. If only Alice is allowed to send information, $\Omega(n)$ bit communication complexity is required.
\end{lemma}

\begin{proof}[Proof of~\Cref{lem:index_same}]
    We can reduce this problem to the problem in~\Cref{lem:index}. Considering the game in~\Cref{lem:index}, given any $x \in \{0,1\}^n$, we can construct $\tilde x_{i} = \sum_{j = 1}^i x_i \mod 2$. Then $\tilde x$ is a string of length $n$ with only $0$ and $1$. Moreover, $x_k = x_{k - 1}$ if and only if $\tilde x_k = \tilde x_{k - 1}$. Hence, if Bob can solve the problem in~\Cref{lem:index_same} with $o(n)$ bit, he can solve the problem in~\Cref{lem:index}. This is a contradiction.
\end{proof}

\begin{proof}[Proof of~\Cref{lem:rnn_istree}]

 Now suppose that we have a streaming algorithm for IsTree with only $o(n)$ memory. We shall prove Alice and Bob in~\Cref{lem:index_same} can use it to solve the original question with $o(n)$ memory.

 Consider the following graph with $n + 2$ nodes. There is a node $i$ corresponding to each $x_i$ for $i \in [n]$ and two special nodes $n + 1, n + 2$. Node $i$ will be connected to $n + 1$ if $x_i = 0$ and to $n + 2$ if $x_i = 1$. Moreover, $k - 1$ and $k$ will be connected. Now the original answer Bob wants is False if and only if the graph is a Tree. Hence, given access to the streaming algorithm, Alice can run it on the edges that she knows exist and send the memory to Bob. Bob can then run it on the edges that he knows exist. Combining they will be able to solve the original problem. This is a contradiction.

 Moreover, as RNN with CoT is also a streaming algorithm, it also requires $\Omega(n)$ memory.
\end{proof}

\subsubsection{Proof of~\Cref{lem:log_trans_istree}}
\label{sec:lem_log_istree}

\begin{proof}[Proof of~\Cref{lem:log_trans_istree}]

The proof is two-folded. We will first define an algorithm that can solve IsTree by generating a sequence of vertices of length $O(n)$, and then we will show that this sequence can be generated by a Transformer with constant depth and width, and $O(\log n)$ precision as a Chain of Thought.

\begin{algorithm}
\caption{Depth-First Search Algorithm}
\begin{algorithmic}[1]
    \label{alg:istree}
\REQUIRE A graph $G = (V, E)$ with $n$ vertices and $E$ has an ordering $e_1, \ldots, e_m$.
\STATE Initialize two stacks of vertices $S_1, S_2$ with $S_1 = [v_1], S_2 = \emptyset$.
\WHILE{$S_1$ is not empty}
    \STATE Let $v$ be the top of $S_1$. Yield $v$.
    \IF{there exists a neighbor $u$ of $v$ not in $S_1 \cup S_2$}
        \STATE Choose $u$ such that edge $(u,v)$ has the smallest possible order and push $u$ to $S_1$.
    \ELSE
        \STATE Pop $v$ from $S_1$ and push $v$ to $S_2$.
    \ENDIF
\ENDWHILE
\end{algorithmic}
\end{algorithm}

\paragraph{Algorithm for IsTree.} We define~\Cref{alg:istree} as a depth-first search algorithm that can generate a sequence of vertices of length $O(n)$ that can be used to solve IsTree. We will use two stacks $S_1, S_2$ to store the vertices. $S_1$ will be used to store the vertices that are not yet visited, and $S_2$ will be used to store the vertices that are already visited. The algorithm will start with $S_1 = [v_1]$ and $S_2 = \emptyset$. At each step, the algorithm will pop the top of $S_1$ and push it to $S_2$. Then it will push all the neighbors of the popped vertex that are not in $S_1 \cup S_2$ to $S_1$. The algorithm will terminate when $S_1$ is empty. We will denote the yielded vertice sequence for $\graph$ as $\algo(\graph)$. The following lemma shows the connection between the result of the algorithm and the IsTree problem.

\begin{lemma}
\label{lem:dfs}
    For any graph $\graph$, $\algo(\graph)$ is a tree traversal of a spanning tree of the connected component of $\graph$ containing $v_1$. Hence $\IsTree(\graph)$ is True if and only if $\graph$ has $n - 1$ edges and $\algo(\graph)$ contains $2n - 1$ vertices.
\end{lemma}

\begin{proof}[Proof of~\Cref{lem:dfs}]
    First, every vertex in the connected component of $\graph$ containing $v_1$ will be visited. This is because the algorithm will always push all the neighbors of the popped vertex that are not in $S_1 \cup S_2$ to $S_1$. Hence, the algorithm will terminate when all the vertices in the connected component of $\graph$ containing $v_1$ are visited.

    Second, every two consecutive vertices in the yielded sequence will be connected by an edge. This is because the algorithm will always push one of the neighbors of the popped vertex that is not in $S_1 \cup S_2$ to $S_1$. Hence, every two consecutive vertices in the yielded sequence will be connected by an edge. On the other hand, the combination of these edges will form a tree because the algorithm will never push a vertex that is already in $S_1 \cup S_2$ to $S_1$. Hence, the yielded sequence is a tree traversal of a spanning tree of the connected component of $\graph$ containing $v_1$.
\end{proof}

\paragraph{Construction of Transformer.}

We will now show that the yielded sequence of~\Cref{alg:istree} can be generated by a Transformer with constant depth and width, and $O(\log n)$ precision as a Chain of Thought. The Transformer will generate a valid yielded sequence but can terminate early if the graph is not a tree. We will now describe the Transformer in detail. We will assume the input token sequence $\TokenSequence$ is as follows,
\begin{equation}
    \TokenSequence = \Tokenize(\graph), v_1, \ldots v_r 
\end{equation}
for some $r \ge 0$ and $v_1 \ldots v_r$ is a valid yielded sequence. The length of $\Tokenize(\graph)$ is $3n - 2$ with $3$ tokens for each edges and $1$ special token $\StartSentence$. We will further denote the input to the first layer $\hiddenstate$ as $\Embedding(\TokenSequence)$. We will similarly denote the input to layer $\layer$ as $\hiddenstate^{(\layer)}$. We will also denote the output of the last layer as $\hiddenstate^{out}$.

\begin{enumerate}[leftmargin=*]

\item \textbf{Layer 1 and Layer 2 Attention.} The attention at Layer 1 will output zero and the FFN at Layer 1 and Attention at Layer 2 will implement a counting function (\Cref{def:count}) to count the number of vertices $n$ appears in the previous token sequence and write $\ROUND\left(\frac{1}{n}, \precision\right)$ in a new dimension $i_1$ as a result.

\item \textbf{Layer 2 FFN and Layer 3 Attention.} The FFN at Layer 2 and Attention at Layer 3 will implement a copying function (\Cref{def:copy}) copying the first dimension and the counting dimension $i_1$ of each token to its successor at two new dimensions $i_2$ and $i_3$. For each edge, this moves the type of the first vertice and the number of times the first vertice appears to $\sim$. For every vertice in the chain of thought, this moves the type of the previous vertice to them.

\item \textbf{Layer 3 FFN and Layer 4 Attention.} The FFN at Layer 3 and Attention at Layer 4 will implement another copying function, copying the dimensions $i_2$ and $i_3$ of each token to its successor at two new dimensions $i_4$ and $i_5$. Especially, for each edge, this moves the type of the first vertice and the number of times the first vertice appears to the position corresponding to the second vertices.

\item \textbf{Layer 4 FFN.}
This FFN will process the information gathered from the previous layer and prepare for the next layer. It will make sure the following properties hold for $\hiddenstate^{(5)}$,
\begin{itemize}
    \item For every token, the position number, its square, and $1$ will be kept in the last three dimensions.
    \item For the first vertices in each edges, $\sim$ and $\StartSentence$ The rest dimension will be zero.
    \item For the second vertices of each edges $(a, b)$, there will be four dimensions $i_6, i_7, i_8, i_9$ with value $a, b$ and $n_{a,e}, n_{b,e}$, where $n_{a, e} = \ROUND(\frac{1}{1 + \# a\text{appears up to current edge}}, 1)$.
    \item For vertice $v_{l}$ in $v_1, \ldots, v_r$, there will be four dimensions $i_{10}, i_{11}, i_{12}, i_{13}$ with value $v_{l}, v_{l - 1}$ and $v_{l}^2, v_{l-1}^2$ ($v_0 = 0$).
\end{itemize}

\item \textbf{Layer 5  Attention.} Combining with the previous Layer 4 FFN layer, we will implement two match functions with two attention heads matching $(i_{10}, i_{11})$ or $(i_{11}, i_{10})$ with $(i_6, i_7)$ at Layer 5 Attention, i.e. finding the edge in input for each step in the chain of thought, we will then copy $n_{v_{l}, (v_{l}, v_{l - 1})}$ to dimensions $i_{8}$ and $i_{9}$.

\item \textbf{Layer 6.} We will use Layer 5 FFN and Layer 6 Attention to implement the match function that matches dimension $i_{10}$ of the current token to $i_{10}$ in the previous token. This will match $v_l$ to the first appearance of $v_l$ in the chain of thought and we will copy $i_{11}$ of the matched token to $i_{22}$. This dimension will be the first predecessor of $v_l$ in the chain of thought (0 for $v_1$). We will denote this predecessor of $v_l$ to be $f(v_l)$ as it is the father of $v_l$ in the tree. Now we will need to split into two cases depending on whether $v_{l-1}$ is $f(v_l)$. If $v_{l - 1} = f(v_l)$ or $v_{l - 1} = 0$ (for $v_1$), we will set dimension $i_{8}$ to be $1$ and $i_{9}$ to be $0$. Otherwise, we will keep dimension $i_{8}$ and $i_{9}$ as $n_{v_{l}, (v_{l}, v_{l - 1})}$.

\item \textbf{Layer 7.} Now we will use Layer 6 FFN  and Layer 7 Attention with two attention heads to implement two $\MatchNext$ functions\footnote{The constant $A$ here in~\Cref{def:matchnextfunc}is $1$} (\Cref{def:matchnextfunc}) which use $i_8$ or $i_9$ as the counting index, and match $v_{l}$ at $i_{10}$ to $i_6$ or $i_7$ respectively. We will then copy dimensions $i_6$ to $i_9$ of the matched tokens to $i_{14}$ to $i_{21}$ (because there will be two of them). 

The match next function will be able to retrieve the first edge containing $v_1$. For any $i \ge 2$, one of the matches next function will be able to retrieve the next edge containing $v_i$ after $(v_i, v_{i + 1})$ if it exists. If it doesn't exist, the corresponding counting dimension will either be zero or no smaller than $n_{v_{l}, (v_{l}, v_{l - 1})}$. We will use Layer 6 FFN to decide whether the next edge exists and set dimension $i_{14}$ of the output of Layer 6 to be the other edge in the next edge if it exists, or $0$ otherwise, and $i_{15},i_{16}$ of the output of layer $6$ to be the counting dimension of the next edge if it exists, or $0$ otherwise. For each edge in the original input, we will also set dimension $i_{15}, i_{16}$ to be the counting dimension of the edge.

\item \textbf{Layer 8 Attention} We will grab the next edge again, in the same manner as Layer 6, but this time using dimension $i_{15}$ and $i_{16}$. The necessity of this step is that the next edge containing $(v_{i -1}, v_{i})$ in the original graph can be the same as the $(f(v_l), v_i)$ and in such case we need to check whether the next edge after this edge.

\item \textbf{Layer 8 FFN.} We now have, at each position corresponding to $v_l$, the first edge $(f(v_l), v_l)$ in the yielded sequence containing $v_l$ and the other vertex in the edge containing $v_l$ that hasn't been visited if it exists. If they don't exist, the corresponding dimension will be zero. This allows us to use Layer 8s FFN to decide the next vertex in the yielded sequence, which is exactly the first vertex different with $f(v_l)$ in the two edges if they exist, or $f(v_l)$ otherwise. We will use Layer 8 FFN to calculate the potential next vertex and put it in dimension $i_{23}$ and its square in $i_{24}$.

\item \textbf{Layer 9 Attention.} Combining with Layer 8 FFN, we will match $i_{23}$ of the current token to $i_{10}$ of the previous token to find the first appearance of the potential next vertex in the chain of thought. We will then copy dimension $d$ of the matched token to $i_{25}$. This value being $1$ will imply this vertex has never appeared in the chain of thought before and any other value will imply the vertex has appeared before.

\item \textbf{Layer 9 FFN.} We can now check several cases,
\begin{itemize}
    \item If the potential next vertex $v$ is either $f(v_r) \neq 0$ or never appears in the chain-of-thought sequence, then Layer 9 will output $n[-v, 1, -1, \ldots -1]$, which will decodes to $v$.
    \item If the potential next vertex $v$ is not $f(v_r)$ and appears in the chain-of-thought sequence, then Layer 9 will output $n\onehot_6$, which will decodes to $\NO$, because the chain of thought has already visited $v$ and hence the graph is not a tree.
    \item If $v_r = 1$ and the potential next vertex $v$ is $f(v_r) = 0$, this means the chain of thought has finished. In this case, layer 9 FFN will check whether the position is $3n - 2 + 2n - 1 = 5n - 3$ and output $n\onehot_5$ if it is, or output $n\onehot_6$ otherwise, which will decode to $\YES$ and $\NO$ respectively.
\end{itemize}
\end{enumerate}
This concludes the construction of the Transformer.
\end{proof}

\subsection{Proof of~\Cref{thm:trans_beat_rnn}}
\label{sec:transbeatrnn}

We will now prove~\Cref{thm:trans_beat_rnn}. We will first restate the theorem for convenience.

\thmtransrnn*

\begin{proof}

The proof is inspired by~\Cref{thm:trans_circuit} from~\cite{li2024chain}.

\begin{theorem}[Theorem 3.3 of~\cite{li2024chain}]
\label{thm:trans_circuit}
For any $n$ and any circuit with size $T(n)$ and input size $n$, there exists a Transformer with constant depth and precision, $O(\log n)$ width, and a position embedding with size $O(T(n)\log n)$, such that for any input $\TokenSequence$ of length $n$, the Transformer computes the output of the circuit on $\TokenSequence$ using $T(n)$ steps.
\end{theorem}

However, direct utilization of~\Cref{thm:trans_circuit} is not feasible because we are interested in (1) $O(\log n)$ precision Transformer, and (2) simulating the RNN for $n^A$ step, which would correspond to a circuit with $n^A \CircuitSize(n)$ in size. However, as the calculation is recurrent, we can encode the RNN circuit in $O(\CircuitSize(n))$ parameter instead. 

To do so, we will unroll the circuit of each recurrent step of the RNN into $\CircuitSize(n)$ gates. We will then assign each gate a unique id in $[\CircuitSize(n)]$ and assume the circuit is calculated in the order of the gate id in the following manner.

\begin{enumerate}
    \item Each gate has a type $t(i)$, which is either a constant gate outputting $1$, an input gate, a hidden state gate, an AND gate, or an XOR gate.
    \item Each gate $i$'s output depends on two values $l(i)$ and $r(i)$. If $t(i)$ is a constant gate, then $l(i)$ and $r(i)$ are assigned to be $0$. When it is an input gate, $l(i)$ will be assigned to be the coordinate of the input embedding and $r(i)$ will be assigned to be the index of the bit of the value at $l(i)$ coordinate. When it is a hidden state gate, $l(i)$ will be assigned to be the coordinate of the hidden state embedding, and $r(i)$ will be assigned to be the index of the bit of the value at $l(i)$ coordinate. If it is an AND gate or an XOR gate, $l(i)$ and $r(i)$ will be assigned to be the id of the two gates that it depends on.
\end{enumerate}

We will further assume without loss of generality that the hidden state gate is the first $\precision \hiddendimension$ gate. The output of the last $\precision \hiddendimension$ gate will be the next hidden state. We will also assume that the last $\precision(\hiddendimension + \dimension)$ to $\precision\hiddendimension - 1$ gates are the output gates. 
We will now first describe the chain of thought that the Transformer will output and then construct the Transformer.

\paragraph{Chain of thought} Taking any input $\TokenSequence$ with length $n$, the Transformer will output a sequence of $0$ and $1$ tokens. The first $n$ tokens will be the same as the input sequence. For each $a \ge 0$ and $b \in [\CircuitSize(n) + 1]$, the $n + a \left(\CircuitSize(n) + 1\right) + b$ token is 
\begin{enumerate}
    \item the output of gate $b$ when RNN circuit is calculating the output at $a$ position plus $1$, if $b \le \CircuitSize(n)$.
    \item the $n + a + 1$ token in the RNN chain of thought, if $b = \CircuitSize(n) + 1$.
\end{enumerate}

\paragraph{Construction of the Transformer.}

\begin{enumerate}
    \item \textbf{Layer 1.} The first attention layer will output zero and the first FFN layer will be of width $O(\CircuitSize(n))$, encoding all the gate information. The output of the first layer at position $n + a \left(\CircuitSize(n) + 1\right) + b$ will have the following coordinate:
    \begin{itemize}
        \item The input $i$ will be encoded in the first dimensions.
        \item $a, a^2, b, b^2$ will be encoded in four different dimensions.
        \item The gate type $t(s(b))$ will be encoded in the next dimension, where $s(b) = (b + 1) \mod (\CircuitSize(n) + 1)$ If $b = \CircuitSize(n) - 1$, then the gate type will be encoded as $0$.
        \item The necessary dependence $l(s(b)), l^2(s(b))$ and $r(s(b)), r^2(s(b))$ will be encoded in the next two dimensions.
        \item A constant $1$ will be encoded in the next dimension.
    \end{itemize}

    \item \textbf{Layer 2 Attention.} Together with the Layer 1 FFN, the Layer 2  Attention will implement two match functions (\Cref{def:match}) to copy the output of gate $l(b+1)$ and $r(b + 1)$ when RNN circuit is calculating the output at $a$ position. When the type of gate $b + 1$ is not AND or XOR, the result will not be used in the following calculation.
    
    \item \textbf{Layer 2 FFN} Layer 2 FFN will be of width $O(1)$. The output of the layer will be 
    \begin{itemize}
        \item When $b < \CircuitSize(n)$ and $t(s(b))$ is AND or XOR or constant, one dimension of the output will be the output of gate $b + 1$ when RNN circuit is calculating the output at $a$ position.
        \item  When $b < \CircuitSize(n)$ and $t(s(b))$ is an input or hidden state gate or $b = \CircuitSize(n) + 1$, one dimension of the output will be the position in the current chain of thought where the input bit or hidden state bit is copied from and the other dimension will be the square of that position
        \item When $b = \CircuitSize(n)$, the output remains the same as the input to Layer 2 FFN.
    \end{itemize}

    \item \textbf{Layer 3 Attention.} Layer 3 Attention will be of width $O(1)$. Together with Layer 2 FFN, Layer 3 Attention will implement $\Match$ heads (\Cref{def:match}) to copy the output at the position where the input bit or hidden state bit is copied from. When the type of gate $b + 1$ is not input or hidden state gate, the result will not be used in the following calculation.
    
    \item \textbf{Layer 3 FFN} Layer 3 FFN will be of width $O(1)$. The output of the layer will be
    \begin{itemize}
        \item When $b \neq \CircuitSize(n)$, one dimension of the output will be output of gate $s(b)$ when RNN circuit is calculating the output at $a + \one{b = \CircuitSize(n) + 1}$ position.
        \item When $b = \CircuitSize(n)$, the output remains the same as the input to Layer 3 FFN.
    \end{itemize} 

    \item \textbf{Layer 4} Layer 4 Attention will have $\precision - 1$ heads and each head will be of width $O(1)$. Head $h \in [\precision - 1]$ will copy the first dimension of the output of Layer 3 FFN at position $n + a(P(n) + 1) + b - (\precision - h)$ and weight each of them by $2^{-h + (\precision - 1) / 2}$ and add them in one dimension. The Layer 4 FFN will calculate $r$ when the first dimension of the input is $1$ and $-r$ otherwise. Hence, for each $a \ge 0$, the $n + a \left(\CircuitSize(n) + 1\right) - h \precision, h \in[\Lambda: \Lambda + d]$ token contains a dimension $i_1$ which is the $k - \Lambda$ dimension of the output of the RNN at position $a$.
    
    \item \textbf{Layer 5} Layer 5 Attention will have $\dimension + 1$ heads and each head will be of width $O(1)$. Head $h \in [\dimension + 1]$ will copy the dimension $i_1$ of the output of Layer 4 FFN at position $n + a(P(n) + 1) + b - (h + \Lambda)\precision$ to a disjoint dimension $i_{h + 1}$. The Layer 5 FFN will then make sure the output of Layer 5 satisfies the following:
    \begin{itemize}
        \item When $b \neq \CircuitSize(n)$, one dimension of the output will be $n$ times the output of gate $s(b)$ when RNN circuit is calculating the output at $a + \one{b = \CircuitSize(n) + 1}$  position plus $1$, which will decode to the corresponding value.
        \item When $b = \CircuitSize(n)$, the first $\dimension$ dimension of the output will be the same as the output of the RNN at position $a$, and the rest dimension will be $0$, which will decode to the same token as the chain of thought of the RNN at position $a + 1$.
    \end{itemize} 
\end{enumerate}
This concludes the construction of the Transformer.
\end{proof}

\subsection{Proof of~\Cref{thm:index_re}}

We will first state the theorem for clarity.

\thmindexre*

\begin{proof}

When task $T$ is Index, given an input sequence that ends with a query token $k$, RNN will generate the following search query sequence:$\StartSearch$, {\verb|^(?:\S\s*){|\regexparam{$k-1$}\verb|}(\S)|}, $\EndSearch$.

Then the regular expression will match the $k$-th token in the input sequence. The RNN needs to recite the format of the query, remember the index $k$ and calculate $k - 1$ to generate the regular expression. As we have shown in~\Cref{lem:recite_fixed_sequence,lem:recent_input_memorizing}, RNN can recite a fixed sequence at fixed position and memorize the recent input, the above sequence can be generated by an RNN. The explicit construction of the RNN is omitted here.

    When task $T$ is AR, given an input sequence that ends with a query token $k$, RNN will generate the following search query sequence: $\StartSearch$, {\verb|\b|\,\regexparam{$k$}\,\verb|\b(\S+)\b|}, $\EndSearch$.
    
    Then the regular expression will match the next token after the first occurrence of $k$ in the input sequence. The RNN needs to recite the format of the query and remember the query token $k$ to generate the regular expression. The explicit construction of the RNN is omitted here.

    When task $T$ is c-gram retrieval, given an input sequence that ends with query tokens, RNN will generate the following search query sequence: $\StartSearch$, {\verb|\b|\,\regexparam{$k_1 \ldots k_{c-1}$}\,\verb|\b(\S+)\b|}, $\EndSearch$.

    Then the regular expression will match the next token after the first occurrence of $k_1, \ldots k_{c-1}$ in the input sequence. The RNN needs to recite the format of the query and remember the query tokens $k_1, \ldots k_{c-1}$ to generate the regular expression. The explicit construction of the RNN is omitted here.

    When task $T$ is Counting, given an input sequence that ends with a query token $v$ and a query threshold $k$, RNN will generate the following search query sequence $\StartSearch$,{\verb|(\b|\,\regexparam{$v$}\,\verb|\b){|\regexparam{$k+1$}\verb|}|}, $\EndSearch$.

    Then the regular expression will match the $k$-th occurrence of $v$ in the input sequence. The RNN needs to recite the format of the query and remember the query token $v$ and the threshold $k$ to generate the regular expression. The RNN can then check whether the retrieval result is $\Failed$ to determine if the count is less than $k$. The explicit construction of the RNN is omitted here.

\end{proof}
\subsection{Proof of~\Cref{thm:rag_istree}}
\label{sec:proof_rag_istree}

In this section, we will prove~\Cref{thm:rag_istree}. We will first state the theorem for convenience.

\thmragistree*

\begin{proof}[Proof of~\Cref{thm:rag_istree}]

    We will first define the sequence that the retrieval-augmented RNN will generate and then construct an RNN that can generate such a sequence.

    \paragraph{Sequence Generation.} We will use a variant of~\Cref{alg:istree} to generate the sequence and we will use the concatenation of the tokenization of the sequence returned by~\Cref{alg:istreeretrieval} as the sequence that the retrieval augmented RNN will generate.

    \begin{algorithm}
        \caption{Depth-First Search Algorithm with Retrieving}
        \begin{algorithmic}[1]
            \label{alg:istreeretrieval}
        \REQUIRE A graph $G = (V, E)$ with $n$ vertices and $E$ has an ordering $e_1, \ldots, e_m$.
        \STATE Initialize two stacks of vertices $S_1, S_2$ with $S_1 = [v_1], S_2 = \emptyset$, a list $L$ with $L = \emptyset$, and a vertex $v' = \Failed$.
        \WHILE{$S_1$ is not empty}
            \STATE Let $v$ be the top of $S_1$. {\color{green}{Push $v$ to $L$.}}
            \STATE Generate the regular expression
            $ r_2 =$ {\verb|\b(\S+)\b|\,$-$\,\verb|\b|\regexparam{$v$}\verb|\b|} .
            \STATE Let $f(v)$ be the predecessor of $v$ in $S_1$ for the first time and $\Failed$ when $v = v_1$.
            \STATE {\color{green}{Push $\StartSearch, r_2, \EndSearch, f(v)$ to $L$.}}
            \IF {$v' \neq f(v)$}
                \STATE Generate the regular expression
                
                $r_1$ = \Big({\verb|\b|\regexparam{$v'$}\verb|\b|\,$\sim$\,\verb|\b|\regexparam{$v$}\verb|\b|} {\large{|}}   {\verb|\b|\regexparam{$v$}\verb|\b|\,$\sim$\,\verb|\b|\regexparam{$v'$}\verb|\b|}\Big)\verb|.*?|\Big({\verb|\b(\S+)\b|\,$\sim$\,\verb|\b|\regexparam{$v$}\verb|\b|} {\large{|}}   {\verb|\b|\regexparam{$v$}\verb|\b|\,$\sim$\,\verb|\b(\S+)\b|} \Big)                  
            \ELSE 
                \STATE Generate the regular expression
                $r_1 =${\verb|\b(\S+)\b|\,$\sim$\,\verb|\b|\regexparam{$v$}\verb|\b|} {\large{|}}   {\verb|\b|\regexparam{$v$}\verb|\b|\,$\sim$\,\verb|\b(\S+)\b|}     
            \ENDIF
            \STATE {\color{green}{Push $\StartSearch, r_1, \EndSearch$ to $L$.}}
            \IF{ there exists a neighbor $u$ of $v$ such that $(u, v)$ has larger order than $(v, v')$ when $v' \neq f(v)$ or there exists a neighbor $u$ of $v$ such that $u \neq f(v)$ when $v' = f(v)$}
                \STATE Choose $u$ such that edge $(u,v)$ has the smallest possible order and {\color{green}{push $u$ to $L$.}} Let $v'' = u$.
            \ELSE
                \STATE {\color{green}{Push $\Failed$ to $L$.}} Let $v'' = \Failed$.
            \ENDIF
            \IF{$v'' = f(v) \neq \Failed$}
            \STATE Generate the regular expression
            
            $
            r_3$ = \Big({\verb|\b|\regexparam{$v''$}\verb|\b|\,$\sim$\,\verb|\b|\regexparam{$v$}\verb|\b|} {\large{|}}   {\verb|\b|\regexparam{$v$}\verb|\b|\,$\sim$\,\verb|\b|\regexparam{$v''$}\verb|\b|}\Big)\verb|.*?|\Big({\verb|\b(\S+)\b|\,$\sim$\,\verb|\b|\regexparam{$v$}\verb|\b|} {\large{|}}   {\verb|\b|\regexparam{$v$}\verb|\b|\,$\sim$\,\verb|\b(\S+)\b|} \Big)
            \STATE {\color{green}{Push $\StartSearch, r_3, \EndSearch$ to $L$.}}
            \IF{ there exists a neighbor $u$ of $v$ such that $(u, v)$ has larger order than $(v, v'')$}
                \STATE Choose $u$ such that edge $(u,v)$ has the smallest possible order and {\color{green}{push $u$ to $L$.}} Let $v'' = u$.
            \ELSE
                \STATE {\color{green}{Push $\Failed$ to $L$.}} Let $v'' = \Failed$.
            \ENDIF
            \ENDIF
            \IF{$v'' = \Failed$}
                \STATE Pop $v$ from $S_1$. Push $v$ to $S_2$. Let $v' = v$.
            \ELSE 
                \STATE Generate the regular expression
                $r_4 = ${\verb|\b(\S+)\b|\,$-$\,\verb|\b|\regexparam{$v''$}\verb|\b|}
                \STATE {\color{green}{Push $\StartSearch, r_4, \EndSearch, 0$ to $L$.}}
                \IF{$v''$ is not in $S_1$}
                    \STATE {\color{green}{Push $\Failed, v, -, v''$ to $L$.}}
                    \STATE Push $v''$ to $S_1$. Let $v' = v$.
                \ELSE 
                    \STATE Let $f(v'')$ be the predecessor of $v$ in $S_1$ for the first time and $\Failed$ when $v'' = v_1$.
                    \STATE {\color{green}{Push $f(v''), \NO$ to $L$.}}
                    \STATE \textbf{return} $L$.
                \ENDIF
            \ENDIF
        \ENDWHILE
        \IF {$S_2$ has $n$ vertices}
            \STATE {\color{green}{Push $\YES$ to $L$.}}
            \STATE \textbf{return} $L$.
        \ELSE
            \STATE {\color{green}{Push $\NO$ to $L$.}}
            \STATE \textbf{return} $L$.
        \ENDIF
        \end{algorithmic}
        \end{algorithm}

    \paragraph{RNN Construction.} We can use similar proof in~\Cref{thm:rnn_turing} by having the RNN memorize local sequences and determine the phase of~\Cref{alg:istreeretrieval} it is in. The RNN will maintain the length of $S_2$ (\Cref{lem:summation}) and the top of $S_1$ in the state (\Cref{lem:recent_input_memorizing}) and it is easy to check that the retrieval function will retrieve the correct result for each search query. The way to determine the next vertex in the stack is the same as in the proof of~\Cref{lem:log_trans_istree}. We will omit the simple but tedious detailed construction here.
\end{proof}

\subsection{Proof of~\Cref{thm:rnn_turing}}

In this section, we will prove~\Cref{thm:rnn_turing}. We will first restate the theorem for convenience.

\rnnturing*

\begin{proof}[Proof of~\Cref{thm:rnn_turing}]
    We will denote the state of $T$ as $1, \ldots, B$ (we will use $1$ as the initial state) and the vocabulary of $T$ as $1, \ldots, B$. We will assume $T$ operates on an infinite tape $\tape$, which is a sequence of cells indexed by $\mathbb{Z}$. We will also assume that the tape is initialized with all cells being $0$ except for the $n$ cell starting at $1$. The Turing machine also has a pointer $p$ that points to a cell in the tape. The pointer is initialized to $1$. At each time step, the Turing machine reads the cell pointed by $\pointer$ and updates the cell pointed by $\pointer$ and the pointer $p$ according to the transition function $\delta: [B + 1] \times [B] \to [B] \times [B] \times \{-1, 1\}$, which takes the current state and the current cell value (could be empty, which corresponds to $B + 1$) as input and outputs the new state, the new cell value and the direction to move the pointer. The Turing machine halts when the state is $B$. Because $T \in \Time(n^A)$, the Turing machine will always halt in $n^A$ steps. We will use $\tape[t,i]$ as the value on the $i$-th cell on $\tape$ before the $t$ timestep. We will use $\pointer[t]$ as the value of the pointer before the $t$ timestep and $\TuringState[t]$ as the state of the Turing machine before the $t$ timestep. We will further use $\Direction[t]$ as the direction of the pointer movement before the $t$ timestep. 

    We will first define the sequence that the retrieval-augmented RNN will generate and then construct an RNN that can generate such a sequence.

    \paragraph{Sequence generation.} The input token sequence $\InputTokenSequence$ will be as followed,
    \begin{equation*}
        \InputTokenSequence = \StartSentence, \tape[1,1], \tape[1,2], \ldots, \tape[1,n]
    \end{equation*}
    Here all the symbols on the tape are represented by one special symbol in the vocabulary. Given this input token sequence, the retrieval augmented RNN will generate the following output token sequence, 
    \begin{align*}
        \TokenSequence =\InputTokenSequence,& \StartSearch,\texttt{\detokenize{^(?:\S}\unskip\detokenize{\s}\unskip*)\unskip.\{\regexparam{1}\}({\detokenize{\S}}\unskip)}, \EndSearch, \\
        &\result{\tape[1,1]}\\
        & 1, \tape[1,1], 1, \\
        & \ldots \\
        & \StartSearch, \texttt{\detokenize{^(?:\S}\unskip\detokenize{\s}\unskip*)\unskip.\{\regexparam{n}\}({\detokenize{\S}}\unskip)}, \EndSearch, \\ 
        &\result{\tape[1,n]}\\
        & n, \tape[1,n], n, \\
        &\StartSearch, ((\text{\regexparam{$\Pointer[1]$} \ (.) \ \regexparam{$\Pointer[1]$}\ \texttt{.*?\$}})), \EndSearch, \\
        & \result{\SearchResult(1)}, \\
        &\Pointer[1], \Tape[2, \Pointer[1]], \Pointer[1] \\
        &\TuringState[2], \Direction[2],\\
        & \ldots \\
        &\StartSearch, (\text{\regexparam{$\Pointer[t]$} \ (.) \ \regexparam{$\Pointer[t]$}\ \texttt{.*?\$}}), \EndSearch, \\
        & \result{\SearchResult(t)} \\
        & \Pointer[t], \Tape[t + 1, \Pointer[t]], \Pointer[t], \\
        &\TuringState[t + 1], \Direction[t + 1],
    \end{align*}
    Here $\SearchResult(t)$ is defined as 
    \begin{align*}
        \SearchResult(t) = \begin{cases}
            &\Failed;   \text{if } \Pointer[t] \text{ is empty cell before } t \\
            & \Tape[t, \Pointer[t]] 
              \text{;otherwise}  \\
        \end{cases}
    \end{align*}
    The output token sequence simulates the Turing machine $T$ on the tape $\tape$ due to the following lemma.

    \begin{lemma}
    \label{lem:sim_turing}
        Given any $t \in [n^A]$ and $i \in [n^A]$, the last string in $\TokenSequence$ that contains $i, i$ as a substring is $i, \tape[t,i], i$ if $\tape[t,i]$ is not empty and is the empty string otherwise.
    \end{lemma}
    \begin{proof}
        The proof is by induction, for $t = 1$, the result holds. For any $t \ge 2$, we only need to notice that $\Pointer[t -1]$ is the only cell that can be updated at time $t - 1$.
    \end{proof}

    \paragraph{Construction of RNN}

    Given the input, the RNN can first iterate over $1$ to $n$ and generate the first $n$ search queries and results by maintaining a counter in its state and memorizing the most recent search result (\Cref{lem:recent_input_memorizing}). Then it is easy to see that the retrieval oracle will generate the correct $\SearchResult(t)$ given the input $\TokenSequence$. Therefore, we will only need to construct an RNN that can generate the rest part of $\TokenSequence$.

    We will assume the RNN maintains the state and pointer of the Turing machine in its state and show that they can be updated.

    Based on~\Cref{lem:recent_input_memorizing}, the RNN can maintain constant recent token types in its state, we will assume the RNN memorize the last tokens up to the most recent $\StartSearch$ and also calculate the position relative to the most recent $\StartSearch$. By a lookup table in the FFN~\Cref{lem:lookup_table}, the RNN can output the fixed format of the search query. Similarly, RNN can output the $\Pointer[t]$. To generate the update $\Tape[t + 1, \Pointer[t]], \TuringState[t], \Direction[t]$, the RNN can use a FFN with $O(B^2)$ width to memorize the transition function of the Turing machine and output the update. Then, the RNN can use the memorized recent input to update the state and the pointer of the Turing machine at the next $\StartSearch$. The proof is then complete.

\end{proof}

\subsection{Proof of~\Cref{thm:hybrid_index}}

\thmhybridindex*

\begin{proof}
    The proof here is essentially the same as the construction of the Transformer in~\Cref{thm:index}. We would use the same Transformer layer to solve $T$. The only difference is that we would use the output of the RNN, instead of FFN, as the input of the Transformer layer.Also for Counting, instead of using a $\COPY$ function, we write the query token in the state of the RNN (\Cref{lem:recent_input_memorizing}).
\end{proof}

\subsection{Proof of~\Cref{thm:hybrid_istree}}
\label{sec:thmhybridistree}
\thmhybridistree*

\begin{proof}
The proof is similar to the proof of~\Cref{thm:rag_istree}. However, instead of using regular expressions to retrieve the next neighbor and parent, we will need to use the Transformer layer. The Transformer layer can retrieve the parent through an attention head implementing the match closest head (\Cref{lem:matchclose}) if the RNN part maintains the predecessor of each node in the chain of thought. 

Retrieving the next neighbor is more complicated and we will use $O(\log n)$ steps of the chain of thought to do that. Given an edge $(v, v')$, we will first use one match head to retrieve the position $p$ of $(v, v')$ in the input sequence and write it to the chain of thought. Then we will use two $\MatchClose$ heads to retrieve the edge that contains $v$ and is closest to $p + 2^i$ for $i = 0, 1, \ldots, \log_2 n$ iteratively until the heads return an edge that is not $(v, v')$ or $i$ reaches $\log_2 n$. Here $2^i$ can be computed through doubling one of the dimensions in the state of the RNN and reset that dimension to $1$ after termination.
We will then compare the retrieved edge with the father of $v$ to check if it is the same. If it is the same, we will search the next neighbor of $v$ after the parent of $v$ in the same way. The other part of the proof is similar to the proof of~\Cref{thm:rag_istree}.
\end{proof}

\subsection{Proof of~\Cref{thm:hybrid_turing}}

\hybridturing*

\begin{proof}

\textbf{Sequence Generation.} Under the same formulation of proof of the~\Cref{thm:rnn_turing}. The hybrid RNN will output the following sequence.
\begin{align*}
    \TokenSequence =\InputTokenSequence,& \Pointer[1], \Tape[2, \Pointer[1]], \Pointer[1] \\
    &\TuringState[2], \Direction[2], \\
    & \ldots \\
    & \Pointer[t], \Tape[t + 1, \Pointer[t]], \Pointer[t], \\
    &\TuringState[t + 1], \Direction[t + 1],
\end{align*}

Note that~\Cref{lem:sim_turing} still holds. We only need to prove that the hybrid architecture can generate the above sequence.

\textbf{Hybrid Construction.} The way RNN maintains the pointers and the states is the same as the proof of~\Cref{thm:rnn_turing}. Given each pointer value, we can retrieve the last value of the cell at the pointer through the one layer of attention by implementing a match closest head (\Cref{lem:matchclose}).
\end{proof}

\end{document}